\documentclass{article}
\usepackage[accepted]{icml2021}

\usepackage{graphicx}
\usepackage{amsthm}
\usepackage{amssymb}
\usepackage{amsmath}
\usepackage{mathtools}
\usepackage{thmtools}
\usepackage{subcaption}
\usepackage{dblfloatfix}
\usepackage[section]{placeins}
\usepackage{color}
\usepackage{soul}
\usepackage{hyperref}
\usepackage{url}
\usepackage[normalem]{ulem}
\include{pythonlisting}
\usepackage[noend,ruled,vlined]{algorithm2e}

\newtheorem{definition}{Definition}[section]
\newtheorem{setup}{Setup}
\newtheorem{theorem}{Theorem}[section]
\newtheorem{lemma}{Lemma}
\newtheorem{proposition}{Proposition}[section]

\icmltitlerunning{Detecting Rewards Deterioration in Episodic Reinforcement Learning}

\begin{document}

\twocolumn[

\icmltitle{Detecting Rewards Deterioration in Episodic Reinforcement Learning}

\begin{icmlauthorlist}
\icmlauthor{Ido Greenberg}{Technion}
\icmlauthor{Shie Mannor}{Technion,Nvidia}
\end{icmlauthorlist}

\icmlaffiliation{Technion}{Department of Electric Engineering, Technion, Israel}
\icmlaffiliation{Nvidia}{Nvidia Research}

\icmlcorrespondingauthor{Ido Greenberg}{gido@campus.technion.ac.il}
\icmlcorrespondingauthor{Shie Mannor}{shie@ee.technion.ac.il}

\vskip 0.3in
]

\printAffiliationsAndNotice{}

\begin{abstract}
In many RL applications, once training ends, it is vital to detect any deterioration in the agent performance as soon as possible.
Furthermore, it often has to be done without modifying the policy and under minimal assumptions regarding the environment.
In this paper, we address this problem by focusing directly on the rewards and testing for degradation.
We consider an episodic framework, where the rewards within each episode are not independent, nor identically-distributed, nor Markov.
We present this problem as a multivariate mean-shift detection problem with possibly partial observations.
We define the mean-shift in a way corresponding to deterioration of a temporal signal (such as the rewards), and derive a test for this problem with optimal statistical power.
Empirically, on deteriorated rewards in control problems (generated using various environment modifications), the test is demonstrated to be more powerful than standard tests -- often by orders of magnitude.
We also suggest a novel Bootstrap mechanism for False Alarm Rate control (BFAR), applicable to episodic (non-i.i.d) signal and allowing our test to run sequentially in an online manner.
Our method does not rely on a learned model of the environment, is entirely external to the agent, and in fact can be applied to detect changes or drifts in any episodic signal.
\end{abstract}


\section{Introduction}
\label{sec:intro}

Reinforcement learning (RL) algorithms have recently demonstrated impressive success in a variety of sequential decision-making problems~\citep{agent57,drl_improvements}.
While most RL works focus on the maximization of rewards under various conditions, a key issue in real-world RL tasks is the safety and reliability of the system~\citep{real_world_rl_challenges2,RL_reliability}, arising in both offline and online settings.

In {\bf offline settings}, comparing the agent performance in different environments is important for generalization (e.g., in sim-to-real and transfer learning).
The comparison may indicate the difficulty of the problem or help to select the right learning algorithms.
Uncertainty estimation, which could help to address this challenge, is currently considered a hard problem in RL, in particular for model-free methods~\citep{MOPO}.


In {\bf online settings}, where a fixed, already-trained agent runs continuously, its performance may be affected (gradually or abruptly) by changes in the controlled system or its surroundings, or when reaching unfamiliar states.
Some works address robustness to changes~\citep{nonstationary_mdp,context_awareness}, yet performance degradation is sometimes inevitable, and should be detected as soon as possible.
The detection allows us to fall back into manual control, send the agent to re-train, guide diagnosis, or even bring the agent to halt.
This problem is inherently different from robustness to changes during training: it focuses on safety and reliability, in post-training phase where intervention in the policy is limited or forbidden \citep{DeploymentEfficientRL}. It also operates in different time-scales: while training may take millions of episodes, changes should often be detected within tens of episodes, and critical failures -- within less than an episode.

Such post-training performance-awareness is essential for any autonomous system in risk-intolerant applications, such as autonomous driving and medical devices.
For example, when an autonomous car starts acting suspiciously with a passenger sitting inside, activating a training process and exploring for new policies is not an option. {\bf The priority is to notice the suspicious behavior as soon as possible}, so that it can be alerted in time to save lives.




\begin{figure*}
\centering
\begin{subfigure}{.35\textwidth}
  \centering
  \includegraphics[width=1.\linewidth]{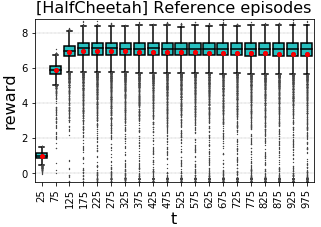}
  \caption{}
  \label{fig:cheetah_rewards}
\end{subfigure}%
\begin{subfigure}{.2\textwidth}
  \centering
  \includegraphics[width=1.\linewidth]{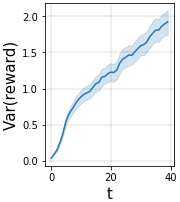}
  \caption{}
  \label{fig:cheetah_var}
\end{subfigure}
\begin{subfigure}{.35\textwidth}
  \centering
  \includegraphics[width=1.\linewidth]{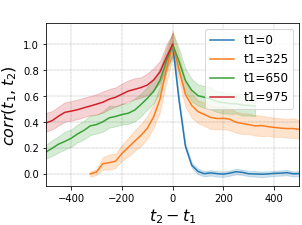}
  \caption{}
  \label{fig:cheetah_acf}
\end{subfigure}%
\caption{\footnotesize Properties of the rewards of a fixed agent in HalfCheetah, estimated over $N=10000$ episodes of $T=1000$ time-steps: (a) distribution of rewards per time-step; (b) variance per time-step; (c) correlation($t_1,t_2$) vs. $t_2-t_1$. The estimations are in resolution of 25 time-steps, i.e., every episode was split into 40 intervals of 25 consecutive steps, and each sample is the average over an interval.}
\label{fig:cheetah_params}
\end{figure*}

Many sequential statistical tests exist for detection of mean degradation in a random process.
However, common methods~\citep{CUSUM,alpha_spending,concept_drift} assume independent and identically distributed (i.i.d) samples, while in RL the feedback from the environment is usually both highly correlated over consecutive time-steps, and varies over the life-time of the task~\citep{autoregressive_policies}.
This is demonstrated in Fig.~\ref{fig:cheetah_params}. 

A possible solution is to apply statistical tests to large blocks of data assumed to be i.i.d~\citep{learning_in_nonstationary_env}.
This is particularly common in RL, where the episodic settings allow a natural blocks-partition (see for example \citet{rl_algos_comparison}).
However, this approach requires complete episodes for change detection, while a faster response is often required.
Furthermore, naively applying a statistical test on the accumulated feedback (e.g., sum of rewards) from complete episodes, ignores the dependencies within the episodes and misses vital information, leading to highly sub-optimal tests (as demonstrated in Section~\ref{sec:results}).

In this work, we devise an optimal test for detection of degradation of the rewards in an episodic RL task (or in any other episodic signal), based on the covariance structure within the episodes.
Even in absence of the assumptions that guarantee its optimality, the test is still asymptotically superior to the common approach of comparing the mean reward~\citep{rl_algos_comparison}.
The test can detect changes and drifts in both the offline and the online settings defined above.
Since tuning of the False Alarm Rate (FAR) of a sequential test usually relies on the underlying signal being i.i.d, we also suggest a novel Bootstrap mechanism for FAR control (BFAR) in sequential tests on episodic signals.
The suggested procedures rely on the ability to estimate the correlations within the episodes, e.g., through a "reference dataset" of episodes.

Since the test is applied directly to the rewards, it is model-free in the following senses: the underlying process is not assumed to be known, to be Markov, or to be observable at all (as opposed to other works, e.g., \citet{Quickest_change_detection_MDP}), and we require no knowledge about the process or the running policy.
Furthermore, as the rewards are simply referred to as episodic time-series, the test can be similarly applied to detect changes in any episodic signal.

We demonstrate the new procedures in the environments of Pendulum~\citep{Pendulum}, HalfCheetah and Humanoid~\citep{HalfCheetah,mujoco}.
BFAR is shown to successfully control the false alarm rate.
The suggested test detects degradation faster and more often than three alternative tests -- in certain cases by orders of magnitude.

The paper is organized as follows: Section~\ref{sec:setup} formulates the offline setup (individual tests) and the online setup (sequential tests).
Section~\ref{sec:optimization} defines the model of an episodic signal, and derives an optimal degradation-test for such a signal.
Section~\ref{sec:sequential_test} shows how to adjust the test for online settings and control the false alarm rate.
Section~\ref{sec:experiments} describes the experiments, and Section~\ref{sec:related_work} discusses related works. 

To the best of our knowledge, we are the first to exploit the covariance between rewards in post-training phase to test for changes in RL-based systems.
Our main contribution is an optimal test that can detect deterioration in agent rewards and other episodic signals reliably, in much shorter times than current standard tests.
We also suggest a novel bootstrap mechanism to control false alarm rate of such tests on episodic (non-i.i.d) data.
Finally, we lay a new framework for statistical tests on episodic signals, which opens the way for further research on this problem.


\section{Preliminaries}
\label{sec:preliminaries}

\paragraph{Reinforcement learning and episodic framework:}
A Reinforcement Learning (RL) problem is usually modeled as a sequential \textit{decision process}, where a learning \textit{agent} has to repeatedly make decisions that affect its future states and rewards.
The process is often organized as a finite sequence of time-steps (an \textit{episode}) that repeats multiple times in different variants, e.g., with different initial states.
Common examples are board and video games~\citep{gym}, as well as more realistic problems such as autonomous driving tasks.

Once the agent is fixed (which is the case in this work), the rewards of the decision process essentially reduce to a (decision-free) random process $\{X_t\}_{t=1}^{n}$, which can be defined by its PDF ($f_{\{X_t\}_{t=1}^{n}}: \mathbb{R}^n \rightarrow [0,\infty)$).
$\{X_t\}$ usually depend on each other: even in the popular \textit{Markov Decision Process}~\citep{MDP_Bellman}, where the dependence goes only a single step back, long-term correlations may still carry information if the states are not observable by the agent.

\paragraph{Hypothesis tests:}
Consider a parametric probability function $p(X|\theta)$ describing a random process, and consider two different hypotheses $H_0,H_A$ determining the value (\textit{simple hypothesis}) or allowed values (\textit{complex hypothesis}) of $\theta$.
When designing a test to decide between the hypotheses, the basic metrics for the test efficacy are its \textit{significance} $P(\text{not reject } H_0 | H_0 )=1-\alpha$ and its \textit{power} $P(\text{reject } H_0 | H_A )=\beta$.
A hypothesis test with significance $1-\alpha$ and power $\beta$ is \textit{optimal} if any test with as high significance $1-\tilde{\alpha} \ge 1-\alpha$ has smaller power $\tilde{\beta} \le \beta$.

The likelihood of the hypothesis $H: \theta \in \Theta$ given data $X$ is defined as $L(H|X) = \text{sup}_{\theta\in\Theta}p(X|\theta)$.
According to Neyman-Pearson lemma~\citep{NeymanPearson}, a threshold-test on the likelihood ratio $LR(H_0,H_A|X) = L(H_0|X) / L(H_A|X)$ is optimal. 
The threshold is uniquely determined by the desired significance level $\alpha$, though is often difficult to practically calculate given $\alpha$. 

In many practical applications, a hypothesis test is repeatedly applied as the data change or grow, a procedure known as a \textit{sequential test}.
If the null hypothesis $H_0$ is true, and any individual hypothesis test falsely rejects $H_0$ with some probability $\alpha$, then the probability that at least one of the multiple tests will reject $H_0$ is $\alpha_0 > \alpha$, termed \textit{family-wise type-I error} (or \textit{false alarm rate} when associated with frequency). See Appendix~\ref{sec:detailed_preliminaries} for more details about hypothesis testing and sequential tests in particular.

Common approaches for sequential tests, such as CUSUM~\citep{CUSUM,CUSUM_book} and $\alpha$-spending functions~\citep{alpha_spending,Pocock}, usually require strong assumptions such as independence or normality, as further discussed in Appendix~\ref{sec:detailed_related_work}.





\section{Problem Setup}
\label{sec:setup}

In this work, we consider two setups where detecting performance deterioration is important -- sequential degradation-tests and individual degradation-tests.
The individual tests, in addition to their importance in offline settings such as sim-to-real and transfer learning, are used in this work as building-blocks for the online sequential tests.

Both setups assume a fixed agent that was previously trained, and aim to detect whenever the agent performance begins to deteriorate, e.g., due to environment changes.
The ability to notice such changes is essential in many real-world problems, as explained in Section~\ref{sec:intro}. 

\begin{setup}[\bf Individual degradation-test]
\label{setup:ind}
\normalfont
We consider a fixed trained agent (policy must be fixed but is not necessarily optimal), whose rewards in an episodic environment (with episodes of length $T$) were previously recorded for multiple episodes (the \textit{reference dataset}).
The agent runs in a new environment for $n$ time-steps (both $n<T$ and $n\ge T$ are valid).
The goal is to decide whether the rewards in the new environment are smaller than the original environment or not.
If the new environment is identical, the probability of a false alarm must not exceed $\alpha$.
\end{setup}

\begin{setup}[\bf Sequential degradation-test]
\label{setup:seq}
\normalfont
As in Setup~\ref{setup:ind}, we consider a fixed trained agent with reference data of multiple episodes.
This time the agent keeps running in the same environment, and at a certain point in time its rewards begin to deteriorate, e.g., due to changes in the environment.
The goal is to alert to the degradation as soon as possible.
As long as the environment has not changed, the probability of a false alarm must not exceed $\alpha_0$ per $\tilde{h}$ episodes.
\end{setup}

Note that while in this work the setups focus on degradation, they can be easily modified to look for any change (as positive changes may also indicate the need for further training, for example).





\section{Optimization of Individual Tests}
\label{sec:optimization}

To tackle the problem of Setup~\ref{setup:ind}, we first define the properties of an episodic signal and the general assumptions regarding its degradation.

\begin{definition}[$T$-long episodic signal]
\label{def:episodic_signal}
Let $n,T \in \mathbb{N}$, and write $n=KT+\tau_0$ (for non-negative integers $K,\tau_0$ with $\tau_0\le T$).
A sequence of real-valued random variables $\{X_t\}_{t=1}^{n}$ is a \textit{$T$-long episodic signal}, if its joint probability density function can be written as
\begin{align}
\label{eq:episodic_dist}
\begin{split}
    f&_{\{X_t\}_{t=1}^{n}} (x_1,...,x_n) = \\
    &\left[ \prod_{k=0}^{K-1} f_{\{X_t\}_{t=1}^{T}}(\{x_{kT+t}\}_{t=1}^{T}) \right] \cdot f_{\{X_t\}_{t=1}^{\tau_0}}(\{x_{KT+t}\}_{t=1}^{\tau_0})
\end{split}
\end{align}
(where an empty product is defined as 1).
We further denote $\pmb{\mu_0} \coloneqq E[(X_1,...,X_T)^\top] \in\mathbb{R}^T$, $\Sigma_0 \coloneqq \text{Cov}((X_1,...,X_T)^\top$, $(X_1,...,X_T)) \in \mathbb{R}^{T \times T}$.
\end{definition}

Note that the episodic signal consists of i.i.d episodes, but is not assumed to be independent or identically-distributed within the episodes -- a setup particularly popular in RL.

In the analysis below we assume that both $\pmb{\mu_0}$ and $\Sigma_0$ are known.
This can be achieved either with detailed domain knowledge, or by estimation from the recorded reference dataset of Setup~\ref{setup:ind}, assuming it satisfies Eq.~\eqref{eq:episodic_dist}.
The estimation errors decrease as $\mathcal{O}(1/\sqrt{N})$ with the number $N$ of reference episodes, and are distributed according to the Central Limit Theorem (for means) and Wishart distribution~\citep{Wishart} (for covariance).
While in this work we use up to $N=10000$ reference episodes, Appendix~\ref{sec:cov_sensitivity} shows that $N=300$ reference episodes are sufficient for reasonable results in HalfCheetah, for example.
Note that correlations estimation has been already discussed in several other RL works~\citep{correlations_prior}.

Fig.~\ref{fig:cheetah_params} demonstrates the estimation of mean and covariance parameters for a trained agent in the environment of HalfCheetah, from a reference dataset of $N=10000$ episodes.
This also demonstrates the non-trivial correlations structure in the environment.
According to Fig.~\ref{fig:cheetah_var}, the variance in the rewards varies and does not seem to reach stationarity within the scope of an episode.
Fig.~\ref{fig:cheetah_acf} shows the autocorrelation function $ACF(t_2-t_1) = corr(t_1,t_2)$ for different reference times $t_1$. The correlations clearly last for hundreds of time-steps, and depend on the time $t_1$ rather than merely on the time-difference $t_2-t_1$. This means that the autocorrelation function is not expressive enough for the actual correlations structure.

Once the per-episode parameters $\pmb{\mu_0} \in \mathbb{R}^T,\Sigma_0 \in \mathbb{R}^{T\times T}$ are known, the mean $\pmb{\mu} \in \mathbb{R}^n$ and covariance $\Sigma \in \mathbb{R}^{n \times n}$ of the whole signal can be derived directly:
$\pmb{\mu}$ consists of periodic repetitions of $\pmb{\mu_0}$, and $\Sigma$ consists of copies of $\Sigma_0$ as $T\times T$ blocks along its diagonal. For both, the last repetition is cropped if $n$ is not an integer multiplication of $T$.
In other words, by taking advantage of the episodic setup, we can treat the temporal univariate non-i.i.d signal as a multivariate signal with easily-measured mean and covariance -- even if the signal ends in the middle of an episode.

The degradation in the signal $X=\{X_t\}_{t=1}^{n}$ is defined through the difference between two hypotheses.
The null hypothesis $H_0$ states that $X$ is a $T$-long episodic signal with expectations $\pmb{\mu_0} \in R^T$ and invertible covariance matrix $\Sigma_0 \in R^{T \times T}$.
Our first alternative hypothesis ($H_A$) -- uniform degradation -- states that $X$ is a $T$-long episodic signal with the same covariance $\Sigma_0$ but smaller expectations: $\exists \epsilon\ge\epsilon_0, \forall 1\le t\le T: (\pmb{\tilde{\mu}_0})_t = (\pmb{\mu_0})_t - \epsilon$.
Note that this hypothesis is complex ($\epsilon\ge\epsilon_0$), where $\epsilon_0$ can be tuned according to the minimal degradation magnitude of interest. In fact, Theorem~\ref{theorem:unif_optimality} shows that the optimal corresponding test is independent of the choice of $\epsilon_0$.

\begin{theorem}[Optimal test for uniform degradation]
\label{theorem:unif_optimality}
Define the \textit{uniform-degradation weighted-mean} $s_{unif}(X) \coloneqq W\cdot X$, where $W \coloneqq \pmb{1}^\top\cdot \Sigma^{-1} \in \mathbb{R}^n$ (and $\pmb{1}$ is the all-1 vector).
If the distribution of $X$ is multivariate normal, then a threshold-test on $s_{unif}$ is optimal.
\end{theorem}

\begin{proof}[Proof Sketch (see full proof in Appendix~\ref{app:deg_calculations})]
According to Neyman-Pearson lemma \citep{NeymanPearson}, a threshold-test on the likelihood-ratio (LR) between $H_0$ and $H_A$ is optimal. Since $H_A$ is complex, the LR is a minimum over $\epsilon\in[\epsilon_0,\infty)$. Lemma~\ref{lemma:unif_optimality} shows that $\exists s_0: s_{unif}\ge s_0 \Rightarrow \epsilon=\epsilon_0$ and $s_{unif}\le s_0 \Rightarrow \epsilon=\epsilon(s_{unif})$.
The rest of the proof substitutes $\epsilon$ in both domains of $s_{unif}$ to prove monotony of the LR in $s_{unif}$, from which we can conclude monotony in $s_{unif}$ over all $\mathbb{R}$.
\end{proof}

Following Theorem~\ref{theorem:unif_optimality}, we define the Uniform Degradation Test ({\bf UDT}) to be a threshold-test on $s_{unif}$, i.e., "declare a degradation if $s_{unif}<\kappa$" for a pre-defined $\kappa$.
If the weights are calculated in advance, $s_{unif}$ can be calculated in $\mathcal{O}(n)$ time, and updated in $\mathcal{O}(1)$ with every new sample.

Recall that test optimality is defined in Section~\ref{sec:preliminaries} as having maximal power per significance level.
To achieve the significance $\alpha$ required in Setup~\ref{setup:ind}, we apply a bootstrap mechanism that randomly samples episodes from the reference data and calculates the corresponding statistic (e.g., $s_{unif}$). This yields a bootstrap-estimate of the statistic's distribution under $H_0$, and the $\alpha$-quantile of the estimated distribution is chosen as the test-threshold ($\kappa = q_\alpha(s_{unif}|H_0)$).

$H_A$ is intended for degradation in a temporal signal, and derives a different optimal statistic than standard mean-change tests in multivariate variables (e.g., Hotelling).
In Section~\ref{sec:experiments}, this is indeed demonstrated to be more powerful for rewards degradation.
Also note that by explicitly referring to the temporal dimension, we allow detections even before the first episode is completed.

Theorem~\ref{theorem:unif_optimality} relies on multivariate normality assumption, which is often too strong for real-world applications.
Theorem~\ref{theorem:unif_power} guarantees that if we remove the normality assumption, it is still beneficial to look into the episodes instead of considering them as atomic blocks; that is, UDT is still asymptotically better than a test on the simple mean $s_{simp}=\sum_{t=1}^n X_t / n$.
Note that "asymptotic" refers to the signal length $n \rightarrow \infty$ (while $T$ remains constant), and is translated in the sequential setup into a "very long lookback-horizon $h$" (rather than very long running time).

\begin{theorem}[Asymptotic power of UDT]
\label{theorem:unif_power}
Denote the length of the signal $n=K\cdot T$, assume a uniform degradation of size $\frac{\epsilon}{\sqrt{K}}$, and let two threshold-tests $\tau_{simp}$ on $s_{simp}$ and UDT on $s_{unif}$ be tuned to have significance $\alpha$.
Then
\begin{align}
\label{eq:udt_power}
\begin{split}
    &\lim_{K\rightarrow \infty} P\left( \tau_{simp} \text{ rejects } \big| H_A \right) = \Phi\left(q_\alpha^0 + \frac{\epsilon T}{\sqrt{\pmb{1}^\top \Sigma_0 \pmb{1}}}\right) \\
    &\le \Phi\left(q_\alpha^0 + \epsilon \sqrt{\pmb{1}^\top \Sigma_0^{-1} \pmb{1}}\right) = \lim_{K\rightarrow \infty} P\left( \text{UDT rejects } \big| H_A \right)
\end{split}
\end{align}
where $\Phi$ is the CDF of the standard normal distribution, and $q_\alpha^0$ is its $\alpha$-quantile.
\end{theorem}

\begin{proof}[Proof Sketch (see full proof in Appendix~\ref{app:deg_calculations})]
Since the episodes of the signal are i.i.d, both $s_{simp}$ and $s_{unif}$ are asymptotically normal according to the Central Limit Theorem.
The means and variances of both statistics are calculated in Lemma~\ref{lemma:unif_properties}.
Calculation of the variance of $s_{unif}$ relies on writing $s_{unif}$ as a sum of linear transformations of $X$ ($s_{unif}=\sum_{i=1}^n (\Sigma^{-1})_{i}X$), and using the relation between $\Sigma$ and $\Sigma_0$.
The inequality between the resulted powers is shown to be equivalent to a matrix-form of the means-inequality, and is proved using Cauchy-Schwarz inequality for $\Sigma_0^{-1/2}\pmb{1}$ and $\Sigma_0^{1/2}\pmb{1}$.
\end{proof}

Motivated by Theorem~\ref{theorem:unif_power}, we define $G^2 \coloneqq \frac{(\pmb{1}^\top \Sigma_0^{-1} \pmb{1})(\pmb{1}^\top \Sigma_0 \pmb{1})}{T^2}$ to be the asymptotic power gain of UDT, quantify it, and show that it increases with the heterogeneity of the spectrum of $\Sigma_0$.
In particular, if the rewards are heterogeneous, the suggested test is guaranteed to detect uniform degradation with much higher probability than the standard mean-test.

\begin{proposition}[Asymptotic power gain]
\label{prop:unif_power_gain}
$G^2 = 1 + \sum_{i,j=1}^T w_{ij}(\lambda_i-\lambda_j)^2$,
where $\{\lambda_i\}_{i=1}^T$ are the eigenvalues of $\Sigma_0$ and $\{w_{ij}\}_{i,j=1}^T$ are positive weights.
\end{proposition}

\begin{proof}[Proof Sketch (see full proof in Appendix~\ref{app:deg_calculations})]
The result can be calculated after diagonalization of $\Sigma_0$, and the weights $\{w_{ij}\}$ are derived from the diagonalizing matrix.
\end{proof}

\begin{figure}
  \begin{center}
    \includegraphics[width=0.8\linewidth]{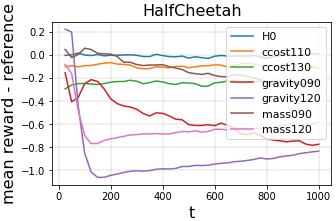}
  \end{center}
  \caption{\footnotesize Rewards degradation of a fixed agent in HalfCheetah following changes in gravity, mass, and control-cost, over $N=5000$ episodes per scenario.}
  \label{fig:cheetah_degradation}
\end{figure}

So far we assumed uniform degradation.
In the context of RL, such a model may refer to changes in constant costs or action costs, as well as certain dynamics whose change influences various states in a similar way.
Fig.~\ref{fig:cheetah_degradation} demonstrates the empiric degradation in the rewards of a fixed agent in HalfCheetah, following changes in gravity, mass and control-cost.
It seems that some modifications indeed cause a quite uniform degradation, while in others the degradation is mostly restricted to certain ranges of time. 

To model effects that are less uniform in time we suggest a partial degradation hypothesis, where some (unknown) entries of $\pmb{\mu_0}$ are reduced by $\epsilon>0$, and others do not change.
The number $m=p\cdot T$ of the reduced entries is defined by a parameter $p\in (0,1)$. 


This time, calculation of the optimal test-statistic through the LR yields a minimum over $\binom{T}{m}$ possible subsets of decreased entries, which is computationally heavy.
However, Theorem~\ref{theorem:part_optimality} shows that if we optimize for small values of $\epsilon$ (where optimality is indeed most valuable), a near-optimal statistic is $s_{part}$, which is the sum of the $m=p\cdot T$ smallest time-steps of $(X-\pmb{\mu})$ after a $\Sigma^{-1}$-transformation (see formal definition in Definition~\ref{def:part_weighted_mean}). The resulted time-complexity is $\mathcal{O}(nT)$.
We define the Partial Degradation Test ({\bf PDT}) as a threshold-test on $s_{part}$ with a parameter $p$.

\begin{theorem}[Near-optimal test for uniform degradation]
\label{theorem:part_optimality}
Assume that $X$ is multivariate normal, and let $P_\alpha$ be the maximal power of a hypothesis test with significance $1-\alpha$.
The power of a threshold-test on $s_{part}$ with significance $1-\alpha$ is $P_\alpha - \mathcal{O}(\epsilon)$.
\end{theorem}

\begin{proof}[Proof Sketch]
The expression to be minimized is shown to be the sum of two terms. One term is the sum of a subset of entries of $\Sigma^{-1}(X-\pmb{\mu})$, which is minimized by simply taking the lowest entries (up to the constraint of consistency across episodes, which requires us to sum the rewards per time-step in advance).
In Appendix~\ref{app:deg_calculations} we bound the second term and its effects on the modified statistic and on the modified test-threshold. We show that the resulted decrease of rejection probability is $\mathcal{O}(\epsilon)$.
\end{proof}




\section{Bootstrap for False Alarm Rate Control (BFAR)} 
\label{sec:sequential_test}

For Setup~\ref{setup:seq}, we suggest a sequential testing procedure:
run an individual test every $d$ steps (i.e., $F=T/d$ test-points per episode), and return once any individual test declares a degradation.
The tests can run according to Section~\ref{sec:optimization}, applied on the $h$ recent episodes.
Multiple tests may be applied every test-point, e.g., with varying test-statistics $\{s\}$ or lookback-horizons $\{h\}$.
This procedure, as implemented for the experiments of Section~\ref{sec:experiments}, is described in Fig.~\ref{fig:sequential_setup}.

Setup~\ref{setup:seq} limits the probability of a false alarm to $\alpha_0$ in a run of $\tilde{h}$ episodes.
To satisfy this condition, we set a uniform threshold $\kappa$ on the $p$-values of the individual tests (i.e., declare once a test returns $p\text{-val}<\kappa$).
The threshold is determined using a Bootstrap mechanism for False Alarm control ({\bf BFAR}, Algorithm~\ref{algo:sequential_bootstrap}).

While bootstrap methods for false alarm control are quite popular, they often rely on the data samples being i.i.d~\citep{MaxSPRT,nonparametric_sprt}, which is crucial for the re-sampling to reliably mimic the source of the signal.
To address the non-i.i.d signal, we take advantage of the episodic framework and sample whole episodes.
We then use the re-sampled sequence to simulate tests on sub-sequences where the first and last episodes may be incomplete, as described below.
This allows simulation of sequences of various lengths (including non-integer number of episodes) without assuming independence, normality, or identical distributions within the episodes.

\begin{algorithm}
\caption{BFAR: Bootstrap for FAR control}
\label{algo:sequential_bootstrap}
\SetAlgoLined
 {\bf Input}: reference dataset $x\in \mathbb{R}^{N\times T}$; statistic functions $\{s\}$; lookback-horizons $\{h_1,...,h_{max}\}$; test length $\tilde{h}\in \mathbb{N}$; bootstrap repetitions $B\in\mathbb{N}$; desired significance $\alpha_0\in(0,1)$\;
 {\bf Output}: test threshold for individual tests\;

 Initialize $P = (1,...,1)\in [0,1]^{B}$\;
 \For{b in 1:$B$}{
  Initialize $Y \in \mathbb{R}^{(h_{max}+\tilde{h})T}$\;
  
  \For{k in 0:($h_{max}$+$\tilde{h}$-1)}{
   Sample $j$ uniformly from $(1,...,N)$\;
   $Y[kT+1:kT+T] \leftarrow (x_{j1},...,x_{jT})$\;
  }
  
  \For{$t$ in test-points}{
    \For{h in \text{lookback-horizons} and s in \text{statistic functions}}{
     $y \leftarrow Y[t-hT : t]$\;
     $p \leftarrow \text{individual\_test\_pvalue}(y\text{ vs. }x; s)$\\
     $P[b] \leftarrow \text{min}(P[b], p)$\;
   }
  }
 }
 Return $quantile_{\alpha_0}(P)$\;
\end{algorithm}

BFAR samples $h_{max}+\tilde{h}$ episodes (where $h_{max}$ is the maximal lookback-horizon) from reference data of $N$ episodes, to simulate sequential data $Y$.
Then individual tests are simulated for any test-point along $\tilde{h}$ episodes, starting after $h_{max}$ episodes.
The minimal $p$-value determines whether a detection would occur in $Y$.
The whole procedure repeats $B$ times, creating a bootstrap estimate of the distribution of the minimal $p$-value along $\tilde{h}$ episodes.
We choose the tests threshold to be the $\alpha_0$-quantile of this distribution, such that $\alpha_0$ of the bootstrap simulations would raise a false alarm.

Note that the statistic for the tests is given to BFAR as an input, making its choice independent of BFAR.
BFAR can run in an offline manner (e.g., a single run before the deployment of the agent). It takes $\mathcal{O}(BF\tilde{h}\tilde{T})$ time, where $\tilde{T}$ is the time of a single update of all the test-statistics.
Additional details are discussed in Appendices~\ref{sec:detailed_sequential_test},\ref{sec:algorithms}.


\section{Experiments}
\label{sec:experiments}

\subsection{Methodology}
\label{sec:methodology}

We run experiments in standard RL environments as described below.
For each environment, we train an agent using the PyTorch version~\citep{pytorchrl} of OpenAI's baseline~\citep{openai_baselines} of A2C algorithm~\citep{A3C}.
We let the trained agent run in the environment for $N_0$ episodes and record its rewards, considered the \textit{trusted reference data}.
We then define several scenarios, and let the agent run for $M \times N$ episodes in each scenario (divided later into $M=100$ blocks of $N$ episodes).
One scenario is named $H_0$ and is identical to the reference up to the random initial-states. The other scenarios are defined per environment, and present environmental changes expected to harm the agent's rewards.
The agent is \textit{not} trained to adapt to these changes, and the goal is to test how long it takes for a degradation-test to detect the degradation.

Individual degradation-tests of length $n$ (Setup~\ref{setup:ind}) are applied for every scenario over the first $n$ time-steps of each block.
Sequential degradation-tests (Setup~\ref{setup:seq}) are applied sequentially over the episodes of each block.
Since the agent is assumed to run continuously as the environment changes from $H_0$ to an alternative scenario, each block is preceded by a random sample of $H_0$ episodes, as demonstrated in Fig.~\ref{fig:sequential_setup}.



\begin{figure*}
\centering
\includegraphics[width=0.7\linewidth]{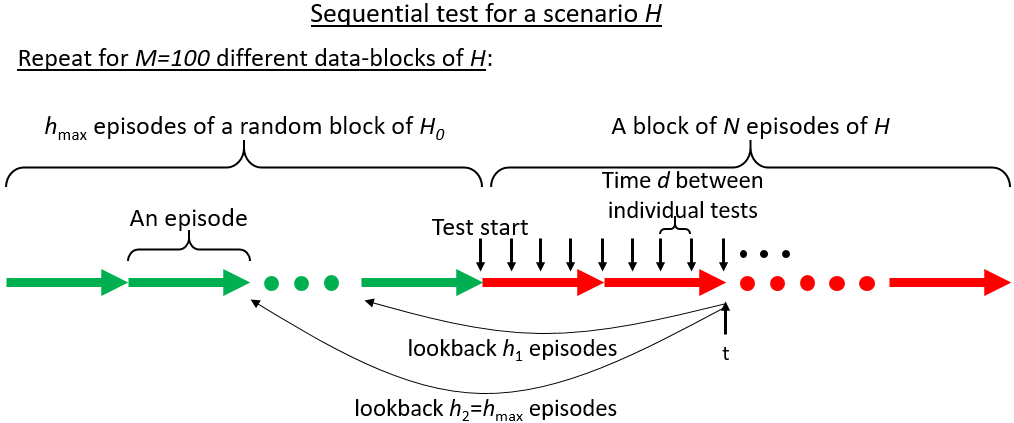}
\caption{\footnotesize A summary of the sequential degradation-test procedure described in Section~\ref{sec:methodology}.}
\label{fig:sequential_setup}
\end{figure*}

\begin{table}[b]
\vspace{-5pt}
\centering
\caption{Environments parameters \\ \footnotesize{(episode length ($T$), reference episodes ($N_0$), test blocks ($M$), episodes per block ($N$), sequential test length ($\tilde{h}$), lookback horizons ($h_1,h_2$), tests per episode ($F=T/d$))}}
\label{tab:envs}
\begin{tabular}{|l|c|c|c|c|c|c|}
\hline
Environment & $T$ & $N_0$ & $M$ & $N=\tilde{h}$ & $h_{1,2}$ & $F$ \\
\hline\hline
Pendulum & 200 & 3e3 & 100 & 30 & 3,30 & 20 \\
\hline
HalfCheetah & 1000 & 1e4 & 100 & 50 & 5,50 & 40 \\
\hline
Humanoid & 200 & 5e3 & 100 & 30 & 3,30 & 10 \\
\hline
\noalign{\smallskip}
\end{tabular}
\end{table}

\begin{figure*}[!b]
    \centering
    \includegraphics[width=1.\linewidth]{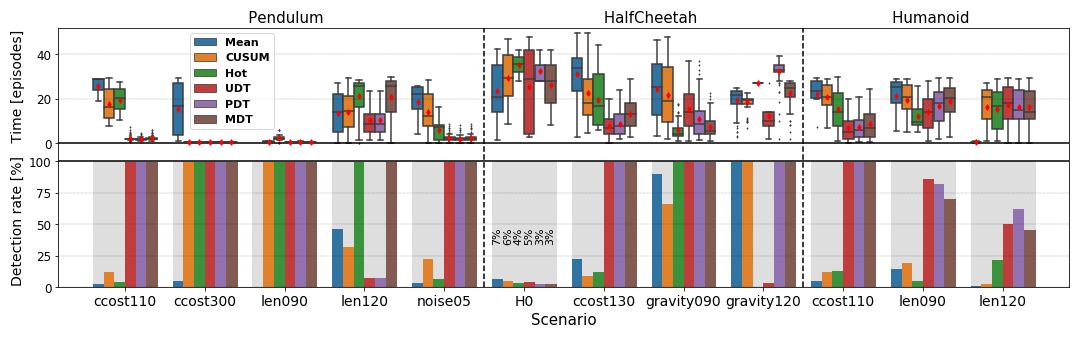}
    \caption{\footnotesize {\bf Bottom}: percent of sequential tests that ended with degradation detection (high is good), over $M=100$ runs with different seeds, for 3 standard tests and 3 variants of our test ({\bf UDT}, {\bf PDT} and {\bf MDT}), in a sample of scenarios in Pendulum, HalfCheetah and Humanoid.
    {\bf Top}: time until detection (low is good) -- for the runs that ended with detection.
    The significance of the tests is shown for HalfCheetah in $H_0$ scenario (and for Pendulum and Humanoid as well in Fig.~\ref{fig:H0} in Appendix~\ref{sec:detailed_results}).
    }
    \label{fig:seq_res}
\end{figure*}

BFAR adjusts the tests thresholds to have a false alarm with probability $\alpha_0=5\%$ per $\tilde{h}=N$ episodes (where $N$ is the data-block size).
Two lookback-horizons $h_1,h_2$ are chosen for every environment.
The rewards are downsampled by a factor of $d$ before applying the tests, intended to reduce the parameters estimation error.
Table~\ref{tab:envs} summarizes the setup of the various environments.

The experimented degradation-tests are a threshold-test on the simple {\bf Mean}; {\bf CUSUM}~\citep{CUSUM_book}; {\bf Hotelling}~\citep{Hotelling}; {\bf UDT} and {\bf PDT} (with $p=0.9$) from Section~\ref{sec:optimization}; and a Mixed Degradation Test ({\bf MDT}) that runs Mean, Hotelling and PDT in parallel -- applying all three in every test-point (as permitted in Algorithm~\ref{algo:sequential_bootstrap}).
All the degradation-tests are tuned according to the same reference data.
Further implementation details are discussed in Appendix~\ref{sec:experiments_implementation}.




\subsection{Results}
\label{sec:results}

We run the tests in the environments of Pendulum~\citep{Pendulum}, where the goal is to keep a pendulum pointing upwards; HalfCheetah~\citep{mujoco}, where the goal is for a 2D cheetah to run as fast as possible; and Humanoid, where the goal is for a person to walk without falling.
In each environment we define the scenario \textit{ccostx} of control cost increased to x\% of its original value, as well as changed-dynamics scenarios specified in Appendix~\ref{sec:experiments_implementation}.


In all the environments the rewards are clearly \textit{not} independent, identically-distributed or normal (see Fig.~\ref{fig:cheetah_params} for example).
Yet the false alarm rates are close to $\alpha_0=5\%$ per $\tilde{h}$ episodes in all the tests, as demonstrated in Fig.~\ref{fig:seq_res} (and in more details in Fig.~\ref{fig:H0} in Appendix~\ref{sec:detailed_results}).
These results under $H_0$ indicate that {\bf BFAR tunes the thresholds properly} in spite of the complexity of the data.
Note that BFAR never observed the data of scenario $H_0$ -- only the reference data.

{\bf In most of the non-$H_0$ scenarios, our tests prove to be more powerful than the standard tests, often by extreme margins}.
For example, increased control cost in all the environments and additive noise in Pendulum are all 100\%-detected by the suggested tests, usually within few episodes (Fig.~\ref{fig:seq_res}); whereas Mean, CUSUM and Hotelling have very poor detection rates. Mean did not detect degradation in Pendulum even after the control cost increased from 110\% to 300\%(!), while keeping the significance level constant ($\alpha_0=5\%$).

Note that we run the tests with two lookback-horizons in parallel, as allowed by BFAR.
This proves useful: with +30\% control cost in HalfCheetah, for example, the short lookback-horizon allows fast detection of degradation; but with merely +10\%, the long horizon is necessary to notice the slight degradation over a large number of episodes.
This is demonstrated in Fig.~\ref{fig:horizons} in Appendix~\ref{sec:detailed_results}.

Covariance-based tests reduce the weights of the highly-varying (and presumably noisier) time-steps. 
In HalfCheetah they turn out to be in the later parts of the episode.
As a result, in certain scenarios, Mean, CUSUM and Hotelling (which do not exploit the different variances optimally) do better in individual tests of 100 samples (out of $T=1000$) than they do in one or even 10 full episodes (see Fig.~\ref{fig:hc_ind_ccost130} in Appendix~\ref{sec:detailed_results}).
This does not occur in UDT and PDT.
Essentially, we see that {\bf ignoring the noise variability leads to violation of the principle that more data are better}.

In Pendulum, the ratio between variance of different steps may reach 5 orders of magnitude.
This phenomenon increases the potential power of the covariance-based tests.
For example, when the pole is shortened, negative changes in the highly-weighted time-steps are detected even when the mean of the whole signal increases.
This feature allows us to detect slight changes in the environment before they develop into larger changes and cause damage.

On the other hand, a challenging situation arises when certain rewards decrease but the highly-weighted ones slightly increase (as in longer Pendulum's pole), which strongly violates the assumptions of Section~\ref{sec:optimization}.
UDT is doomed to falter in such scenarios.
PDT proves somewhat robust to this phenomenon since it is capable of focusing on a subset of time-steps, as demonstrated in increased gravity in HalfCheetah (Fig.~\ref{fig:seq_res}).
However, it cannot overcome the extreme weights differences in Pendulum.
The one test that demonstrated robustness to all the experimented scenarios, including modified Pendulum's length and mass, is MDT. MDT combines Mean, Hotelling and PDT and does not fall far behind any of the three, in any of the scenarios.
Hence, it presents excellent results in some scenarios and reasonable results in the others.

The tests were run on a single i9-10900X CPU core.
BFAR (which needs to run only once and in an offline manner -- before the deployment of the agent) took around 30 minutes per environment and test-statistic (several hours in total). Any parallelization should accelerate the bootstrap linearly with the number of cores.
The sequential (online) tests themselves ran for 10 minutes per scenario -- for all the 6 test-statistics together and for thousands of episodes.

Detailed experiments results are available in Appendix~\ref{sec:detailed_results}.
The code of the experiments is available on \href{https://github.com/ido90/Rewards-Deterioration-Detection}{\underline{GitHub}}.


\section{Related Work}
\label{sec:related_work}

{\bf Training in non-stationary environments} has been widely researched, in particular in the frameworks of Multi-Armed Bandits~\citep{MAB_changepoints_detection,MAB_switch_detection,MAB_non_stationary,mab_adversarial_scaling,mab_multiplayer,mab_adversarial_corruption,mab_adversarial_attack}, model-based RL~\citep{nonstationary_mdp,context_awareness} and general multi-agent environments~\citep{multiagent_envs_survey}.
\citet{Quickest_change_detection_MDP} explicitly detect changes in the environment and modify the policy accordingly, but assume that the environment is Markov, fully-observable, and its transition model is known -- three assumptions that we avoid and that do not hold in many real-world problems.
Safe exploration during training in RL was addressed by~\citet{rl_safety_survey,lyapunov_safety,safety_constrained_rl,safe_rl_end2end,rl_shielding}.
Note that our work refers to changes beyond the scope of the training phase: it addresses the stage where the agent is fixed and required not to train further, in particular not in an online manner.
Robust algorithms may prevent degradation in the first place, but when they fail -- or when their assumptions are not met -- an external model-free monitor with minimal assumptions (as the one suggested in this work) is crucial.

{\bf Sequential tests} were addressed by many over the years.
Common approaches rely on strong assumptions such as samples independence~\citep{CUSUM,CUSUM_book} and normality~\citep{Pocock,ObrienFleming}.
Generalizations exist for certain private cases~\citep{CUSUM_AR1,changepoint_corr_noise}, sometimes at cost of alternative assumptions such as known change-size~\citep{changepoint_detection}.
Samples independence is usually assumed also in recent works with numeric approaches~\citep{MaxSPRT,nonparametric_sprt,concept_drift}, and is often justified by consolidating many samples (e.g., an episode) together as a single sample~\citep{rl_algos_comparison}.
\citet{learning_in_nonstationary_env} wrote that "change detection is typically carried out by inspecting i.i.d features extracted from the incoming data stream, e.g., the sample mean".
Certain works address cyclic signals monitoring~\citep{cyclic_signal_monitoring}, but to the best of our knowledge, we are the first to devise an optimal test for mean change in temporal non-i.i.d signals, and a false alarm control mechanism for such non-i.i.d signals.

Our work can be seen in part as converting a univariate temporal episodic signal into a $T$-dimensional multivariate signal.
Many works addressed the problem of {\bf changepoint detection in multivariate variables}, e.g., using histograms comparison~\citep{QuantTree}, Hotelling statistic~\citep{Hotelling}, and K-L distance~\citep{cdt_multivariate}.
Hotelling in particular also looks for changed mean under unchanged covariance.
However, unlike existing tests, we derive optimal tests for two different negative mean-change hypotheses, intended to detect degradation in temporal signals.
Indeed, Section~\ref{sec:experiments} demonstrates the advantage over Hotelling in such a context.
In addition, by considering the temporal nature of the signal, we are able to handle "incomplete observations" and in particular obtain detections even within the middle of the first episode.



\section{Summary}
\label{sec:summary}

We introduced a novel approach that is optimal (under certain conditions) for detection of changes in episodic signals, exploiting the correlations structure as measured in a reference dataset.
In environments of classic control (Pendulum) and MuJoCo (HalfCheetah, Humanoid), the suggested statistical tests detected degradation faster than alternatives, often by orders of magnitude.
Certain conditions, such as combination of positive and negative changes in very heterogeneous signals, may cause instability in some of the suggested tests; however, this is shown to be solved by running the new test in parallel to standard tests -- with only a small loss of test power.

We also introduced BFAR, a bootstrap mechanism that adjusts tests thresholds according to the desired false alarm rate in sequential tests.
The mechanism empirically succeeded in providing valid thresholds for various tests in all the environments, in spite of the non-i.i.d data.

The suggested approach may contribute to development of reliable RL-based systems.
Future research may
consider different hypotheses, such as a permitted small degradation (instead of $H_0$) or a mix of degradation and improvement (instead of $H_A$);
suggest additional stabilizing mechanisms for covariance-based tests;
exploit other metrics than rewards for tests on model-based RL systems;
and apply comparative tests of episodic signals beyond the scope of sequential change detection.


\section*{Acknowledgements}
This work was partially funded by the Israel Science Foundation (ISF).
The authors wish to thank Guy Tennenholtz and Nadav Merlis for their helpful insights.


\bibliographystyle{icml2021}
\nocite{*}
\bibliography{refs}

\begin{thebibliography}{78}
\providecommand{\natexlab}[1]{#1}
\providecommand{\url}[1]{\texttt{#1}}
\expandafter\ifx\csname urlstyle\endcsname\relax
  \providecommand{\doi}[1]{doi: #1}\else
  \providecommand{\doi}{doi: \begingroup \urlstyle{rm}\Url}\fi

\bibitem[Abhishek \& Mannor(2017)Abhishek and Mannor]{nonparametric_sprt}
Abhishek, V. and Mannor, S.
\newblock A nonparametric sequential test for online randomized experiments.
\newblock \emph{Proceedings of the 26th International Conference on World Wide
  Web Companion}, pp.\  610--6, 2017.

\bibitem[Alatur et~al.(2020)Alatur, Levy, and Krause]{mab_multiplayer}
Alatur, P., Levy, K.~Y., and Krause, A.
\newblock Multi-player bandits: The adversarial case.
\newblock \emph{JMLR}, 2020.

\bibitem[Alshiekh(2017)]{rl_shielding}
Alshiekh, M.
\newblock Safe reinforcement learning via shielding.
\newblock \emph{Logic in Computer Science}, 2017.

\bibitem[Alt et~al.(2019)Alt, Sosic, and Koeppl]{correlations_prior}
Alt, B., Sosic, A., and Koeppl, H.
\newblock Correlation priors for reinforcement learning.
\newblock \emph{NeurIPS}, 2019.

\bibitem[Aminikhanghahi \& Cook(2016)Aminikhanghahi and
  Cook]{changepoint_survey}
Aminikhanghahi, S. and Cook, D.
\newblock A survey of methods for time series change point detection.
\newblock \emph{Knowledge and Information Systems}, 51:\penalty0 339--367,
  2016.

\bibitem[Badia et~al.(2020)]{agent57}
Badia, A.~P. et~al.
\newblock Agent57: Outperforming the atari human benchmark.
\newblock \emph{ICML}, 2020.

\bibitem[Banerjee et~al.(2016)Banerjee, Liu, and
  How]{Quickest_change_detection_MDP}
Banerjee, T., Liu, M., and How, J.
\newblock Quickest change detection approach to optimal control in markov
  decision processes with model changes, 09 2016.

\bibitem[Bellman(1957)]{MDP_Bellman}
Bellman, R.
\newblock A markovian decision process.
\newblock \emph{Indiana Univ. Math. J.}, 6:\penalty0 679--684, 1957.
\newblock ISSN 0022-2518.

\bibitem[Berry \& Fristedt(1985)Berry and Fristedt]{MAB}
Berry, D.~A. and Fristedt, B.
\newblock \emph{Bandit problems}.
\newblock Springer Netherlands, 1985.
\newblock \doi{10.1007/978-94-015-3711-7}.

\bibitem[Besbes et~al.(2014)Besbes, Gur, and Zeevi]{MAB_non_stationary}
Besbes, O., Gur, Y., and Zeevi, A.
\newblock Stochastic multi-armed-bandit problem with non-stationary rewards.
\newblock \emph{Advances in Neural Information Processing Systems (NIPS)}, 27,
  2014.

\bibitem[Boracchi et~al.(2018)Boracchi, Carrera, Cervellera, and
  Maccio]{QuantTree}
Boracchi, G., Carrera, D., Cervellera, C., and Maccio, D.
\newblock Quanttree: Histograms for change detection in multivariate data
  streams.
\newblock \emph{Proceedings of Machine Learning Research}, 80:\penalty0
  639--648, 10--15 Jul 2018.
\newblock URL \url{http://proceedings.mlr.press/v80/boracchi18a.html}.

\bibitem[Brockman et~al.(2016)Brockman, Cheung, Pettersson, Schneider,
  Schulman, Tang, and Zaremba]{gym}
Brockman, G., Cheung, V., Pettersson, L., Schneider, J., Schulman, J., Tang,
  J., and Zaremba, W.
\newblock Openai gym, 2016.

\bibitem[Brook et~al.(1972)]{CUSUM_history1}
Brook, D. et~al.
\newblock An approach to the probability distribution of cusum run length.
\newblock \emph{Biometrika}, 59(3):\penalty0 539--549, 1972.

\bibitem[Bylander()]{MDP_and_episodes}
Bylander, T.
\newblock Lecture notes: Reinforcement learning.
\newblock
  \url{http://www.cs.utsa.edu/~bylander/cs6243/reinforcement-learning.pdf}.

\bibitem[Chan et~al.(2020)]{RL_reliability}
Chan, S.~C. et~al.
\newblock Measuring the reliability of reinforcement learning algorithms.
\newblock \emph{ICLR}, 2020.

\bibitem[Chen(2020)]{CVaR_investopedia}
Chen, J.
\newblock Conditional value at risk (cvar).
\newblock
  \url{https://www.investopedia.com/terms/c/conditional_value_at_risk.asp},
  2020.

\bibitem[Cheng et~al.(2019)]{safe_rl_end2end}
Cheng, R. et~al.
\newblock End-to-end safe reinforcement learning through barrier functions for
  safety-critical continuous control tasks.
\newblock \emph{AAAI Conference on Artificial Intelligence}, 2019.

\bibitem[Chow et~al.(2018)]{lyapunov_safety}
Chow, Y. et~al.
\newblock A lyapunov-based approach to safe reinforcement learning.
\newblock \emph{NIPS}, 2018.

\bibitem[Colas et~al.(2019)Colas, Sigaud, and Oudeyer]{rl_algos_comparison}
Colas, C., Sigaud, O., and Oudeyer, P.-Y.
\newblock A hitchhiker's guide to statistical comparisons of reinforcement
  learning algorithms, 2019.

\bibitem[Dai et~al.(2013)Dai, Ding, and Wahba]{MV_Bernoulli}
Dai, B., Ding, S., and Wahba, G.
\newblock Multivariate bernoulli distribution.
\newblock \emph{Bernoulli}, 19\penalty0 (4):\penalty0 1465--1483, 09 2013.
\newblock \doi{10.3150/12-BEJSP10}.
\newblock URL \url{https://doi.org/10.3150/12-BEJSP10}.

\bibitem[Dhariwal et~al.(2017)Dhariwal, Hesse, Klimov, Nichol, Plappert,
  Radford, Schulman, Sidor, Wu, and Zhokhov]{openai_baselines}
Dhariwal, P., Hesse, C., Klimov, O., Nichol, A., Plappert, M., Radford, A.,
  Schulman, J., Sidor, S., Wu, Y., and Zhokhov, P.
\newblock Openai baselines.
\newblock \url{https://github.com/openai/baselines}, 2017.

\bibitem[Dickey \& Fuller(1979)Dickey and Fuller]{DickeyFuller}
Dickey, D.~A. and Fuller, W.~A.
\newblock Distribution of the estimators for autoregressive time series with a
  unit root.
\newblock \emph{Journal of the American Statistical Association}, 74\penalty0
  (366a):\penalty0 427--431, 1979.
\newblock \doi{10.1080/01621459.1979.10482531}.
\newblock URL \url{https://doi.org/10.1080/01621459.1979.10482531}.

\bibitem[Ditzler et~al.(2015)Ditzler, Polikar, and
  Alippi]{learning_in_nonstationary_env}
Ditzler, G., Polikar, R., and Alippi, C.
\newblock Learning in nonstationary environments: A survey.
\newblock \emph{IEEE Computational Intelligence Magazine}, 2015.

\bibitem[Dulac-Arnold et~al.(2019)Dulac-Arnold, Mankowitz, and
  Hester]{real_world_rl_challenges2}
Dulac-Arnold, G., Mankowitz, D., and Hester, T.
\newblock Challenges of real-world reinforcement learning, 2019.

\bibitem[Efron(2003)]{bootstrap_2003}
Efron, B.
\newblock Second thoughts on the bootstrap.
\newblock \emph{Statist. Sci.}, 18\penalty0 (2):\penalty0 135--140, 05 2003.
\newblock \doi{10.1214/ss/1063994968}.
\newblock URL \url{https://doi.org/10.1214/ss/1063994968}.

\bibitem[Freedman(2017)]{MC_convergence}
Freedman, A.
\newblock Convergence theorem for finite markov chains, 2017.
\newblock URL
  \url{https://math.uchicago.edu/~may/REU2017/REUPapers/Freedman.pdf}.

\bibitem[Garcia \& Fernandez(2015)Garcia and Fernandez]{rl_safety_survey}
Garcia, J. and Fernandez, F.
\newblock A comprehensive survey on safe reinforcement learning.
\newblock \emph{JMLR}, 2015.

\bibitem[Garivier \& Moulines(2011)Garivier and Moulines]{MAB_switch_detection}
Garivier, A. and Moulines, E.
\newblock On upper-confidence bound policies for switching bandit problems.
\newblock \emph{International Conference on Algorithmic Learning Theory}, pp.\
  174--188, 10 2011.
\newblock \doi{10.1007/978-3-642-24412-4_16}.

\bibitem[Goldman(2008)]{bonferroni}
Goldman, M.
\newblock Lecture notes in stat c141: The bonferroni correction.
\newblock https://www.stat.berkeley.edu/~mgoldman/Section0402.pdf, 2008.

\bibitem[Gupta et~al.(2019)Gupta, Koren, and
  Talwar]{mab_adversarial_corruption}
Gupta, A., Koren, T., and Talwar, K.
\newblock Better algorithms for stochastic bandits with adversarial
  corruptions.
\newblock \emph{Proceedings of Machine Learning Research}, 2019.

\bibitem[Harel et~al.(2014)Harel, Crammer, El-Yaniv, and Mannor]{concept_drift}
Harel, M., Crammer, K., El-Yaniv, R., and Mannor, S.
\newblock Concept drift detection through resampling.
\newblock \emph{International Conference on Machine Learning}, pp.\
  II–1009–II–1017, 2014.

\bibitem[Henderson et~al.(2017)]{DRL_that_matters}
Henderson, P. et~al.
\newblock Deep reinforcement learning that matters.
\newblock \emph{AAAI}, 2017.

\bibitem[Hernandez-Leal et~al.(2019)Hernandez-Leal, Kaisers, Baarslag, and
  de~Cote]{multiagent_envs_survey}
Hernandez-Leal, P., Kaisers, M., Baarslag, T., and de~Cote, E.~M.
\newblock A survey of learning in multiagent environments: Dealing with
  non-stationarity, 2019.

\bibitem[Hessel et~al.(2018)Hessel, Modayil, van Hasselt, Schaul, Ostrovski,
  Dabney, Horgan, Piot, Azar, and Silver]{drl_improvements}
Hessel, M., Modayil, J., van Hasselt, H., Schaul, T., Ostrovski, G., Dabney,
  W., Horgan, D., Piot, B., Azar, M., and Silver, D.
\newblock Rainbow: Combining improvements in deep reinforcement learning.
\newblock \emph{AAAI}, 2018.

\bibitem[Hotelling(1931)]{Hotelling}
Hotelling, H.
\newblock The generalization of student's ratio.
\newblock \emph{Ann. Math. Statist.}, 2\penalty0 (3):\penalty0 360--378, 08
  1931.
\newblock \doi{10.1214/aoms/1177732979}.
\newblock URL \url{https://doi.org/10.1214/aoms/1177732979}.

\bibitem[Irwin(2006)]{CLT2}
Irwin, M.~E.
\newblock Lecture notes: Convergence in distribution and central limit theorem.
\newblock \url{http://www2.stat.duke.edu/~sayan/230/2017/Section53.pdf}, 2006.

\bibitem[Jun et~al.(2018)]{mab_adversarial_attack}
Jun, K.-S. et~al.
\newblock Adversarial attacks on stochastic bandits.
\newblock \emph{NeurIPS}, 2018.

\bibitem[Junges et~al.(2016)]{safety_constrained_rl}
Junges, S. et~al.
\newblock Safety-constrained reinforcement learning for mdps.
\newblock \emph{International Conference on Tools and Algorithms for the
  Construction and Analysis of Systems}, 2016.

\bibitem[K.~V.~Mardia \& Bibby(1979)K.~V.~Mardia and Bibby]{Wishart}
K.~V.~Mardia, J. T.~K. and Bibby, J.~M.
\newblock \emph{Multivariate analysis}.
\newblock Academic Press, 1979.

\bibitem[Kharitonov et~al.(2015)Kharitonov, Vorobev, Macdonald, Serdyukov, and
  Ounis]{MaxSPRT}
Kharitonov, E., Vorobev, A., Macdonald, C., Serdyukov, P., and Ounis, I.
\newblock Sequential testing for early stopping of online experiments.
\newblock \emph{Proceedings of the 38th International ACM SIGIR Conference on
  Research and Development in Information Retrieval}, pp.\  473–482, 2015.
\newblock \doi{10.1145/2766462.2767729}.
\newblock URL \url{https://doi.org/10.1145/2766462.2767729}.

\bibitem[Korenkevych et~al.(2019)Korenkevych, Mahmood, Vasan, and
  Bergstra]{autoregressive_policies}
Korenkevych, D., Mahmood, A.~R., Vasan, G., and Bergstra, J.
\newblock Autoregressive policies for continuous control deep reinforcement
  learning, 2019.

\bibitem[Kostrikov(2018)]{pytorchrl}
Kostrikov, I.
\newblock Pytorch implementations of reinforcement learning algorithms.
\newblock \url{https://github.com/ikostrikov/pytorch-a2c-ppo-acktr-gail}, 2018.

\bibitem[Kroese et~al.(2014)Kroese, Brereton, Taimre, and Botev]{MonteCarlo}
Kroese, D.~P., Brereton, T., Taimre, T., and Botev, Z.
\newblock Why the monte carlo method is so important today.
\newblock \emph{Wiley Interdisciplinary Reviews: Computational Statistics},
  6:\penalty0 386--392, 2014.

\bibitem[{Kuncheva}(2013)]{cdt_multivariate}
{Kuncheva}, L.~I.
\newblock Change detection in streaming multivariate data using likelihood
  detectors.
\newblock \emph{IEEE Transactions on Knowledge and Data Engineering},
  25\penalty0 (5):\penalty0 1175--1180, 2013.
\newblock \doi{10.1109/TKDE.2011.226}.

\bibitem[Lan(1994)]{alpha_spending}
Lan, D. L. D. K. K.~G.
\newblock Interim analysis: The alpha spending function approach.
\newblock \emph{Statistics in Medicine}, 13:\penalty0 1341--52, 1994.

\bibitem[Lecarpentier \& Rachelson(2019)Lecarpentier and
  Rachelson]{nonstationary_mdp}
Lecarpentier, E. and Rachelson, E.
\newblock Non-stationary markov decision processes: a worst-case approach using
  model-based reinforcement learning.
\newblock \emph{NeurIPS 2019}, abs/1904.10090, 2019.
\newblock URL \url{http://arxiv.org/abs/1904.10090}.

\bibitem[Lee et~al.(2020)]{context_awareness}
Lee, K. et~al.
\newblock Context-aware dynamics model for generalization in model-based rl.
\newblock \emph{ICML}, 2020.

\bibitem[Lu \& Jr.(2001)Lu and Jr.]{CUSUM_AR1}
Lu, C.-W. and Jr., M. R.~R.
\newblock Cusum charts for monitoring an autocorrelated process.
\newblock \emph{Journal of Quality Technology}, 33\penalty0 (3):\penalty0
  316--334, 2001.
\newblock \doi{10.1080/00224065.2001.11980082}.
\newblock URL \url{https://doi.org/10.1080/00224065.2001.11980082}.

\bibitem[Lund et~al.(2007)Lund, Wang, Lu, Reeves, Gallagher, and
  Feng]{changepoint_detection}
Lund, R., Wang, X.~L., Lu, Q.~Q., Reeves, J., Gallagher, C., and Feng, Y.
\newblock {Changepoint Detection in Periodic and Autocorrelated Time Series}.
\newblock \emph{Journal of Climate}, 20\penalty0 (20):\penalty0 5178--5190, 10
  2007.
\newblock ISSN 0894-8755.
\newblock \doi{10.1175/JCLI4291.1}.
\newblock URL \url{https://doi.org/10.1175/JCLI4291.1}.

\bibitem[Lykouris et~al.(2020)Lykouris, Mirrokni, and
  Leme]{mab_adversarial_scaling}
Lykouris, T., Mirrokni, V., and Leme, R.~P.
\newblock Bandits with adversarial scaling.
\newblock \emph{ICML}, 2020.

\bibitem[MathWorks()]{CVaR_matlab}
MathWorks.
\newblock Conditional value-at-risk (cvar).
\newblock
  \url{https://www.mathworks.com/discovery/conditional-value-at-risk.html}.

\bibitem[Matsushima et~al.(2020)Matsushima, Furuta, Matsuo, Nachum, and
  Gu]{DeploymentEfficientRL}
Matsushima, T., Furuta, H., Matsuo, Y., Nachum, O., and Gu, S.
\newblock Deployment-efficient reinforcement learning via model-based offline
  optimization.
\newblock \emph{ArXiv}, abs/2006.03647, 2020.

\bibitem[Mnih et~al.(2016)Mnih, Badia, Mirza, Graves, Lillicrap, Harley,
  Silver, and Kavukcuoglu]{A3C}
Mnih, V., Badia, A.~P., Mirza, M., Graves, A., Lillicrap, T., Harley, T.,
  Silver, D., and Kavukcuoglu, K.
\newblock Asynchronous methods for deep reinforcement learning.
\newblock \emph{Proceedings of Machine Learning Research}, 48:\penalty0
  1928--1937, 20-22 Jun 2016.

\bibitem[MuJoCo()]{HalfCheetah}
MuJoCo.
\newblock Halfcheetah-v2.
\newblock \url{https://gym.openai.com/envs/HalfCheetah-v2/}.

\bibitem[Mukherjee \& Maillard(2019)Mukherjee and
  Maillard]{MAB_changepoints_detection}
Mukherjee, S. and Maillard, O.-A.
\newblock Distribution-dependent and time-uniform bounds for piecewise i.i.d
  bandits.
\newblock \emph{arXiv preprint arXiv:1905.13159}, 2019.

\bibitem[Murphy et~al.(2001)Murphy, van~der Laan, and
  Robins]{medical_decision_making}
Murphy, S.~A., van~der Laan, M.~J., and Robins, J.~M.
\newblock Marginal mean models for dynamic regimes.
\newblock \emph{Journal of the American Statistical Association}, 2001.

\bibitem[Nachum et~al.(2020)Nachum, Ahn, Ponte, Gu, and Kumar]{multi_agent_rl}
Nachum, O., Ahn, M., Ponte, H., Gu, S.~S., and Kumar, V.
\newblock Multi-agent manipulation via locomotion using hierarchical sim2real.
\newblock \emph{PMLR}, 100:\penalty0 110--121, 30 Oct--01 Nov 2020.
\newblock URL \url{http://proceedings.mlr.press/v100/nachum20a.html}.

\bibitem[NCSS()]{CUSUM_description}
NCSS.
\newblock Cumulative sum (cusum) charts.
\newblock
  \url{https://ncss-wpengine.netdna-ssl.com/wp-content/themes/ncss/pdf/Procedures/NCSS/CUSUM_Charts.pdf}.

\bibitem[Neyman et~al.(1933)Neyman, Pearson, and Pearson]{NeymanPearson}
Neyman, J., Pearson, E.~S., and Pearson, K.
\newblock On the problem of the most efficient tests of statistical hypotheses.
\newblock \emph{Philosophical Transactions of the Royal Society of London},
  1933.
\newblock \doi{10.1098/rsta.1933.0009}.

\bibitem[O'Brien \& Fleming(1979)O'Brien and Fleming]{ObrienFleming}
O'Brien, P.~C. and Fleming, T.~R.
\newblock A multiple testing procedure for clinical trials.
\newblock \emph{Biometrics}, 35\penalty0 (3):\penalty0 549--556, 1979.
\newblock ISSN 0006341X, 15410420.
\newblock URL \url{http://www.jstor.org/stable/2530245}.

\bibitem[OpenAI()]{Pendulum}
OpenAI.
\newblock Pendulum-v0.
\newblock \url{https://gym.openai.com/envs/Pendulum-v0/}.

\bibitem[Page(1954)]{CUSUM}
Page, E.~S.
\newblock {Continuous Inspection Schemes}.
\newblock \emph{Biometrika}, 41\penalty0 (1-2):\penalty0 100--115, 06 1954.
\newblock ISSN 0006-3444.
\newblock \doi{10.1093/biomet/41.1-2.100}.
\newblock URL \url{https://doi.org/10.1093/biomet/41.1-2.100}.

\bibitem[Pardo et~al.(2017)Pardo, Tavakoli, Levdik, and
  Kormushev]{RL_time_limits}
Pardo, F., Tavakoli, A., Levdik, V., and Kormushev, P.
\newblock Time limits in reinforcement learning.
\newblock \emph{CoRR}, abs/1712.00378, 2017.
\newblock URL \url{http://arxiv.org/abs/1712.00378}.

\bibitem[{PennState College of Science}()]{alpha_spending_notes}
{PennState College of Science}.
\newblock Lecture notes in stat 509: Alpha spending function approach.
\newblock \url{https://online.stat.psu.edu/stat509/node/81/}.

\bibitem[Petrov(1972)]{CLT1}
Petrov, V.~V.
\newblock \emph{Sums of Independent Random Variables}.
\newblock Nauka, 1972.

\bibitem[Pocock(1977)]{Pocock}
Pocock, S.~J.
\newblock Group sequential methods in the design and analysis of clinical
  trials.
\newblock \emph{Biometrika}, 64\penalty0 (2):\penalty0 191--199, 08 1977.
\newblock ISSN 0006-3444.
\newblock \doi{10.1093/biomet/64.2.191}.
\newblock URL \url{https://doi.org/10.1093/biomet/64.2.191}.

\bibitem[Rockafellar \& Uryasev(2000)Rockafellar and
  Uryasev]{CVaR_optimization_finance}
Rockafellar, R.~T. and Uryasev, S.
\newblock Optimization of conditional value-at-risk.
\newblock \emph{Journal of Risk}, 2:\penalty0 21--41, 2000.
\newblock \doi{10.21314/JOR.2000.038}.

\bibitem[Ryan(2011)]{CUSUM_book}
Ryan, T.~P.
\newblock \emph{{Statistical Methods for Quality Improvement}}.
\newblock Wiley; 3rd Edition, 2011.

\bibitem[{Todorov} et~al.(2012){Todorov}, {Erez}, and {Tassa}]{mujoco}
{Todorov}, E., {Erez}, T., and {Tassa}, Y.
\newblock Mujoco: A physics engine for model-based control.
\newblock \emph{2012 IEEE/RSJ International Conference on Intelligent Robots
  and Systems}, pp.\  5026--5033, 2012.

\bibitem[Wald(1945)]{SPRT}
Wald, A.
\newblock Sequential tests of statistical hypotheses.
\newblock \emph{Annals of Mathematical Statistics}, 16\penalty0 (2):\penalty0
  117--186, 06 1945.
\newblock \doi{10.1214/aoms/1177731118}.
\newblock URL \url{https://doi.org/10.1214/aoms/1177731118}.

\bibitem[Westgard et~al.(1977)Westgard, Groth, Aronsson, and
  Verdier]{CUSUM_app2}
Westgard, J., Groth, T., Aronsson, T., and Verdier, C.
\newblock Combined shewhart-cusum control chart for improved quality control in
  clinical chemistry.
\newblock \emph{Clinical chemistry}, 23:\penalty0 1881--7, 11 1977.
\newblock \doi{10.1093/clinchem/23.10.1881}.

\bibitem[Wilks(1938)]{wilks1938}
Wilks, S.~S.
\newblock The large-sample distribution of the likelihood ratio for testing
  composite hypotheses.
\newblock \emph{Ann. Math. Statist.}, 9\penalty0 (1):\penalty0 60--62, 03 1938.
\newblock \doi{10.1214/aoms/1177732360}.
\newblock URL \url{https://doi.org/10.1214/aoms/1177732360}.

\bibitem[Williams et~al.(1992)]{CUSUM_app1}
Williams, S.~M. et~al.
\newblock Quality control: an application of the cusum.
\newblock \emph{BMJ: British medical journal}, 304.6838:\penalty0 1359, 1992.

\bibitem[Xie \& Siegmund(2011)Xie and Siegmund]{changepoint_corr_noise}
Xie, Y. and Siegmund, D.
\newblock Weak change-point detection using temporal correlation, 2011.

\bibitem[Yashchin(1985)]{CUSUM_history2}
Yashchin, E.
\newblock On the analysis and design of cusum-shewhart control schemes.
\newblock \emph{IBM Journal of Research and Development}, 29\penalty0
  (4):\penalty0 377--391, 1985.

\bibitem[Yu et~al.(2020)Yu, Thomas, Yu, Ermon, Zou, Levine, Finn, and Ma]{MOPO}
Yu, T., Thomas, G., Yu, L., Ermon, S., Zou, J., Levine, S., Finn, C., and Ma,
  T.
\newblock Mopo: Model-based offline policy optimization, 2020.

\bibitem[Zhao et~al.(2019)]{auto_vehicle_reliability}
Zhao, X. et~al.
\newblock Assessing the safety and reliability of autonomous vehicles from road
  testing.
\newblock \emph{ISSRE}, 2019.

\bibitem[Zhou et~al.(2005)Zhou, Jin, and Jin]{cyclic_signal_monitoring}
Zhou, S., Jin, N., and Jin, J.~J.
\newblock Cycle-based signal monitoring using a directionally variant
  multivariate control chart system.
\newblock \emph{IIE Transactions}, 37\penalty0 (11):\penalty0 971--982, 2005.
\newblock \doi{10.1080/07408170590925553}.
\newblock URL \url{https://doi.org/10.1080/07408170590925553}.

\end{thebibliography}


\onecolumn
\appendix

\newpage
\setcounter{tocdepth}{1}
\tableofcontents

\newpage
\listoftheorems

\newpage


\section{Detailed Preliminary Materials}
\label{sec:detailed_preliminaries}

\subsection{Reinforcement Learning and Episodic Framework}
\label{sec:bg_rl}

The environment of a Reinforcement Learning (RL) problem is usually modeled as a \textit{Decision Process}.
This is essentially a state-machine, where the (possibly random) transition between states depends on decision-making, as well as on the current and the previous states (in the general case).
Every state (and possibly every decision) is assigned a corresponding reward,
and the goal of the decision-making system (termed \textit{agent}) is to maximize some function of the rewards, named the \textit{return function}.
In contrast to Supervised Learning, the feedback from the environment does not inform the agent whether it succeeded to maximize the rewards, but merely how high the rewards were.
It is up to the agent to explore the possible decisions (also termed \textit{actions}) and the corresponding rewards.

The return function is usually a simple sum of the rewards for a finite process, and a decayed sum for an infinite process.
In the finite case, the process usually repeats multiple times in different variants, e.g., with different initial states.
Common examples are board and video games~\citep{gym}, as well as more realistic problems such as repeating drives in autonomous driving task.
In the context of RL, the repetitions of the decision process are usually named \textit{episodes}.
Bylander~\citep{MDP_and_episodes} defined an episode as the "path from initial to a terminal state".
Pardo et al.~\citep{RL_time_limits} wrote that "it is common to let an agent interact for a fixed amount of time with its environment before resetting it and repeating the process in a series of episodes".

Note that once the agent chooses a decision-making scheme (termed \textit{policy}), the decision process essentially reduces to a (decision-free) random process.
Every time-step in the process has a certain distribution of (state and) reward, and different time-steps may depend on each other.

The decision process in RL is often modeled as a \textit{Markov Decision Process} (MDP)~\citep{MDP_Bellman}, where every state depends only on the preceding state and the agent's action.
The decision-free process received from an MDP with relation to a fixed policy is a \textit{Markov Chain} (MC), which under certain further assumptions is guaranteed to converge into a stationary state~\citep{MC_convergence}.
However, even in such a restrictive model, long-term correlations between rewards may still carry information if the states are not observable by the agent; and even under the further conditions of convergence to a stationary state, the rate of convergence may be slow compared to the length of an episode.
The non-stationarity of the rewards within an episode is demonstrated, for example, in Fig.~\ref{fig:cheetah_var}.

This work exploits the repetitive nature of the episodic random processes -- and in particular the rewards of episodic decision processes in the context of RL -- to estimate the expectations and the correlations in the process.
Since we measure the rewards directly, without considering the underlying states or any other observations available to the agent, we may call this approach model-free in the context of RL.

Note that in the scope of this work, the goal of the episodes is to provide i.i.d samples of a non-i.i.d random process, so that the covariance parameters of the process can be estimated.
Hence, the scope of "episodic problems" may be quite extensive: it may include even life-time systems that run continuously without ever resetting -- as long as a reference dataset of other instances of the system is available, and the sample resolution does not introduce too many covariance parameters to estimate from the reference dataset.
Indeed, the model defined in Section~\ref{sec:detailed_model} and the optimality results in Section~\ref{sec:detailed_optimal_test} are fully capable of handling a part of a single, long episode (with the exception of the asymptotic results in Section~\ref{sec:unif_asymptotics}).


\subsection{Hypothesis Testing}
\label{sec:bg_hypothesis_tests}

In a standard hypothesis test, two hypotheses are formulated regarding some observable phenomenon, and we wish to decide which one is true according to available evidence, given in the form of observations $X \in \mathbb{X}$ from a corresponding observation space $\mathbb{X}$.
One hypothesis is often regarded as the default, named the \textit{Null Hypothesis} and denoted $H_0$; and given $X$ we have to decide whether to \textit{reject} $H_0$ in favor of the \textit{Alternative Hypothesis} $H_A$.

The fundamental distinction between the hypotheses lays on their different probabilistic models $P\left(X \big| H\right)$ (either probability function or probability density function), also referred to as the \textit{likelihood} $L\left(H\big|X\right)$ of the hypothesis given the observations.
The difference between the models is often formulated in terms of different values of a parameter $\theta$ for some parametric probability function $P\left(X\big|\theta\right)$.
A \textit{complex} hypothesis is one that allows different possible probabilistic models, represented by a set $\Theta$ of permitted values of $\theta$.
The likelihood of a complex hypothesis $H: \theta \in \Theta$ is defined as $L\left(H\big|X\right) = \text{sup}_{\theta\in\Theta}P\left(X\big|\theta\right)$.
The \textit{likelihood-ratio} between two hypotheses is defined as $LR\left(H_0,H_A\big|X\right) = \frac{L\left(H_0\big|X\right)}{L\left(H_A\big|X\right)}$.
The log-likelihood-ratio is often used instead~\citep{wilks1938}, since it tends to derive simpler expressions for exponential families of distributions such as the Normal distribution.
In this work we often denote $\lambda_{LR}\left(H_0,H_A\big|X\right) = 2\text{ln}\left(LR\right)$.

The basic metrics for the efficiency of a hypothesis test are its \textit{significance} $P\left(\text{not reject } H_0 \big| H_0 \right)=1-P\left(\text{type-I error}\right)=1-\alpha$ and its \textit{power} $P\left(\text{reject } H_0 \big| H_A \right)=1-P\left(\text{type-II error}\right)=\beta$.
A statistical hypothesis test with significance $1-\alpha$ and power $\beta$ is said to be \textit{optimal} if any statistical test with as high significance $1-\tilde{\alpha} \ge 1-\alpha$ has smaller power $\tilde{\beta} \le \beta$.

According to Neyman-Pearson lemma~\citep{NeymanPearson}, a threshold-test on the likelihood ratio is an optimal hypothesis test.
In a likelihood-ratio threshold-test with a threshold $\kappa\in\mathbb{R}$, we reject $H_0$ if $LR\left(H_0,H_A\big|X\right)<\kappa$; reject with a certain probability $\rho\in[0,1]$ if $LR=\kappa$; and do not reject $H_0$ otherwise.
Note that the behavior in the edge-case $LR=\kappa$ (controlled by $\rho$) only matters in the case of non-continuous distributions, where it is possible that $P\left(LR=\kappa\right) \ne 0$.

Note that the optimal hypothesis test is not unique, but rather leaves a degree of freedom in the tradeoff between $\alpha$ and $\beta$. In the case of a threshold-test, this degree of freedom is controlled by the threshold $\kappa$ (and the edge probability $\rho$).
It is common to define the test according to a desired significance level (often $\alpha=0.01$ or $\alpha=0.05$), and derive the corresponding threshold $\kappa_\alpha$.

In certain cases, given a test-statistic and desired $\alpha$, the threshold $\kappa_\alpha$ can be analytically calculated from the corresponding probabilistic model $P\left(X \big| H_0\right)$.
If the model is too complex or not well-defined, but expresses the sum of i.i.d random variables, then according to the \textit{Central Limit Theorem} (CLT)~\citep{CLT1,CLT2}, the model becomes closer to a Normal distribution as the number of summed variables grows, allowing to analytically calculate the asymptotic value of $\kappa_\alpha$.
Note that the CLT lays on the independence and identical distributions of the summed variables -- two properties which are not generally satisfied by episodic rewards in the decision processes described in Section~\ref{sec:bg_rl}.

Numeric methods are also available for estimation of properties of a hypothesis test (or the properties of a statistic of the observations).
In \textit{Monte-Carlo method}~\citep{MonteCarlo}, the test is simulated (or the statistic is computed) multiple times for observations $X$ generated in a way which is assumed to be similar to a hypothesis $H$ (in particular $H_0$ for significance estimation).
In the \textit{bootstrap} method~\citep{bootstrap_2003}, given i.i.d observations $X\in\mathbb{R}^n$ (assumed to be drawn according to a hypothesis $H$), Monte-Carlo method is applied on artificial observations $X_b$ drawn by repeatedly sampling $n$ elements from $X$ with replacement.


\subsection{Sequential Tests}
\label{sec:bg_sequential_tests}

Section~\ref{sec:bg_hypothesis_tests} describes the general scheme of a standard hypothesis test for distinction between two hypotheses according to certain available data.
In many practical applications, the hypothesis test is repeatedly applied as the data change or grow, a procedure known as a \textit{sequential test}.
If the null hypothesis $H_0$ is true, and any individual hypothesis test falsely rejects $H_0$ with some probability $\alpha$, then the probability that at least one of the multiple tests will reject $H_0$ is $\alpha_0 > \alpha$, termed \textit{family-wise type-I error rate}.
For simplicity, consider the private case of $k$ independent tests, where $\alpha_0 = 1-(1-\alpha)^k \xrightarrow{\enskip k\rightarrow \infty \enskip} 1$.

This problem, also known as inflation of significance or inflation of $\alpha$ in sequential tests, was addressed by many over the years.
A simple solution is the \textit{Bonferroni correction}~\citep{bonferroni}, setting significance level of $1-\alpha/k$ in every individual test. This way, we have $P(\exists i: \text{test $i$ rejects} | H_0) \le \sum_{i=1}^k \alpha/k = \alpha$. However, the inequality becomes equality only if the rejections of the various tests are disjoint events (not even independent); thus in practice we often have $\alpha_0 \ll \alpha$, which makes the Bonferroni correction extremely conservative.
Appendix~\ref{sec:detailed_related_work} describes other relevant works on sequential testing.


\section{Related Work: Detailed Discussion}
\label{sec:detailed_related_work}

As explained in Section~\ref{sec:preliminaries}, sequential tests repeatedly apply individual hypothesis tests with certain significance level $1-\alpha \in (0,1)$.
The probability that at least one test would reject the null hypothesis $H_0$ increases with the number of the individual tests, leading to "inflation of $\alpha$" and decreased family-wise significance level $1-\alpha_0 < 1-\alpha$.
Section~\ref{sec:sequential_test} discusses this problem in the context of tests on episodic signals.
Here we discuss some of the existing methods for design of sequential tests.

\paragraph{Sequential Probability Ratio Test (SPRT):}
SPRT~\citep{SPRT} considers a symmetric approach between two hypotheses $H_1$,$H_2$, and aims to decide between them as fast as possible, subject to the probability of a wrong decision being bounded by $\alpha$.
The decision rule is chosen such that the expected time until decision is minimized.
The element that bounds the probability of wrong decision is the setup of the flow of the test.
Every iteration, the decision rule decides between three possibilities: accept $H_1$, accept $H_2$, or continue.
The possibility to stop on acceptance of the true hypothesis limits the inflation of $\alpha$.

In contrast to this setup, in the change-point detection problem -- where continuously looking for changes -- we either reject $H_0$ or continue, but never stop to accept $H_0$.
Dedicated sequential tests are designed for the problem of change-point detection.

\paragraph{Cumulative Sum test (CUSUM):}
The CUSUM test~\citep{CUSUM,CUSUM_description} is a well-studied~\citep{CUSUM_history1,CUSUM_history2} and very popular method in quality control and change detection~\citep{CUSUM_app1,CUSUM_app2}.
While being useful in a wide scope of problems, the test requires the size of change to be defined in advance as a parameter (a requirement that exists in other methods as well~\citep{changepoint_detection}).
In addition, CUSUM assumes to observe i.i.d samples.
The statistic is defined incrementally in a non-linear way, making it more difficult to generalize to non-i.i.d models, although several generalizations do exist, e.g., for the case of first-order autoregressive signal AR(1)~\citep{CUSUM_AR1}.
However, for example, Fig.~\ref{fig:cheetah_acf} demonstrates empiric rewards in HalfCheetah environment~\citep{HalfCheetah}, where the dependencies in the signal require a more expressive model.

\paragraph{Persistent drift and Dickey-Fuller test:}
Certain methods are available for detection of persistent drifts (also known as trends) in time-series.
For example, Dickey-Fuller test~\citep{DickeyFuller} for unit-roots in autoregressive models essentially looks for linear drifts.
However, in the scope of this work we do not assume a persistent drift, nor limit ourselves to autoregressive models.

\paragraph{$\alpha$-spending functions:}
The $\alpha$-spending functions~\citep{alpha_spending,alpha_spending_notes} deal with the inflation of $\alpha$ in sequential tests by conceptually referring to $\alpha$ as a limited budget of significance, where every individual test spends some of the budget.
Due to the dependence between the individual tests, the total budget spent is smaller than the sum of the individual spends $\alpha_0 < \sum_i \alpha_i$.
Thus, careful calculations are required for tuning of the family-wise significance level $\alpha_0$.

\citet{Pocock}, for example, showed how to calculate a constant individual significance level $\alpha_i \equiv \alpha$ given a desired family-wise significance $\alpha_0$ and known number of $k$ individual tests, assuming that the tests are applied to accumulated normal i.i.d data samples.
For many applications, such a constant significance level tends to spend too much $\alpha$-budget in the first individual tests, reducing too much power from the later tests -- where most of the data are available.
It is often preferred to keep high significance level for these final tests, and reject $H_0$ in earlier tests only in radical cases.
Accordingly, the O'Brien-Fleming function~\citep{ObrienFleming} determines the individual significance levels $\{\alpha_i\}_{i=1}^k$ under similar i.i.d and normality assumptions as Pocock, but lets $\alpha_i$ gradually increase over the sequential test.
In Section~\ref{sec:sequential_test} we consider the $\alpha$-spending approach and generalize it through a bootstrap mechanism to handle any sequence of individual tests for the case of episodic data; that is, i.i.d episodes consisting of samples which are not assumed to be independent, normal, or identically-distributed.

\paragraph{Multivariate mean shift:}
In a way, our work can be seen as a test for change-point or mean-shift of i.i.d $T$-dimensional multivariate random variables -- the episodes.
This problem was addressed before, e.g., using Hotelling statistic~\citep{Hotelling}, histograms comparison~\citep{QuantTree}, and K-L distance~\citep{cdt_multivariate}.
However, our setup has two essential differences from the multivariate mean-shift problem:
first, since we look for a signed (negative) change in a univariate signal, we form the test's alternative hypothesis $H_A$ correspondingly. This results in the uniform and partial degradation hypotheses, which are essentially different from the alternative hypothesis of Hotelling test, for example.
Indeed, Section~\ref{sec:experiments} demonstrates the advantage over Hotelling in the framework of RL, that is, episodic univariate rewards signal.

Second, since the episodic signal is \textit{temporal} univariate, the coordinates of the "multivariate variables" are not observed simultaneously. As a result, when observing in the middle of an episode, we have incomplete information about the last multivariate variable (and possibly the first one, depending on how the lookback-horizon is defined).
Both BFAR and the test statistics in this work take care of this issue.
This is required for correct inference at any mid-episode time, but is particularly important for fast detection of large changes -- which should be detected in the middle of the first episode.

\paragraph{Numeric methods:}
\citet{rl_algos_comparison} address the problem of comparing different RL algorithms, referring to whole episode as a single data sample for the tests.
\citet{concept_drift} apply permutations test to detect changes in i.i.d data, focusing on drifts that impair predictive models of the data.
The bootstrap mechanism discussed in Section~\ref{sec:sequential_test} can be seen as a permutations test on i.i.d episodes (instead of single samples).
\citet{nonparametric_sprt} also bring together ideas from bootstrap and sequential tests to construct a nonparametric sequential hypothesis test.
The test applies bootstrap on single samples within blocks of data, assuming the data samples are i.i.d.
Certain machine-learning based approaches were also suggested for changepoint detection in time-series~\citep{changepoint_survey}.
\citet{learning_in_nonstationary_env} wrote that "change detection is typically carried out by inspecting independent and identically distributed (i.i.d) features extracted from the incoming data stream, e.g., the sample mean, the sample variance, and/or the classification error".

\paragraph{Changing environment and safety in RL:}
In Multi-Armed Bandits (MAB)~\citep{MAB}, where by default each bandit (action) yields i.i.d rewards, several works address the problem of regret minimization (namely, optimization of rewards during training) with abrupt changes~\citep{MAB_switch_detection,MAB_changepoints_detection}, gradual changes~\citep{MAB_non_stationary} and even adversarial changes~\citep{mab_adversarial_scaling,mab_multiplayer,mab_adversarial_corruption,mab_adversarial_attack}.

Training in presence of non-stationary environment was also considered in other environments such as multi-agent environments~\citep{multiagent_envs_survey} and in model-based RL with varying model~\citep{nonstationary_mdp,Quickest_change_detection_MDP}.
Several works addressed the problem of safety in exploration of RL algorithms during training~\citep{rl_safety_survey,lyapunov_safety,safety_constrained_rl}, often using model-based learning of the environment~\citep{safe_rl_end2end} or specified constraints~\citep{rl_shielding}.

Note that our work refers to changes beyond the scope of the training phase, at the stage where the agent is fixed and required not to train further, in particular not in an online manner.
Robust algorithms may prevent rewards degradation in the first place, but when they do not -- it is crucial to be alerted.
To the best of our knowledge, we are the first to exploit correlations between rewards in post-training phase to test for changes in both model-based and model-free RL.


\section{Extended Definitions and Model Discussions}
\label{sec:detailed_model}

\paragraph{Episodic signal model:}
Below is the formal definition of an episodic signal, as discussed in Section~\ref{sec:detailed_model}.

\begin{definition}[Episodic index decomposition]
\label{def:index_decomposition}
Let $t,T \in \mathbb{N}$. We define $k(t,T) \coloneqq \lfloor \frac{t-1}{T} \rfloor$, $\tilde{\tau}(t,T) \coloneqq t \text{ (mod } T)$, and $\tau(t,T) \coloneqq \begin{cases} T & \text{if } \tilde{\tau}(t,T)=0 \\ \tilde{\tau}(t,T) & \text{if } \tilde{\tau}(t,T)\ne 0 \end{cases}$.
When no confusion is risked, we may simply write $k=k(t,T), \tau=\tau(t,T)$.
Note that $\forall t,T \in \mathbb{N}: t = kT+\tau$.
\end{definition}

\begin{definition}[$T$-long episodic signal; an extended formulation of Definition~\ref{def:episodic_signal}]
\label{def:episodic_signal_copy}
Let $n,T \in \mathbb{N}$.
Denote $K=k(n,T)$, $\tau_0=\tau(n,T)$ according to Definition~\ref{def:index_decomposition}.
A sequence of real-valued random variables $\{X_t\}_{t=1}^{n}$ is a \textit{$T$-long episodic signal}, if its joint probability density distribution can be written as
\begin{align}
\begin{split}
    f_{\{X_t\}_{t=1}^{n}} & (x_1,...,x_n) = \\
    & \left[ \prod_{k=0}^{K-1} f_{\{X_t\}_{t=1}^{T}}(x_{kT+1},...,x_{kT+T}) \right] \cdot f_{\{X_t\}_{t=1}^{\tau_0}}(x_{KT+1},...,x_{KT+\tau_0})
\end{split}
\end{align}
(where in the edge case $K=0$ we define the empty product to be 1).
We further denote $\pmb{\mu_0} \coloneqq E[(X_1,...,X_T)^\top] \in \mathbb{R}^T, \Sigma_0 \coloneqq \text{Cov}((X_1,...,X_T)^\top, (X_1,...,X_T)) \in \mathbb{R}^{T \times T}$.
\end{definition}

\paragraph{Expectation and covariance of an episodic signal:}
The expectations and covariance matrix of a whole episodic signal can be directly derived from the parameters $\pmb{\mu_0},\Sigma_0$ corresponding to the expectations and covariance matrix of a single episode.

\begin{proposition}[Expectation and covariance of an episodic signal]
\label{prop:episodic_signal_covariance}
Let $\{X_t\}_{t=1}^{n}$ be a $T$-long episodic signal with parameters $\pmb{\mu_0},\Sigma_0$.
The expectations $\pmb{\mu} \coloneqq E[(X_1,...,X_n)^\top] \in \mathbb{R}^n$ and covariance matrix $\Sigma \coloneqq \text{Cov}((X_1,...,X_n)^\top, (X_1,...,X_n)) \in \mathbb{R}^{n \times n}$ are uniquely determined by $\pmb{\mu_0}$ and $\Sigma_0$, respectively.
\end{proposition}

\begin{proof}
For any $t \in \{1,...,n\}$, denote $t=kT+\tau$ according to Definition~\ref{def:index_decomposition}.
From Eq.~\eqref{eq:episodic_dist} it is clear that $\forall t_1 = k_1 T + \tau_1, t_2 = k_2 T + \tau_2 \in \{1,...,n\}:$
\begin{align}
    \label{eq:episodic_signal_covariance}
    \begin{split}
\mu_{t_1} &= E[X_{t_1}] = (\pmb{\mu_0})_{\tau_1} \\
\Sigma_{t_1t_2} &= \text{Cov}(X_{t_1},X_{t_2}) = \begin{cases} (\Sigma_0)_{\tau_1\tau_2} & \text{if } k_1=k_2 \\ 0 & \text{if } k_1\ne k_2 \end{cases}
    \end{split}
\end{align}
\end{proof}

Proposition \ref{prop:episodic_signal_covariance} essentially means that $\pmb{\mu}$ consists of periodic repetitions of $\pmb{\mu_0}$, and $\Sigma$ consists of copies of $\Sigma_0$ as $T\times T$ blocks along its diagonal. For both parameters, the last repetition is cropped if $\tau(n,T) < T$.

\paragraph{Multivariate normal episodic signal:}
Some of the theoretical results in Section~\ref{sec:detailed_optimal_test} assume multivariate normality of the episodic signal. The formal definition of such a signal is given below.

\begin{definition}[Multivariate normal $T$-long episodic signal]
\label{def:multivariate_normal}
Let $\{X_t\}_{t=1}^{n}$ be a $T$-long episodic signal (Definition~\ref{def:episodic_signal_copy}).
For any $1\le \tau \le \text{min}(T,n)$, define $\pmb{\mu_\tau} \in \mathbb{R}^\tau$ to be the first $\tau$ elements of $\pmb{\mu_0}$ and $\Sigma_\tau \in \mathbb{R}^{\tau \times \tau}$ to be the upper-left $\tau \times \tau$ block of $\Sigma_0$.
The signal $\{X_t\}_{t=1}^{n}$ is \textit{multivariate normal} if $\forall 1\le \tau \le \text{min}(T,n)$, 
\begin{equation}
\label{eq:multivariate_normal}
    f_{X_1,...,X_\tau}(\pmb{x}) = (2\pi)^{-\tau/2} \text{det}(\Sigma_\tau)^{-1/2} e^{-\frac{1}{2} (\pmb{x}-\pmb{\mu_\tau})^\top \Sigma_\tau^{-1} (\pmb{x}-\pmb{\mu_\tau})}
\end{equation}
\end{definition}

From Definitions~\ref{def:episodic_signal_copy},\ref{def:multivariate_normal} it is clear that if $\{X_t\}_{t=1}^{n}$ form a multivariate normal $T$-long episodic signal, then in particular $X=(X_1,...,X_n)^\top \in \mathbb{R}^n$ is an $n$-dimensional multivariate normal variable, with expectations $\pmb{\mu}$ and covariance $\Sigma$ determined by Eq.~\eqref{eq:episodic_signal_covariance}.

\paragraph{Parameters estimation:}
As mentioned above, a possible way to estimate the parameters $\pmb{\mu_0},\Sigma_0$ of an episodic signal is to compute the mean vector and the covariance matrix of a dataset $\{x_{i\tau} | 1\le i\le N, 1\le\tau\le T\}$ of $N$ episodes assumed to satisfy Eq.~\eqref{eq:episodic_dist}.
According to the Central Limit Theorem~\citep{CLT1,CLT2}, since the episodes are i.i.d, for any time-step $\tau$ the estimate $(\hat{\pmb{\mu_0}})_\tau = \frac{1}{N}\sum_{i=1}^N x_{i\tau}$ is asymptotically normally-distributed around the true mean $(\pmb{\mu_0})_\tau$ with variance $\frac{\text{Var}((\pmb{\mu_0})_\tau)}{N}$.
Furthermore, in the private case of a multivariate normal signal, the covariance matrix estimate $(\hat{\Sigma_0})_{ij} = \frac{1}{N-1}\sum_{k=1}^N (x_{ik}-\bar{x}_i)(x_{jk}-\bar{x}_j)$ follows Wishart distribution~\citep{Wishart} (up to a factor of $N-1$), with $N-1$ degrees of freedom and variance $\text{Var}((\hat{\Sigma_0})_{ij}) = \frac{1}{N-1}\left( (\Sigma_0)_{ii}(\Sigma_0)_{jj}+(\Sigma_0)_{ij}^2 \right)$.

If $N$ is suspected to be too small for accurate estimation, it is possible to deal with the estimation error of the model parameters through regularization. One possible regularization is assuming absence of correlations between distant time-steps ($\exists\delta\in\mathbb{N},\forall |t_2-t_1|>\delta: (\Sigma_0)_{t_1t_2}=0$). Another is to essentially reduce $T$ through grouping of sequences of time-steps together (as we do in Section~\ref{sec:experiments}, for example).

In the analysis in the following sections we assume that both $\pmb{\mu_0}$ and $\Sigma_0$ are known.

\paragraph{Multidimensional signals:}
For simplicity of the theoretical discussion, we only consider one-dimensional signals: for any $t$, the random variable $X_t$ returns a scalar $x_t\in \mathbb{R}$.
However, a generalization to multidimensional signals ($x_t\in\mathbb{R}^d$) is straight-forward:
A $d$-dimensional $T$-long episodic signal is simply a one-dimensional $(dT)$-long episodic signal, where the observations arrive in groups of $d$ samples per group (i.e., $n$ is always an integer multiplication of $d$).
Since the various dimensions are equivalent to time-steps in the eyes of this model, the correlations between the various dimensions are inherently captured.
Note that for a large number of dimensions, the $\mathcal{O}(d^2T^2)$ degrees of freedom in the model may be impractical to estimate through a reference dataset.


\section{Likelihood-Ratio Test for Drift in Episodic Signal: Formal Development}
\label{sec:detailed_optimal_test}

In this section we look for an optimal hypothesis test for detection of a negative drift in multivariate normal episodic signal (see Definitions~\ref{def:episodic_signal_copy},\ref{def:multivariate_normal}).
The corresponding hypotheses are episodic signal with known parameters ($H_0$), and episodic signal with identical covariance matrix but smaller expected values ($H_A$), as defined below.
By "optimal test" we mean that given the test's significance level (i.e., type-I error rate), it should provide the maximum possible power (i.e., minimum type-II error rate) with respect to $H_A$.
To that end, we calculate the log-likelihood-ratio and use it (up to a monotonous transformation) as a test-statistic according to Neyman-Pearson lemma~\citep{NeymanPearson}.

In Section~\ref{sec:unif_asymptotics}, after proving optimality for a certain negative drift, we eliminate the multivariate-normality assumption and analyze the asymptotic power of the suggested statistical test.
In particular, we show that it is asymptotically superior to a simple threshold-test on the average reward.

Note that in the scope of this section we assume an individual test at a certain point of time. Adjustment of the significance level to sequential tests is handled in Section \ref{sec:sequential_test}.

Formally, the test is defined with respect to some real-valued random variables $X_1,...,X_n$.

\begin{definition}[Null hypothesis]
\label{def:H0}
Let $\{X_t\}_{t=1}^{n}$ be real-valued random variables, and let $T\in \mathbb{N}, \pmb{\mu_0} \in R^T, \Sigma_0 \in R^{T \times T}$.
The null hypothesis $H_0(T,\pmb{\mu_0},\Sigma_0)$ in the scope of this section, is that $\{X_t\}_{t=1}^{n}$ form a $T$-long episodic signal (Definition~\ref{def:episodic_signal_copy}), with known parameters $T,\pmb{\mu_0},\Sigma_0$.
For simplicity, we further assume that $\Sigma_0$ is of full-rank (i.e., invertible).
\end{definition}


We define a standard setup for most of the analysis below, both with and without the multivariate-normality assumption.

\begin{definition}[The standard setup]
\label{def:setup}
In the \textit{standard setup}, we denote by $X=\{X_t\}_{t=1}^n$ a $T$-long episodic signal for some $n,T\in\mathbb{N}$ (Definition~\ref{def:episodic_signal_copy}),
and let the null hypothesis $H_0$ be as in Definition~\ref{def:H0}, with known parameters $\pmb{\mu_0}\in\mathbb{R}^T, \Sigma_0 \in \mathbb{R}^{T\times T}$.

Note that under $H_0$, the complete signal's expectations $\pmb{\mu}\in \mathbb{R}^n$ and covariance matrix $\Sigma \in \mathbb{R}^{n\times n}$ are also known through Proposition~\ref{prop:episodic_signal_covariance}.

We also denote $k(t)=k(t,T),\tau(t)=\tau(t,T)$ as in Definition~\ref{def:index_decomposition}, and in particular $K\coloneqq k(n,T),\tau_0 \coloneqq \tau(n,T)$.
\end{definition}

\begin{definition}[The standard normal setup]
\label{def:normal_setup}
The \textit{standard normal setup} is the standard setup where $X$ is a multivariate-normal episodic signal (Definition~\ref{def:multivariate_normal}).
\end{definition}


\subsection{Uniform Degradation Test}
\label{sec:uniform_degradation}

The general alternative hypothesis we use assumes conservation of the correlations structure of $H_0$, along with decrease in the expectations.

\begin{definition}[General degradation hypothesis]
\label{def:H1_general}
Given the standard setup (Definition~\ref{def:setup}),
let $\mathbb{E} \subseteq \mathbb{R}^T$ s.t. $\forall \pmb{\epsilon_0} \in \mathbb{E}, 1 \le t \le T: (\pmb{\epsilon_0})_t \ge 0$. According to the \textit{$\mathbb{E}$-degradation hypothesis}, denoted $H_A(\mathbb{E})$, there exists $\pmb{\epsilon_0} \in \mathbb{E}$ such that $\{X_t\}_{t=1}^{n}$ form $\tilde{T}$-long episodic signal with the parameters $\tilde{T}=T, \tilde{\Sigma_0}=\Sigma_0$ and $\pmb{\tilde{\mu_0}}=\pmb{\mu_0}-\pmb{\epsilon_0}$.

In particular, according to Eq.~\eqref{eq:episodic_signal_covariance}, the covariance and the mean of the whole signal under $H_A(\mathbb{E})$ are $\tilde{\Sigma}=\Sigma$ and $\pmb{\tilde{\mu}}=\pmb{\mu} - \pmb{\epsilon}$, where $\pmb{\epsilon} = \pmb{\epsilon}(\pmb{\epsilon_0}) \in \mathbb{R}^n$ is a cyclic completion defined by $\forall t=kT+\tau: (\pmb{\epsilon})_{t} \coloneqq (\pmb{\epsilon_0})_{\tau}$.
\end{definition}


Proposition~\ref{prop:lr_general} calculates the log-likelihood-ratio with respect to the hypotheses in Definitions~\ref{def:H0},\ref{def:H1_general}, assuming a multivariate-normal episodic signal.
Still, to derive a concrete statistical test, further assumptions must be applied on $\mathbb{E}$. We begin with the \textit{uniform degradation} assumption, corresponding to a disturbance source that affects the whole signal uniformly. For example, in the context of Reinforcement Learning, such a model may refer to changes in constant costs or action costs, as well as certain environment dynamics whose change influences the various states in a similar way.

\begin{definition}[Uniform degradation hypothesis]
\label{def:H1_unif}
Let $\epsilon_0 > 0$. The uniform degradation hypothesis, denoted $H_A^{unif}(\epsilon_0)$, is a degradation hypothesis $H_A(\mathbb{E})$ with $\mathbb{E} \coloneqq \{ \epsilon \cdot \pmb{1} | \epsilon \ge \epsilon_0 \}$, where $\pmb{1} \coloneqq (1,...,1)^\top \in \mathbb{R}^T$.
\end{definition}

Fig.~\ref{fig:cheetah_degradation} demonstrates the empiric degradation in the rewards of a trained agent in HalfCheetah environment, following changes in gravity, mass, and control-cost (see Table~\ref{tab:scenarios} for details).
It seems that some modifications indeed cause a quite uniform degradation, while in others the degradation is mostly restricted to certain ranges of time. This may be important, in particular if the non-degraded time-steps happen to be assigned large weights by the test, as demonstrated in Section~\ref{sec:results}.
In Section~\ref{sec:partial_degradation} we suggest an alternative model, whose corresponding test is proved in Section~\ref{sec:results} to be more robust to such non-uniform degradation.

We now show that an optimal hypothesis test for detection of uniform degradation in multivariate normal episodic signal is a threshold-test on the weighted-mean of the signal, where the weights are derived from the inverted covariance matrix.

Note that according to Proposition~\ref{prop:episodic_signal_covariance}, the covariance matrix $\Sigma = \Sigma(\Sigma_0) \in \mathbb{R}^{n\times n}$ of the full signal is block-diagonal with the blocks being $\Sigma_0 \in \mathbb{R}^{T\times T}$ (or an upper-left block of $\Sigma_0$). Hence, the inverted $\Sigma$ is given directly by inverting $\Sigma_0$ (and possibly one upper-left block of $\Sigma_0$).

\begin{definition}[Uniform degradation weighted-mean]
\label{def:unif_weighted_mean}
Given the standard setup (Definition~\ref{def:setup}),
the \textit{uniform-degradation weighted-mean} of $X$ is $s_{unif}(X | \Sigma_0) \coloneqq W\cdot X$, where $W \coloneqq \pmb{1}^\top\cdot \Sigma^{-1} \in \mathbb{R}^n$.

Note that the first $KT$ elements of $W$ are $T$-periodic with $\forall 0\le k\le K-1: w_{kT+1},...,w_{kT+T}=\pmb{1}^\top\cdot\Sigma_0^{-1} \in \mathbb{R}^T$.
We define accordingly $W_0 \coloneqq \pmb{1}^\top\cdot\Sigma_0^{-1}$ and $W_{\tau_0} \coloneqq (w_{KT+1},...,w_{KT+\tau_0})^\top = \pmb{1}^\top\cdot\Sigma_{\tau_0}^{-1}$, where $\Sigma_{\tau_0}$ is the upper-left $\tau_0 \times \tau_0$ block of $\Sigma_0$.
\end{definition}

Proposition \ref{prop:unif_consistency} shows the consistency of the uniform-degradation weighted-mean, and Theorem \ref{theorem:unif_optimality2} shows that it derives an optimal hypothesis test for uniform degradation.

\begin{definition}[Threshold test]
\label{def:threshold_test}
Assume the standard setup (Definition~\ref{def:setup}), and let $S: \mathbb{R}^n \rightarrow \mathbb{R}$ (statistic), $\kappa \in \mathbb{R}$ (threshold) and $\rho \in [0,1]$ (edge-case probability).
The corresponding $\kappa$-\textit{threshold-test} is defined as follows:

Given the observations $\forall 1\le t\le n: X_t=x_t\in\mathbb{R}$, calculate the statistic $s=S(x_1,...,x_n)$.
If $s<\kappa$, reject $H_0$.
If $s=\kappa$, reject $H_0$ with probability $p=\rho$ (note that this is only relevant for non-continuous distributions, where $P(S=\kappa)\ne0$).
If $s>\kappa$, do not reject $H_0$.

We denote the significance level of the test $\alpha \coloneqq P(\text{reject }H_0 | H_0)$.
For simplicity, in the discussion below we often omit $\rho$, implicitly assuming continuous distribution of the signal.
\end{definition}

\begin{theorem}[Optimal test for uniform degradation; an extended formulation of Theorem~\ref{theorem:unif_optimality}]
\label{theorem:unif_optimality2}
Assume the standard normal setup (Definition~\ref{def:normal_setup}) with $H_A^{unif}(\epsilon_0)$ of Definition~\ref{def:H1_unif} as the alternative hypothesis, and let $\alpha \in (0,1)$.
Then, there exists $\kappa \in \mathbb{R}$ such that a $\kappa$-threshold-test on the uniform-degradation weighted-mean statistic has the greatest power among all the statistical tests with significance level $\tilde{\alpha} \le \alpha$.
\end{theorem}

\begin{proof}
The proof is available in Appendix \ref{app:deg_calculations}.
Roughly speaking, according to Neyman-Pearson lemma~\citep{NeymanPearson} a threshold-test on the likelihood-ratio is optimal, hence it is sufficient to show that the uniform-degradation weighted-mean $s_{unif}$ is monotonous with the likelihood-ratio.

Note that the likelihood-ratio is taken with respect to a complex hypothesis $H_A^{unif}(\epsilon_0)$ that has a degree of freedom $\epsilon \in [\epsilon_0,\infty)$, where $\epsilon$ depends on $X$.
Some algebraic work is required to show that $\epsilon$ only depends on $X$ through $s_{unif}$, and that the whole likelihood-ratio is monotonous with $s_{unif}$.
\end{proof}

Algorithm~\ref{algo:individual_test} describes the threshold-test in the non-sequential framework.
The uniform-degradation test-statistic (or any other function) can be fed into the algorithm as an input.

As can be seen, the rejection threshold $\kappa_\alpha \in \mathbb{R}$ is chosen according to the desired type-I error rate $\alpha \in (0,1)$, using a bootstrap mechanism described in Algorithm~\ref{algo:individual_bootstrap}. $B$ bootstrap-samples are sampled from a reference dataset of $N$ episodes of the signal, assumed to follow the null hypothesis $H_0$ of Definition~\ref{def:H0}.
For each bootstrap-sample\footnote{As a terminological note, this sampling mechanism can be considered a bootstrap in the sense of distribution estimation from a single dataset using sampling with replacement; or can be merely considered a Monte-Carlo simulation in the sense that the test signal is compared to distribution estimated by an external simulative source (the reference data).} the test-statistic is calculated, yielding a bootstrap-estimate for the distribution of the statistic under $H_0$.
The rejection threshold $\kappa_\alpha$ is set to be the $\alpha$-quantile of the estimated distribution.
If the estimated distribution is close to the true distribution, then we have $P\left( s \le \kappa_\alpha | H_0 \right) \approx P\left( s \le q_\alpha(s | H_0) | H_0 \right) = \alpha$, where $q_\alpha(s | H_0)$ is the $\alpha$-quantile of $s$ under $H_0$.

\subsubsection{Asymptotic analysis in absence of the normality assumption}
\label{sec:unif_asymptotics}

The optimality of the uniform-degradation weighted-mean test (proved in Theorem~\ref{theorem:unif_optimality2}) relies on the assumption that the episodic signal is multivariate normal.
In this section we show that even in absence of the normality assumption, the test while not necessary is asymptotically superior to a standard threshold-test on the average of the signal (though it is not necessarily the optimal test anymore).

Since the episodes in the signal are still assumed to be i.i.d, both a simple mean and the uniform-degradation weighted-mean $s_{unif}$ are asymptotically normal (where $n \rightarrow \infty$ with respect to a constant episode length $T\in\mathbb{N}$).
For simplicity of the asymptotic analysis below, we focus on integer number of episodes, i.e., $n=KT$ and $K\rightarrow \infty$ (rather than $n\rightarrow \infty$).
We also define normalized variants of our statistics, with zero-mean and unit-variance per episode:
\begin{align}
\label{eq:normalized_statistics}
\begin{split}
    &s_{simp}(\{X_t\}_{t=1}^n) \coloneqq \sum_{t=1}^n X_t \\
    &\tilde{s}_{simp}^K \coloneqq \frac{s_{simp}(\{X_t\}_{t=1}^{KT}) - K\cdot E\left[ s_{simp}(\{X_t\}_{t=1}^T) \big| H_0 \right] }{\sqrt{K\cdot\text{Var}(s_{simp}(\{X_t\}_{t=1}^T) \big| H_0)}} \\
    &\tilde{s}_{unif}^K \coloneqq \frac{s_{unif}(\{X_t\}_{t=1}^{KT}) - K\cdot E\left[ s_{unif}(\{X_t\}_{t=1}^T) \big| H_0\right] }{\sqrt{K\cdot\text{Var}(s_{unif}(\{X_t\}_{t=1}^T) \big| H_0)}} \\
\end{split}
\end{align}
Note that Algorithm~\ref{algo:individual_test} is invariant to linear transformation of the statistic, since the test-statistic and the reference bootstrap distribution pass through the same transformation.
Hence, the tests on $s_{simp},s_{unif}$ are equivalent to the tests on $\tilde{s}_{simp},\tilde{s}_{unif}$, respectively.

Since $\tilde{s}_{simp},\tilde{s}_{unif}$ are asymptotically normal with zero-mean and unit-variance under $H_0$, the desired test threshold for sufficiently large $K$ is $\kappa \approx q_\alpha^0$, where $q_\alpha^0$ is the $\alpha$-quantile of the standard normal distribution.
This threshold should be indirectly estimated by Algorithm~\ref{algo:individual_bootstrap}.

Note that the sequential test of Algorithm~\ref{algo:sequential_test} in Section~\ref{sec:sequential_test} applies the individual tests of Algorithm~\ref{algo:individual_test} on a constant number of episodes (defined by the lookback horizon $h$).
Hence, in the context of the sequential tests suggested in this work, the asymptotic analysis in this section refers to a very long lookback horizon, rather than very long running time.
Regardless, as the analysis refers to a varying $n$, we need to generalize the standard setup (that assumes a constant signal length $n$).

\begin{definition}[The rolling setup]
\label{def:rolling_setup}
Let $\{X_t\}_{t\in\mathbb{N}}$ be an infinite sequence of real-valued random variables.
In the \textit{rolling setup}, for any $n\in\mathbb{N}$ we assume the standard setup of Definition~\ref{def:setup} with relation to the variables $\{X_t\}_{t=1}^n$ and the parameters $T,\pmb{\mu_0},\Sigma_0$ (which are independent of $n$).
\end{definition}


We first show that the test threshold $\kappa = q_\alpha^0$ indeed yields asymptotic significance level of $1-\alpha$, and guarantees asymptotic rejection of $H_0$ for uniform degradation of any size $\epsilon>0$.
Note that Algorithm~\ref{algo:individual_test} does not pick $q_\alpha^0$ directly as a threshold, but should estimate it indirectly through Algorithm~\ref{algo:individual_bootstrap}.

\begin{proposition}[Uniform degradation test consistency]
\label{prop:unif_test_consistency}
Given the rolling setup, we define the alternative hypothesis $H_A^\epsilon$ to be that $\forall K \in \mathbb{N}$, the parameters of the signal $\{X_t\}_{t=1}^{KT}$ are $\pmb{\mu_0}-\epsilon\cdot\pmb{1}, \Sigma_0$ (i.e., $H_A(\{\epsilon\pmb{1}\})$ in terms of Definition~\ref{def:H1_general}).
Given a significance level $\alpha\in(0,1)$, we have
\begin{align*}
    \text{lim}_{K\rightarrow \infty} P\left( s \le q_\alpha^0 \big| H_0 \right) &= \alpha \\
    \text{lim}_{K\rightarrow \infty} P\left( s \le q_\alpha^0 \big| H_A^\epsilon \right) &= 1 \\
\end{align*}
for both $s=\tilde{s}_{simp}^K$ and $s=\tilde{s}_{unif}^K$ of Eq.~\eqref{eq:normalized_statistics},
where $q_\alpha^0$ is the $\alpha$-quantile of the standard normal distribution.
\end{proposition}

\begin{proof}
The proof, fully provided in Appendix~\ref{app:deg_calculations}, applies the Central Limit Theorem~\citep{CLT1,CLT2} on the i.i.d episodes.
\end{proof}

Theorem~\ref{theorem:unif_power2} quantifies the asymptotic power of the threshold test for both simple mean and uniform-degradation weighted-mean. To that end, we consider uniform-degradation scaled as $\epsilon \propto \frac{1}{\sqrt{K}}$.
We also denote by $\Phi$ the Cumulative Distribution Function of the standard normal distribution

\begin{theorem}[Uniform degradation test asymptotic power; an extended formulation of Theorem~\ref{theorem:unif_power}]
\label{theorem:unif_power2}
Given the rolling setup, we define the alternative hypothesis $H_A^{\epsilon,K}$ to be that $\forall K \in \mathbb{N}$, the parameters of the signal $\{X_t\}_{t=1}^{KT}$ are $\pmb{\mu_0}-\frac{\epsilon}{\sqrt{K}}\cdot\pmb{1}, \Sigma_0$ (i.e., $H_A(\{\frac{\epsilon}{\sqrt{K}}\pmb{1}\})$ of Definition~\ref{def:H1_general}).
Given a significance level $\alpha\in(0,1)$, we have
\begin{align*}
    \text{lim}_{K\rightarrow \infty} P\left( \tilde{s}_{simp}^K \le q_\alpha^0 \big| H_A^{\epsilon,K} \right) &= \Phi\left(q_\alpha^0 + \frac{\epsilon T}{\sqrt{\pmb{1}^\top \Sigma_0 \pmb{1}}}\right) \\
    \le \Phi\left(q_\alpha^0 + \epsilon \sqrt{\pmb{1}^\top \Sigma_0^{-1} \pmb{1}}\right) &= \text{lim}_{K\rightarrow \infty} P\left( \tilde{s}_{unif}^K \le q_\alpha^0 \big| H_A^{\epsilon,K} \right)
\end{align*}
\end{theorem}

\begin{proof}
Similarly to Proposition~\ref{prop:unif_test_consistency}, the proof applies the Central Limit Theorem on the i.i.d episodes to calculate the asymptotic properties. Full details are provided in Appendix~\ref{app:deg_calculations}.
\end{proof}


Note that while Theorem~\ref{theorem:unif_optimality2} shows optimality of the uniform-degradation weighted-mean test for multivariate-normal episodic signal,
Theorem~\ref{theorem:unif_power2} proves that even in absence of normality the test is asymptotically superior to a threshold-test on the simple mean.

Finally, we quantify the difference of power between the tests.

\begin{definition}[Uniform degradation asymptotic power gain]
\label{def:unif_power_gain}
Given the rolling setup, the \textit{uniform-degradation power gain} is defined to be $G^2 \coloneqq \frac{(\pmb{1}^\top \Sigma_0^{-1} \pmb{1})(\pmb{1}^\top \Sigma_0 \pmb{1})}{T^2}$.

Note that according to Theorem~\ref{theorem:unif_power2}, if the asymptotic power of the simple-mean threshold-test with relation to the alternative hypothesis $H_A^{\epsilon,K}$ is $\Phi(q_\alpha^0+y)$ (where $y\in\mathbb{R}$), then the asymptotic power of the weighted-mean threshold-test is $\Phi(q_\alpha^0+G\cdot y)$.
%
\end{definition}

\begin{proposition}[Uniform degradation test asymptotic power gain]
\label{prop:unif_power_gain2}
Under the setup of Theorem~\ref{theorem:unif_power2},
there exist positive weights $\{w_{ij}\}_{i,j=1}^T$ such that the uniform-degradation power gain is
$$G^2 = 1 + \sum_{i,j=1}^T w_{ij}(\lambda_i-\lambda_j)^2$$
where $\{\lambda_i\}_{i=1}^T$ are the eigenvalues of $\Sigma_0$.
\end{proposition}

\begin{proof}
The result is received from simple algebra after diagonalizing the symmetric positive-definite covariance matrix $\Sigma_0$.
The full details are available in Appendix~\ref{app:deg_calculations}.
\end{proof}

Clearly, the asymptotic power gain $G$ becomes larger as the covariance matrix $\Sigma_0$ introduces more heterogeneous eigenvalues.
Note that in the independent case, the eigenvalues are simply the variances of the different time-steps.
In particular, in the i.i.d case, the variances are identical and the gain becomes $G=1$, which is consistent with the fact that the two tests are equivalent in this case.


\subsection{Partial Degradation Test}
\label{sec:partial_degradation}

Definition~\ref{def:H1_unif} assumes uniform degradation over all the time-steps in every episode.
However, the effects of many environmental changes may focus on certain states (which is translated in our model-free setup into "certain time-steps"). An example is available in Fig.~\ref{fig:cheetah_degradation}, as discussed before.
To model such effects we introduce the \textit{partial degradation hypothesis}.

\begin{definition}[Partial degradation hypothesis]
\label{def:H1_part}
Let $\epsilon > 0, p\in(0,1)$. Define $A_m^T \coloneqq \{ a_0\in\{0,1\}^T | \sum_{t=1}^T (a_0)_t = m\}$ the set of binary vectors with exactly $m$ one-entries.
The partial degradation hypothesis, denoted $H_A^{part}(\epsilon,p)$, is a degradation hypothesis $H_A(\mathbb{E})$ (see Definition~\ref{def:H1_general}) with $\mathbb{E} \coloneqq \{ \epsilon\cdot a_0 | a_0 \in A_{\lceil pT \rceil}^T \}$.
\end{definition}

\paragraph{Interpretation:} As a private case of Definition~\ref{def:H1_general}, Definition~\ref{def:H1_part} assumes conservation of the correlations structure.
One possible interpretation of this assumption is causal relationships (where change in a certain time-step affects any other time-steps correlated with it).
Another possible interpretation is that the partial degradation hypothesis does not restrict the degradation to only $m=\lceil pT \rceil$ time-steps, but rather distributes the degradation from $m$ time-steps to all the episode, according to the same relations that created the correlations in the signal from the first place.

Similarly to the case of uniform-degradation, we can use the likelihood-ratio to derive a test-statistic and prove its approximate optimality with respect to $H_A^{part}(\epsilon,p)$. 

\begin{definition}[Partial degradation mean]
\label{def:part_weighted_mean}
Given the standard setup (Definition~\ref{def:setup}),
denote $\tilde{X} \coloneqq X-\mu$, $\tilde{s} \coloneqq \Sigma^{-1}\tilde{X} \in \mathbb{R}^n$ and $\forall: 1\le \tau \le T: \tilde{S}_\tau \coloneqq \sum_{k=0}^{\lfloor \frac{n-\tau}{T}\rfloor} \tilde{s}_{kT+\tau}$.
Given $a_0 \in \{0,1\}^T$ denote $o(a_0) \coloneqq \{1\le t\le T | (a_0)_t=1\}$.
Given $p\in(0,1)$, the \textit{$p$-partial-degradation mean} of $X$ is $s_{part}(X;p) \coloneqq \text{min}_{a_0\in A_m^T} \sum_{\tau\in o(a_0)} \tilde{S}_\tau$, where $m(p)=\lceil pT \rceil$ and $A_m^T$ is defined as in Definition~\ref{def:H1_part}.
\end{definition}

Note that while $a$ has $|A_m^T| = \binom{T}{m}$ possible values, to compute $s_{part}$ we only need to sum the lowest $m$ values in $\{\tilde{S}_\tau\}_{\tau=1}^T$.


\begin{theorem}[Optimal test for partial degradation]
\label{theorem:part_optimality2}
Consider the standard normal setup (Definition~\ref{def:normal_setup}) with $H_A^{part}(\epsilon,p)$ of Definition~\ref{def:H1_part} as the alternative hypothesis, and let $\alpha \in (0,1)$.
Denote by $P_\alpha$ the largest possible power (with respect to $H_A^{part}(\epsilon,p)$) of a statistical test with significance level $\le \alpha$.
Then, there exists $\kappa \in \mathbb{R}$ such that the power of the $\kappa$-threshold-test on the partial-degradation mean is $P_\alpha - \mathcal{O}(\epsilon)$ (where $\mathcal{O}(\epsilon)$ is defined with relation to $\epsilon\rightarrow0$).
\end{theorem}

\begin{proof}
The proof is provided in Appendix~\ref{app:deg_calculations}.
Similarly to the proof of Theorem~\ref{theorem:unif_optimality2}, it is based on calculation of the log-likelihood-ratio $\lambda_{LR}$ from Lemma~\ref{lemma:unif_optimality} -- after substituting Definition~\ref{def:H1_part}.
The calculation results in $s_{part}$ along with an $\epsilon$-dependent term whose effect on the test power is shown to be $\mathcal{O}(\epsilon)$.
\end{proof}

\paragraph{The parameter $p$ and comparison to CVaR:}
In Definition~\ref{def:part_weighted_mean}, the "weighted mean" completely eliminates $T-m$ of the entries of $\Sigma^{-1}X$, and only sums the most negative ones. Hence it can be interpreted as the well known CVaR statistic~\citep{CVaR_optimization_finance} of $X-\mu$ after transformation to $\Sigma^{-1}$-basis.
CVaR (Conditional Value at Risk) is intended to measure the "risky" tail of a random variable's distribution~\citep{CVaR_investopedia} by estimating its expectation -- conditioned on it being below the $p$-quantile of the distribution. This is done simply by averaging the $p$ "worst" (lowest) values in the corresponding data.
To express risk, the parameter $p$ is often set below 5\%~\citep{CVaR_matlab}.

In our context, however, the relative part of time-steps $p$ represents the scope of degradation rather than extremity of risk, and there is usually no reason to assume that $p\ll1$.
Such an assumption, when misplaced, may cause elimination of necessary information from the statistic.
In fact, $0.9\le p < 1$ is shown in Section~\ref{sec:experiments} to often achieve superior results, presumably because it maintains most of the information while still being able to filter noisy or misleading time-steps (e.g., time-steps with particularly large values).

Note that filtering positive time-steps is not "cheating" in the context of our problem: we essentially apply a monitor which looks for negative changes, and thus positive changes in other time-steps are indeed considered as noise for the sake of our monitor.
If a more symmetric comparison is desired, then the test can be applied twice -- once for negative changes, and once for positive changes.

\paragraph{The dependence on $\epsilon$ and $\mathcal{O}(\epsilon)$ approximation:}
The partial degradation mean is shown to be equal to an optimal test-statistic up to $\mathcal{O}(\epsilon)$.
The approximation is used to handle the dependence of the minimum $$\text{min}_{a_0 \in A_m^T} \left( a^\top\Sigma^{-1}\tilde{X} + 0.5\epsilon (a^\top\Sigma^{-1}a) \right)$$ on $\epsilon$.
Note that for small degradation the second term is indeed negligible, while for larger degradation the distinction between the two hypotheses should pose little challenge to any detection algorithm.

Furthermore, the test is entirely invariant to a constant additive factor (due to the adjustment of the test threshold using the bootstrap in Algorithm~\ref{algo:individual_bootstrap}); hence, the true distortion in the test is not of size $\gamma = 0.5\epsilon (a^\top\Sigma^{-1}a)$, but rather the change in $\gamma$ due to the possibly-changed choice of $a$.
Note that (a) since $\Sigma^{-1}$ is positive-definite, we have $\forall a: a^\top\Sigma^{-1}a>0$, hence the change in $\gamma$ is necessarily smaller than $\gamma$; (b) if the parameter $p$ is close to 100\% (as discussed above), then most of the entries of $a$ are necessarily kept unchanged, further reducing the change in $\gamma$.

If we wish to apply a more formal test, we can define for example $H_A^{part}(p)\coloneqq \exists \epsilon>0: H_A^{part}(\epsilon,p)$ (similarly to Definition~\ref{def:H1_unif} of uniform degradation, for $\epsilon_0=0$), which yields the log-likelihood-ratio $\text{min}_{a\in A_m^T} -\frac{(a^\top \Sigma^{-1}\tilde{X})^2}{a^\top \Sigma^{-1}a}$ s.t. $a^\top \Sigma^{-1}\tilde{X} \le 0$ (after minimization with relation to $\epsilon>0$, similarly to the proof of Theorem~\ref{theorem:unif_optimality2}).
Due to the discrete domain $A_m^T$ of $a$, this becomes a non-linear integer programming problem, which should be solved for every run of the test on new data $X$.

Note that from a practical point of view, a major role of the partial-degradation model is to allow focusing on negative (degrading) entries of $a^\top \Sigma^{-1}\tilde{X}$, and filtering positive ones (as discussed above).
For this role, both the approximate $s_{part}$ and the accurate minimizer of $H_A^{part}(p) = \exists \epsilon>0: H_A^{part}(\epsilon,p)$ above are perfectly qualified, as both would tend to reject positive entries of $a^\top \Sigma^{-1}\tilde{X}$.

In the scope of this work, we stick to the approximately-correct and computationally-simpler partial degradation mean $s_{part}$.


\section{Supplementary Calculations}
\label{app:deg_calculations}

\begin{proposition}[Likelihood ratio with respect to general degradation]
\label{prop:lr_general}
Let the standard normal setup (Definition~\ref{def:normal_setup} with the $\mathbb{E}$-degradation hypothesis $H_A(\mathbb{E})$ (Definition~\ref{def:H1_general}).
Define $\pmb{\epsilon}(\pmb{\epsilon_0})$ as in Definition~\ref{def:H1_general}, and denote $\tilde{X}_t \coloneqq X_t-(\pmb{\mu})_t$.
Then, the log-likelihood $\lambda_{LR}(H_0,H_A | \{X_t\}_{t=1}^{n}) \coloneqq 2\text{ln} (\frac{P(\{X_t\}_{t=1}^{n}|H_0)}{\text{sup}_{H\in H_A}P(\{X_t\}_{t=1}^{n}|H)})$ of $\{X_t\}_{t=1}^{n}$ with respect to (the simple hypothesis) $H_0$ and (the complex hypothesis) $H_A$ is
\begin{align}
\label{eq:lr_general}
\begin{split}
    \lambda_{LR}&(H_0,H_A | \{X_t\}_{t=1}^{n}) = \\
    &\text{min}_{\pmb{\epsilon_0} \in \mathbb{E}} 2(\pmb{\epsilon}(\pmb{\epsilon_0}))^\top\Sigma^{-1}\tilde{X} + (\pmb{\epsilon}(\pmb{\epsilon_0}))^\top\Sigma^{-1}\pmb{\epsilon}(\pmb{\epsilon_0})
\end{split}
\end{align}
\end{proposition}

\begin{proof}
Using the density function of Eq.~\eqref{eq:multivariate_normal}, we have
\begin{align*}
\begin{split}
    \lambda_{LR}&(H_0,H_A | \{X_t\}_{t=1}^{n}) = \\
    &= 2\text{ln} ( \frac{e^{-0.5\tilde{X}^\top\Sigma^{-1}\tilde{X}}}{\text{max}_{\pmb{\epsilon_0} \in \mathbb{E}}e^{-0.5(\tilde{X}+\pmb{\epsilon})^\top\Sigma^{-1}(\tilde{X}+\pmb{\epsilon})}} ) = \\
    &= \text{min}_{\pmb{\epsilon_0} \in \mathbb{E}} (\tilde{X}+\pmb{\epsilon})^\top\Sigma^{-1}(\tilde{X}+\pmb{\epsilon}) - \tilde{X}^\top\Sigma^{-1}\tilde{X} \\
    &= \text{min}_{\pmb{\epsilon_0} \in \mathbb{E}} \pmb{\epsilon}^\top\Sigma^{-1}\tilde{X} + \tilde{X}^\top\Sigma^{-1}\pmb{\epsilon} + \pmb{\epsilon}^\top\Sigma^{-1}\pmb{\epsilon} \\
    &= \text{min}_{\pmb{\epsilon_0} \in \mathbb{E}} 2\pmb{\epsilon}^\top\Sigma^{-1}\tilde{X} + \pmb{\epsilon}^\top\Sigma^{-1}\pmb{\epsilon}
\end{split}
\end{align*}
where the last equality relies on the invariance of the scalar $\tilde{X}^\top\Sigma^{-1}\pmb{\epsilon} \in R$ to the transpose operation, as well as the symmetry of the covariance matrix (and its inverse).
\end{proof}

\begin{proposition}[Expected value of uniform-degradation weighted-mean]
\label{prop:unif_weighted_mean_expectation}
Let $X$ be a $T$-long episodic signal of length $n=KT$ for some $K\in \mathbb{N}$ (i.e., integer number of episodes), with parameters $\pmb{\mu_0}, \Sigma_0$.
The expected value of $\frac{1}{n}s_{unif}(X | \Sigma_0)$ defined in Definition~\ref{def:unif_weighted_mean} is $\frac{1}{T} (W_0 \cdot \pmb{\mu_0})$.
\end{proposition}

\begin{proof}
Using the $T$-periodicity of $W$ and $\pmb{\mu}$ (see Eq.~\eqref{eq:episodic_signal_covariance}), we have $ E[\frac{1}{n}s_{unif}] = E[\frac{1}{n} W \cdot X] = \frac{1}{n} W\cdot\pmb{\mu}(\pmb{\mu_0}) = \frac{1}{n} K \cdot (W_0 \cdot \pmb{\mu_0}) = \frac{1}{T} (W_0 \cdot \pmb{\mu_0}) $.
\end{proof}

\begin{proposition}[Consistency of uniform-degradation weighted-mean]
\label{prop:unif_consistency}
$\frac{1}{n}s_{unif}$ defined in Definition~\ref{def:unif_weighted_mean} is consistent with relation to the expected value $\frac{1}{T} (W_0 \cdot \pmb{\mu_0})$ calculated in Proposition~\ref{prop:unif_weighted_mean_expectation}, i.e., $\forall \epsilon>0: \text{lim}_{n\rightarrow \infty} P \left( | \frac{1}{n}s_{unif} - \frac{1}{T} (W_0 \cdot \pmb{\mu_0}) | \ge \epsilon \right) = 0$.
\end{proposition}

\begin{proof}
The consistency is proven through the L.L.N over the i.i.d episodes, where the last possibly-partial episode becomes negligible in the limit of infinitely-many episodes.

Using the episodic index decomposition of Definition~\ref{def:index_decomposition}, and the notations of $W_0,W_\tau$ from Definition~\ref{def:unif_weighted_mean}, we can write
\begin{align*}
\begin{split}    
\frac{1}{n}s_{unif} &- \frac{1}{T} (W_0 \cdot \pmb{\mu_0}) = \frac{1}{n}WX - \frac{1}{T} (W_0 \cdot \pmb{\mu_0}) = \frac{1}{n} \sum_{t=1}^n [w_t X_t] - \frac{1}{T} (W_0 \cdot \pmb{\mu_0}) \\
&= \left[ \frac{1}{n} \sum_{k=0}^{K-1} \sum_{\tau=1}^T (W_0)_\tau X_{kT+\tau} - \frac{KT}{n} W_0 \pmb{\mu_0} \right] + \left[ \frac{1}{n} \sum_{\tau=1}^{\tau_0} (W_{\tau_0})_\tau X_{kT+\tau} - \frac{\tau_0}{n} W_0 \pmb{\mu_0} \right] \\
&= \frac{1}{n} \sum_{k=0}^{K-1} \sum_{\tau=1}^T (W_0)_\tau (X_{kT+\tau} - (\pmb{\mu_0})_\tau) + \frac{1}{n} \sum_{\tau=1}^{\tau_0} (W_{\tau_0})_\tau (X_{kT+\tau} - (\pmb{\mu_0})_\tau) \\
&= \frac{1}{n} \sum_{k=0}^{K-1} S_k + \frac{1}{n} S^{tail}_{K,\tau_0}
\end{split}
\end{align*}
where $\tau_0 \coloneqq \tau(n,T)$, $S_k \coloneqq \sum_{\tau=1}^T (W_0)_\tau (X_{kT+\tau} - (\pmb{\mu_0})_\tau)$ and $S^{tail}_{K,\tau_0} \coloneqq \sum_{\tau=1}^{\tau_0} (W_{\tau_0})_\tau (X_{kT+\tau} - (\pmb{\mu_0})_\tau)$.

To prove consistency we have to show that $\forall \epsilon>0: \text{lim}_{n\rightarrow \infty} P \left( | \frac{1}{n}s_{unif} - \frac{1}{T} (W_0 \cdot \pmb{\mu_0}) | \ge \epsilon \right) = 0$.
Indeed, given $\epsilon > 0$, we have
\begin{align*}
&\text{lim}_{n\rightarrow \infty} P \left( \left| \frac{1}{n}s_{unif} - \frac{1}{T} (W_0 \cdot \pmb{\mu_0}) \right| \ge \epsilon \right) \\
&\le \text{lim}_{n\rightarrow \infty} P \left( \left| \frac{1}{n} \sum_{k=0}^{K-1} S_k \right| \ge \epsilon/2 \lor \left|\frac{1}{n} S^{tail}_{K,\tau_0}\right| \ge \epsilon/2 \right) \\
&\le \text{lim}_{n\rightarrow \infty} P \left( \left| \frac{1}{n} \sum_{k=0}^{K-1} S_k \right| \ge \epsilon/2 \right) + P \left( \left|\frac{1}{n} S^{tail}_{K,\tau_0}\right| \ge \epsilon/2 \right)
\end{align*}
where $P \left( \left| \frac{1}{n} \sum_{k=0}^{K-1} S_k \right| \ge \epsilon/2 \right) \rightarrow 0$ according to the Law of Large Numbers applied to the i.i.d sequence $\{S_k\}$; and
\begin{align*}
    \text{lim}_{n\rightarrow \infty} & P \left( \left|\frac{1}{n} S^{tail}_{K,\tau_0}\right| \ge \epsilon/2 \right) \\
    &\le \text{lim}_{n\rightarrow \infty} P \left( \sum_{\tau=1}^{T} |\text{max}_{1\le\tilde{\tau}\le T}(W_{\tilde{\tau}})_\tau| \cdot |X_{kT+\tau} - (\pmb{\mu_0})_\tau| \ge \frac{n\epsilon}{2} \right) \\
    &= \text{lim}_{n\rightarrow \infty} P \left( \sum_{\tau=1}^{T} |\text{max}_{1\le\tilde{\tau}\le T}(W_{\tilde{\tau}})_\tau| \cdot |X_{\tau} - (\pmb{\mu_0})_\tau| \ge \frac{n\epsilon}{2} \right) = 0
\end{align*}
\end{proof}

\begin{lemma}[Maximum likelihood with relation to the complex hypothesis of uniform-degradation]
\label{lemma:unif_optimality}
Under the setup of Theorem \ref{theorem:unif_optimality2}, denote $s_0 \coloneqq W\cdot\pmb{\mu} - \epsilon_0 [\pmb{1}^\top \Sigma^{-1} \pmb{1}]$.
$\lambda_{LR}(H_0,H_A^{unif}(\epsilon_0) | \{X_t\}_{t=1}^{n})$ is minimized by $\epsilon = \epsilon_0$ if $s_{unif} \ge s_0$, and by $\epsilon = \frac{W \cdot\pmb{\mu} - s_{unif}}{\pmb{1}^\top \Sigma^{-1} \pmb{1}}$ if $s_{unif} \le s_0$.
\end{lemma}

\begin{proof}
Applying Proposition~\ref{prop:lr_general} to Definition~\ref{def:H1_general} yields
\begin{align}
\label{eq:lr_unif1}
\begin{split}
    \lambda_{LR}&(H_0,H_A^{unif}(\epsilon_0) | \{X_t\}_{t=1}^{n}) = \\
    &\text{min}_{\epsilon \ge \epsilon_0} 2\epsilon [ W\tilde{X} ] + \epsilon^2 [ \pmb{1}^\top \Sigma^{-1} \pmb{1} ] = \text{min}_{\epsilon \ge \epsilon_0} P(\epsilon)
\end{split}
\end{align}
where $P(\epsilon)$ is a parabola with respect to $\epsilon$, with leading coefficient $\pmb{1}^\top \Sigma^{-1}\pmb{1} > 0$ (since the full-rank covariance matrix $\Sigma$ is necessarily positive definite) and minimum $\epsilon_{min} = -\frac{2W\tilde{X}}{2[\pmb{1}^\top \Sigma^{-1} \pmb{1}]} = \frac{W \pmb{\mu} - s_{unif}}{\pmb{1}^\top \Sigma^{-1} \pmb{1}}$ (remember that $\tilde{X}=X-\pmb{\mu}$).
If $s_{unif} \le s_0$ then $\epsilon_{min} \ge \epsilon_0$ and $\text{min}_{\epsilon \ge \epsilon_0} P(\epsilon)$ is minimized by $\epsilon = \epsilon_{min}$.
If $s_{unif} \ge s_0$ then $\epsilon_{min} \le \epsilon_0$ (i.e., $\epsilon_0$ is to the right of the minimum of the parabola), hence $\forall \epsilon > \epsilon_0: P(\epsilon) > P(\epsilon_0)$, and $\text{min}_{\epsilon \ge \epsilon_0} P(\epsilon)$ is minimized by $\epsilon = \epsilon_0$.
\end{proof}

{\bf Proof of Theorem~\ref{theorem:unif_optimality2} (also compactly formulated in Theorem~\ref{theorem:unif_optimality})}:
Given $\alpha \in (0,1)$, consider a $\tilde{\kappa}$-threshold-test (Definition~\ref{def:threshold_test}) with relation to the log-likelihood $\lambda_{LR}(H_0,H_A^{unif}(\epsilon_0) | \{X_t\}_{t=1}^{n})$ (and with edge-case rejection-probability $\rho\in [0,1]$), such that the significance level of the test is $1-\alpha$.
Since $\lambda_{LR} = 2\text{ln}(LR)$ is monotonously increasing with relation to the likelihood-ratio, then according to Neyman-Pearson lemma~\citep{NeymanPearson} this test has the greatest power among all tests with significance $\tilde{\alpha} \le \alpha$.
We will show that this test is equivalent to a threshold-test on the uniform-degradation weighted-mean.

According to Lemma~\ref{lemma:unif_optimality}, we have
\begin{align}
\label{eq:lr_unif2}
\begin{split}
    \lambda_{LR}&(H_0,H_A^{unif}(\epsilon_0) | \{X_t\}_{t=1}^{n}) \\
    &= \text{min}_{\epsilon \ge \epsilon_0} 2\epsilon [ W\tilde{X} ] + \epsilon^2 [ \pmb{1}^\top \Sigma^{-1} \pmb{1} ] \\
    &= \begin{cases}
        2\epsilon_0 [ W\tilde{X} ] + \epsilon_0^2 [ \pmb{1}^\top \Sigma^{-1} \pmb{1} ] & \text{if } s_{unif} \ge s_0 \\ 
        2\frac{W\cdot\pmb{\mu} - s_{unif}}{\pmb{1}^\top \Sigma^{-1} \pmb{1}} [W\tilde{X}] + [\frac{W\cdot\pmb{\mu} - s_{unif}}{\pmb{1}^\top \Sigma^{-1} \pmb{1}}]^2 [ \pmb{1}^\top \Sigma^{-1} \pmb{1} ] & \text{if } s_{unif} \le s_0
    \end{cases} \\
    &= \begin{cases}
        2\epsilon_0 s_{unif} - 2\epsilon_0 W\pmb {\mu} + \epsilon_0^2 [ \pmb{1}^\top \Sigma^{-1} \pmb{1} ] & \text{if } s_{unif} \ge s_0 \\
        -\frac{(s_{unif} - W\pmb{\mu})^2}{\pmb{1}^\top \Sigma^{-1} \pmb{1}} & \text{if } s_{unif} \le s_0
    \end{cases} \\
\end{split}
\end{align}
Clearly $\lambda_{LR}$ is strictly increasing with $s_{unif}$ in $(-\infty, s_0]$.
Note that in the case $s_{unif} \ge s_0$, $\lambda_{LR}$ is the parabola $P(s_{unif}) = -(s_{unif}-W\pmb{\mu})^2$ (up to a positive multiplicative factor), whose maximum is $s_{max}=W\pmb{\mu}$. Since in this case $s_{unif} \ge s_0 = W\pmb{\mu} - \epsilon_0 [\pmb{1}^\top \Sigma^{-1} \pmb{1}] \le W\pmb{\mu} = s_{max}$, then $s_{unif}$ is to the left of the maximum of the parabola, and hence $\lambda_{LR}$ is strictly increasing with $s_{unif}$ in $[s_0, \infty)$.

Since $\lambda_{LR}$ is strictly increasing with $s_{unif}$ in both $(-\infty, s_0]$ and $[s_0, \infty)$, then it is strictly monotonously increasing with $s_{unif}$ in $\mathbb{R}$.
Hence there exists $\kappa \in \mathbb{R}$ such that $\lambda_{LR} < \tilde{\kappa} \Leftrightarrow s_{unif} < \kappa$, and the tests are equivalent.
$\square$

\begin{lemma}[Properties of statistics under uniform-degradation]
\label{lemma:unif_properties}
Let $X$ be a $T$-long episodic signal of length $n=KT$ for some $K\in \mathbb{N}$ (i.e., integer number of episodes), with parameters $\pmb{\mu_0}-\epsilon\cdot\pmb{1} \in \mathbb{R}^T, \Sigma_0 \in \mathbb{R}^{T\times T}$.
Denote $s_{simp}=\sum_{t=1}^nX_t$ as in Eq.~\eqref{eq:normalized_statistics} and $s_{unif}=WX$ as in Definition~\ref{def:unif_weighted_mean}.
Then we have:
\begin{align*}
    E[s_{simp}] &= K\pmb{1}^\top \pmb{\mu_0} - KT\epsilon \\
    E[s_{unif}] &= KW_0\pmb{\mu_0} - \epsilon KW_0\pmb{1} \\
    \text{Var}(s_{simp}) &= K\pmb{1}^\top\Sigma_0\pmb{1} \\
    \text{Var}(s_{unif}) &= K\pmb{1}^\top \Sigma_0^{-1} \pmb{1} \\
\end{align*}
\end{lemma}

\begin{proof}
The first 3 identities are straight-forward:
\begin{align*}
    E[s_{simp}] &= \sum_{k=0}^{K-1}\sum_{\tau=1}^T (\pmb{\mu_0})_\tau - \epsilon = K\pmb{1}^\top \pmb{\mu_0} - KT\epsilon \\
    E[s_{unif}] &= \sum_{k=0}^{K-1} W_0 \cdot (\pmb{\mu_0} - \epsilon\pmb{1}) = KW_0\pmb{\mu_0} - \epsilon KW_0\pmb{1} \\
    \text{Var}(s_{simp}) &= \sum_{i,j=1}^n \text{Cov}(X_i,X_j) = K\sum_{i,j=1}^T \text{Cov}(X_i,X_j) = K\pmb{1}^\top\Sigma_0\pmb{1}
\end{align*}

For the last identity denote $Y \coloneqq \Sigma_0^{-1} X \in \mathbb{R}^n$ (i.e., $Y_i = \sum_{m=1}^T (\Sigma_0^{-1})_{im} X_m$), such that $s_{unif} = \sum_i Y_i$.
\begin{align*}
    \text{Var}(s_{unif}) &= \sum_{i,j=1}^n \text{Cov}(Y_i,Y_j) \\
    &= \sum_{i,j=1}^n \text{Cov}(\sum_{m=1}^n \Sigma_{im}^{-1} X_m, \sum_{l=1}^n \Sigma_{jl}^{-1} X_l) \\
    &= \sum_{i,j,m,l=1}^n \Sigma_{im}^{-1} \Sigma_{jl}^{-1} \text{Cov}(X_m,X_l) \\
    &= K \sum_{i,j,m,l=1}^T (\Sigma_0)_{im}^{-1} (\Sigma_0)_{jl}^{-1} \text{Cov}(X_m,X_l) \\
    &= K \sum_{i,j,m,l=1}^T (\Sigma_0)_{jl}^{-1} (\Sigma_0)_{im}^{-1} (\Sigma_0)_{ml} \\
    &= K \sum_{i,j,l=1}^T (\Sigma_0)_{jl}^{-1} \left( (\Sigma_0)_{i\cdot}^{-1} \cdot (\Sigma_0)_{\cdot l} \right) \\
    &= K \sum_{i,j,l=1}^T (\Sigma_0)_{jl}^{-1} \delta_{il}
    = K\sum_{i,j=1}^T (\Sigma_0)_{ji}^{-1}
    = K\pmb{1}^\top \Sigma_0^{-1} \pmb{1}
\end{align*}
\end{proof}

{\bf Proof of Proposition~\ref{prop:unif_test_consistency}}:
Both $\sqrt{K}\tilde{s}_{simp}^K$ and $\sqrt{K}\tilde{s}_{unif}^K$ defined in Eq.~\eqref{eq:normalized_statistics} are under $H_0$ the sums of $K$ i.i.d variables with mean 0 and variance 1.
Thus, according to the Central Limit Theorem~\citep{CLT1,CLT2}, both converge-in-distribution to the standard normal distribution under $H_0$:
\begin{align*}
\begin{split}
    &\tilde{s}_{simp}^K \xrightarrow[K\rightarrow \infty]{\enskip D \enskip} N(0,1) \\
    &\tilde{s}_{unif}^K \xrightarrow[K\rightarrow \infty]{\enskip D \enskip} N(0,1) \\
\end{split}
\end{align*}
Hence, from the definition of convergence in distribution, we have
\begin{align*}
    \text{lim}&_{K\rightarrow \infty} P\left( \tilde{s}_{simp}^K \le q_\alpha^0 \big| H_0 \right) =
    \text{lim}_{K\rightarrow \infty} F_{\tilde{s}_{simp}^K | H_0}\left( q_\alpha^0 \right) =
    \Phi\left( q_\alpha^0 \right) = \alpha \\
    \text{lim}&_{K\rightarrow \infty} P\left( \tilde{s}_{unif}^K \le q_\alpha^0 \big| H_0 \right) =
    \text{lim}_{K\rightarrow \infty} F_{\tilde{s}_{unif}^K | H_0}\left( q_\alpha^0 \right) =
    \Phi\left( q_\alpha^0 \right) = \alpha \\
\end{align*}
where $F_{s | H}$ is the Cumulative Distribution Function of the random variable $s$ under the hypothesis $H$, and $\Phi$ is of the standard normal distribution.

Note that $\tilde{s}_{simp}^K, \tilde{s}_{unif}^K$ can be computed from $s_{simp},s_{unif}$ by substituting Lemma~\ref{lemma:unif_properties} (with $\epsilon=0$, corresponding to $H_0$) in Eq.~\eqref{eq:normalized_statistics}:
\begin{align}
\label{eq:normalized_statistics_simplified}
\begin{split}
    &\tilde{s}_{simp}^K = \frac{s_{simp} - K\pmb{1}^\top \pmb{\mu_0}}{\sqrt{K\pmb{1}^\top\Sigma_0\pmb{1}}} \\
    &\tilde{s}_{unif}^K = \frac{s_{unif} - KW_0\pmb{\mu_0}}{\sqrt{K\pmb{1}^\top \Sigma_0^{-1} \pmb{1}}} \\
\end{split}
\end{align}
and by substituting Lemma~\ref{lemma:unif_properties} with $\epsilon>0$ in Eq.~\eqref{eq:normalized_statistics_simplified}, we also have the properties of $\tilde{s}_{simp}^K, \tilde{s}_{unif}^K$ under $H_A^{\epsilon}$:
\begin{align*}
\begin{split}
    &E\left[ \tilde{s}_{simp}^K | H_A^{\epsilon} \right] = -\frac{KT\epsilon}{\sqrt{K\pmb{1}^\top\Sigma_0\pmb{1}}} = -\frac{\sqrt{K}T\epsilon}{\sqrt{\pmb{1}^\top\Sigma_0\pmb{1}}} \\
    &E\left[ \tilde{s}_{unif}^K | H_A^{\epsilon} \right] = -\frac{KW_0\pmb{1}\epsilon}{\sqrt{K\pmb{1}^\top \Sigma_0^{-1} \pmb{1}}} = -\frac{\sqrt{K}W_0\pmb{1}\epsilon}{\sqrt{\pmb{1}^\top \Sigma_0^{-1} \pmb{1}}} \\
    &\text{Var}(\tilde{s}_{simp}^K | H_A^{\epsilon}) = \text{Var}(\tilde{s}_{unif}^K | H_A^{\epsilon}) = 1
\end{split}
\end{align*}
Accordingly, using the Central Limit Theorem again, we have under $H_A^\epsilon$:
\begin{align*}
\begin{split}
    &\tilde{s}_{simp}^K + \frac{\sqrt{K}T\epsilon}{\sqrt{\pmb{1}^\top\Sigma_0\pmb{1}}} \xrightarrow[K\rightarrow \infty]{\enskip D \enskip} N(0,1) \\
    &\tilde{s}_{unif}^K + \frac{\sqrt{K}W_0\pmb{1}\epsilon}{\sqrt{\pmb{1}^\top \Sigma_0^{-1} \pmb{1}}} \xrightarrow[K\rightarrow \infty]{\enskip D \enskip} N(0,1) \\
\end{split}
\end{align*}
and by the definition of convergence in distribution, we receive
\begin{align}
\label{eq:unif_asymptotic_rejection}
\begin{split}
    \text{lim}&_{K\rightarrow \infty} P\left( \tilde{s}_{simp}^K \le q_\alpha^0 \big| H_A^\epsilon \right) =
    \text{lim}_{K\rightarrow \infty} F_{\tilde{s}_{simp}^K | H_A^\epsilon}\left( q_\alpha^0 \right) = \\
    &\text{lim}_{K\rightarrow \infty} F_{\tilde{s}_{simp}^K + \frac{\sqrt{K}T\epsilon}{\sqrt{\pmb{1}^\top\Sigma_0\pmb{1}}} | H_0}\left( q_\alpha^0 \right) =
    \text{lim}_{K\rightarrow \infty} \Phi\left( q_\alpha^0 + \frac{\sqrt{K}T\epsilon}{\sqrt{\pmb{1}^\top\Sigma_0\pmb{1}}} \right) = 1 \\
    \text{lim}&_{K\rightarrow \infty} P\left( \tilde{s}_{unif}^K \le q_\alpha^0 \big| H_A^\epsilon \right) =
    \text{lim}_{K\rightarrow \infty} F_{\tilde{s}_{unif}^K | H_A^\epsilon}\left( q_\alpha^0 \right) = \\
    &\text{lim}_{K\rightarrow \infty} F_{\tilde{s}_{unif}^K + \frac{\sqrt{K}W_0\pmb{1}\epsilon}{\sqrt{\pmb{1}^\top \Sigma_0^{-1} \pmb{1}}} | H_0}\left( q_\alpha^0 \right) =
    \text{lim}_{K\rightarrow \infty} \Phi\left( q_\alpha^0 + \frac{\sqrt{K}W_0\pmb{1}\epsilon}{\sqrt{\pmb{1}^\top \Sigma_0^{-1} \pmb{1}}} \right) = 1 \\
\end{split}
\end{align}
$\square$


{\bf Proof of Theorem \ref{theorem:unif_power2} (also compactly formulated in Theorem~\ref{theorem:unif_power})}:
Following identical reasoning to the proof of Proposition~\ref{prop:unif_test_consistency} with $\epsilon$ replaced by $\epsilon/\sqrt{K}$, and recalling that $W_0=\pmb{1}^\top \Sigma_0^{-1}$ (Definition~\ref{def:unif_weighted_mean}), we receive the analog of Eq.~\eqref{eq:unif_asymptotic_rejection}:
\begin{align}
\label{eq:unif_asymptotic_power}
\begin{split}
    \text{lim}&_{K\rightarrow \infty} P\left( \tilde{s}_{simp}^K \le q_\alpha^0 \big| H_A^{\epsilon,K} \right) =
    \Phi\left( q_\alpha^0 + \frac{T\epsilon}{\sqrt{\pmb{1}^\top\Sigma_0\pmb{1}}} \right) \\
    \text{lim}&_{K\rightarrow \infty} P\left( \tilde{s}_{unif}^K \le q_\alpha^0 \big| H_A^{\epsilon,K} \right) =
    \Phi\left( q_\alpha^0 + \epsilon \sqrt{\pmb{1}^\top \Sigma_0^{-1} \pmb{1}} \right) \\
\end{split}
\end{align}

To complete the proof, since $\Phi$ is monotonously increasing, we only have to show that $\frac{T}{\sqrt{\pmb{1}^\top \Sigma_0 \pmb{1}}} \le \sqrt{\pmb{1}^\top \Sigma_0^{-1} \pmb{1}}$, or equivalently $\frac{T}{\pmb{1}^\top \Sigma_0^{-1} \pmb{1}} \le \frac{\pmb{1}^\top \Sigma_0 \pmb{1}}{T}$, which can be seen as a matrix-form generalization for the harmonic-algebraic means inequality.

Since the invertible covariance matrix $\Sigma_0$ is necessarily symmetric and positive definite, it has a symmetric positive definite square-root $R^2=\Sigma_0$.
Since $\pmb{1}^\top \Sigma_0 \pmb{1} = \pmb{1}^\top R^\top R \pmb{1} = \lVert R\pmb{1} \rVert^2$ and $\pmb{1}^\top \Sigma_0^{-1} \pmb{1} = \lVert R^{-1}\pmb{1} \rVert^2$, we indeed have by Cauchy-Schwarz inequality
\begin{align}
\label{eq:matrix_means_inequality}
    (\pmb{1}^\top \Sigma_0^{-1} \pmb{1}) (\pmb{1}^\top \Sigma_0 \pmb{1}) =
    \lVert R^{-1}\pmb{1} \rVert^2 \cdot \lVert R\pmb{1} \rVert^2 \ge ((\pmb{1}^\top R^{-1}) (R \pmb{1}))^2 =
    (\pmb{1}^\top\pmb{1})^2 =
    T^2
\end{align}
$\square$

{\bf Proof of Proposition~\ref{prop:unif_power_gain}}
Since $\Sigma_0$ is symmetric it is orthogonally diagonalizable, i.e., $\Sigma_0 = U^\top AU$ where $A$ is diagonal and $U$ is orthogonal.
Since $\Sigma_0$ is positive-definite (note that Definition~\ref{def:H0} assumes full-rank covariance matrix), its eigenvalues are positive, i.e., $\forall 1\le i \le T: \lambda_i=A_{ii}>0$.
We also have $\pmb{1}^\top \Sigma_0 \pmb{1} = \pmb{1}^\top U^\top AU \pmb{1} = u^\top A u$ (where $u=U\pmb{1}$), and $\pmb{1}^\top \Sigma_0^{-1} \pmb{1} = \pmb{1}^\top U^\top A^{-1}U \pmb{1} = u^\top A^{-1} u$.

From this we receive
\begin{align*}
    G^2 =& \frac{(\pmb{1}^\top \Sigma_0^{-1} \pmb{1})(\pmb{1}^\top \Sigma_0 \pmb{1})}{T^2} =
    \frac{(u^\top A^{-1}u)(u^\top Au)}{T^2} \\=&
    \frac{1}{T^2} (\sum_{i=1}^T u_i^2\lambda_i) (\sum_{j=1}^T u_j^2/\lambda_j) =
    \frac{1}{T^2} \sum_{i,j=1}^T (u_iu_j)^2\frac{\lambda_i}{\lambda_j} \\=&
    \frac{1}{2T^2} \sum_{i,j=1}^T (u_iu_j)^2\left(\frac{\lambda_i}{\lambda_j} + \frac{\lambda_j}{\lambda_i}\right) \\=&
    \frac{1}{2T^2} \sum_{i,j=1}^T (u_iu_j)^2\frac{\lambda_i^2+\lambda_j^2}{\lambda_i\lambda_j} \\=&
    \frac{1}{2T^2} \sum_{i,j=1}^T (u_iu_j)^2\frac{(\lambda_i-\lambda_j)^2 + 2\lambda_i\lambda_j}{\lambda_i\lambda_j} \\=&
    \frac{1}{2T^2} \left[ 2\sum_{i,j=1}^T (u_iu_j)^2 + \sum_{i,j=1}^T (u_iu_j)^2\frac{(\lambda_i-\lambda_j)^2}{\lambda_i\lambda_j} \right] \\=&
    \frac{1}{2T^2} \left[ 2u^\top u + \sum_{i,j=1}^T (u_iu_j)^2\frac{(\lambda_i-\lambda_j)^2}{\lambda_i\lambda_j} \right] \\=&
    \frac{1}{2T^2} \left[ 2(\pmb{1}^\top I\pmb{1}) + \sum_{i,j=1}^T (u_iu_j)^2\frac{(\lambda_i-\lambda_j)^2}{\lambda_i\lambda_j} \right] \\=&
    1 + \frac{1}{2T^2} \sum_{i,j=1}^T (u_iu_j)^2\frac{(\lambda_i-\lambda_j)^2}{\lambda_i\lambda_j}
\end{align*}
and since $\forall i: u_i = U_{i\cdot}\pmb{1}\ne0$ as the sum of a row of an orthogonal matrix, we only need to denote $w_{ij} \coloneqq \frac{(u_iu_j)^2}{2T^2\lambda_i\lambda_j} > 0$.
$\square$

\begin{lemma}[Sensitivity of the minimum to deviations in the elements]
\label{lemma:min_delta_bound}
Let a finite set $\mathbb{A}$, functions $f,g: \mathbb{A}\rightarrow\mathbb{R}$, and $\epsilon > 0$.
Note that since $\mathbb{A}$ is finite, both $f,g$ are bounded and denote $|g|\le G$ an upper bound.
Denote $y\coloneqq \text{min}_{a\in\mathbb{A}} f(a)+\epsilon g(a)$ and $\tilde{y}\coloneqq \text{min}_{a\in\mathbb{A}} f(a)$.
Then $|y-\tilde{y}| \le 3\epsilon G$.
\end{lemma}

\begin{proof}
Denote by $a_0, \tilde{a}_0$ the minimizers of $y,\tilde{y}$ respectively, i.e., $y=f(a_0)+\epsilon g(a_0) \le f(\tilde{a}_0)+\epsilon g(\tilde{a}_0)$ and $\tilde{y} = f(\tilde{a}_0) \le f(a_0)$.
From the last two inequalities we get $0 \le f(a_0) - f(\tilde(a)_0) \le \epsilon (g(\tilde{a}_0) - g(a_0))$.
Finally from the triangle inequality,
\begin{align*}
|y-\tilde{y}| &= |f(a_0)-f(\tilde{a}_0)+\epsilon g(a_0)| \le |\epsilon (g(\tilde{a}_0) - g(a_0))| + |\epsilon g(a_0)| \\
&\le \epsilon \left[ |g(\tilde{a}_0)| + |g(a_0))| + |g(a_0)| \right] \le 3\epsilon G
\end{align*}
\end{proof}


{\bf Proof of Theorem \ref{theorem:part_optimality}}:
Similarly to the proof of Theorem \ref{theorem:unif_optimality2}, we wish to show that the log-likelihood-ratio $\lambda_{LR}$ from Lemma~\ref{lemma:unif_optimality} -- after substituting Definition~\ref{def:H1_part} -- is strictly monotonously increasing with $s_{part}$.

Given $a_0 \in \{0,1\}^T$, let $a(a_0) \in \{0,1\}^n$ be its $T$-periodic completion to $n$ dimensions, and recall the notations $m(p) = \lceil pT \rceil, n=KT+\tau_0$.
Then we have 
\begin{align*}
\begin{split}
    \lambda_{LR}&(H_0,H_A^{part} | \{X_t\}_{t=1}^{n}) = \\
    &= \text{min}_{a_0 \in A_m^T} 2\epsilon a^\top\Sigma^{-1}\tilde{X} + \epsilon^2 a^\top\Sigma^{-1}a \\
    &= 2\epsilon \cdot \text{min}_{a_0 \in A_m^T} \left( a^\top\Sigma^{-1}\tilde{X} + 0.5\epsilon (a^\top\Sigma^{-1}a) \right) \\
\end{split}
\end{align*}

Denote $f(a_0)=a(a_0)^\top\Sigma^{-1}\tilde{X}$, $g(a_0)=0.5a(a_0)^\top\Sigma^{-1}a(a_0)$ and $y=\text{min}_{a_0\in A_m^T} f(a_0)+\epsilon g(a_0)$, such that $\lambda_{LR} = 2\epsilon y$ is monotonously increasing with $y$.
Note that
\begin{align*}
\begin{split}
    \text{min}_{a_0\in A_m^T} f(a_0)
    =& \text{min}_{a_0 \in A_m^T} a^\top \Sigma^{-1} \tilde{X} \\
    =& \text{min}_{a_0 \in A_m^T} \sum_{k=0}^{K-1} \sum_{\substack{\tau=1\\(a_0)_\tau=1}}^T \tilde{s}_{kT+\tau} + \sum_{\substack{\tau=1\\(a_0)_\tau=1}}^{\tau_0} \tilde{s}_{kT+\tau} \\
    =& \text{min}_{a_0 \in A_m^T} \sum_{\tau\in o(a_0)} \tilde{S}_\tau \\
    =& s_{part}(X)
\end{split}
\end{align*}
Also note that the term $g(a_0)$ is bounded:
$$\forall a_0 \in A_m^T: |g(a_0)| \le \frac{1}{2}\sum_{i,j=1}^T |(\Sigma^{-1})_{ij}| \le \frac{K+1}{2}\sum_{i,j=1}^T |(\Sigma_0^{-1})_{ij}|$$
Hence, by Lemma~\ref{lemma:min_delta_bound}, we have
\begin{equation}
\label{eq:part_mean_approx_error}
    |y-s_{part}| \le \epsilon\frac{3(K+1)}{2}\sum_{i,j=1}^T |(\Sigma_0^{-1})_{ij}| = \mathcal{O}(\epsilon)
\end{equation}
(where $\mathcal{O}(\epsilon)$ is defined with relation to $\epsilon\rightarrow0$).

Now consider the $\alpha$-quantile of $y$ under $H_0$, denoted $\tilde{\kappa}=q_\alpha(y|H_0)$.
By construction $P(y\le \tilde{\kappa} | H_0) = \alpha$ (up to non-continuous probability density in the edge case $y=\tilde{\kappa}$).
According to Neyman-Pearson lemma, a threshold-test on $y$ has the greatest power $P_\alpha$ among all statistical tests with significance level $\le \alpha$ (see the proof of Theorem~\ref{theorem:unif_optimality2} for more details), i.e., $P_\alpha = P(y\le \tilde{\kappa} | H_A^{part}) = F_{y|H_A^{part}}(\tilde{\kappa})$ (where $F_{s|H}$ is the Cumulative Distribution Function of the variable $s$ given the hypothesis $H$).

Denote the $\alpha$-quantile of the actual test-statistic $s_{part}$ by $\kappa=q_\alpha(s_{part}|H_0)$.
Since $\forall X \in \mathbb{R}^n: |y-s_{part}|=\mathcal{O}(\epsilon)$ (Eq.~\eqref{eq:part_mean_approx_error}), we also have $\big|\tilde{\kappa}-\kappa\big| = \big|q_\alpha(y|H_0)-q_\alpha(s_{part}|H_0)\big| = \mathcal{O}(\epsilon)$.
Hence, along with Eq.~\eqref{eq:part_mean_approx_error}, we have
\begin{align*}
    P\left(s_{part}\le \kappa | H_A^{part}\right) &\ge P\left(y+\mathcal{O}(\epsilon) \le \tilde{\kappa}-\mathcal{O}(\epsilon) \big| H_A^{part}\right) \\
    &= P\left(y \le \tilde{\kappa}-\mathcal{O}(\epsilon) \big| H_A^{part}\right) \\
    &= F_{y|H_A^{part}}\left(\tilde{\kappa}-\mathcal{O}(\epsilon)\right) \\
    &= F_{y|H_A^{part}}\left(\tilde{\kappa}\right) - \mathcal{O}(\epsilon) \\
    &= P_\alpha - \mathcal{O}(\epsilon)
\end{align*}
where the second-to-last equality is true since $F_{y|H_A^{part}}(x)$ has a finite derivative at $x=\tilde{\kappa}$, as the CDF of the minimum of the normal variables $\{ a^\top\Sigma^{-1}\tilde{X} + 0.5\epsilon (a^\top\Sigma^{-1}a) \}_{a\in A_m^T}$.
$\square$


\section{Bootstrap for Sequential Tests: Extended Discussion}
\label{sec:detailed_sequential_test}

Section~\ref{sec:sequential_test} describes a mechanism for sequential hypothesis testing with relation to episodic signals.
The mechanism simply runs individual hypothesis tests repeatedly with a constant significance level $\alpha$, similarly to the concept of $\alpha$-spending functions~\citep{alpha_spending,alpha_spending_notes}, and in particular to Pocock approach~\citep{Pocock}.

Note that Pocock's constant $\alpha$-spending function is often avoided, as it is claimed to spend "too much" $\alpha$-budget in the beginning of the sequential test on account of its end.
In our online setup this time-homogeneous approach is welcome, as we do not to rely on well-defined beginning and end.
However, in contrast to Pocock, we cannot assume independence between nor normality of the aggregative parts of the data.

The sequential test (described in Algorithm~\ref{algo:sequential_test}) uses a constant manually-determined lookback-horizon $h$.
Any individual test at time $t = kT + \tau$ runs the simple threshold-test of Algorithm~\ref{algo:individual_test} on the signal $X_{(k-h)T},...,X_{kT+\tau}$, i.e., it looks exactly $h+\tau/T$ episodes back.
In practice, $n_h$ different lookback-horizons $h_1,...,h_{n_h}$ can be used simultaneously, such that at any point of time, we reject $H_0$ if any of the lookback tests rejects it.
This allows us to detect slight changes which are only detectable over large amounts of data (large $h$); while still allowing quick detection of larger abrupt changes, without mixing them with older irrelevant data (small $h$).

In order to determine the significance level $\alpha$ for the individual tests within the sequential test, we use the bootstrap mechanism described in Algorithm~\ref{algo:sequential_bootstrap} (also see extended pseudo-code in Algorithm~\ref{algo:sequential_bootstrap2} in Appendix~\ref{sec:algorithms}).
The mechanism simulates sequential tests using bootstrap-sampling of $\text{max}(h_1,...,h_{n_h}) + \tilde{h}$ episodes (length of simulation + maximum lookback horizon) from a reference dataset of $N$ episodes assumed to be i.i.d.
Once the episodes are sampled, the simulation runs $\tilde{h}$ episodes without stopping condition, keeps track of the resulted $p$-values, and eventually returns the minimal $p$-value among all the individual tests.
This simulative process is repeated $\tilde{B}$ times with different bootstrap-samples, and the $\alpha_0$-quantile among all the minimal-$p$-values is chosen as the individual-test significance level $\alpha$. Indeed, $\alpha_0$ is the relative part of bootstrap-samples in which at least one individual test returned $p$-value smaller than $\alpha$.

The sequential bootstrap mechanism of Algorithm~\ref{algo:sequential_bootstrap} may look computationally overwhelming since it runs a bootstrap that calls another bootstrap (Algorithm~\ref{algo:individual_bootstrap}, called through Algorithm~\ref{algo:individual_test}).
However, if the sequential test runs individual tests in $n_h$ different lookback-horizons (where typically $n_h \le 3$) in frequency of $F$ test-points per episode, then the inner bootstrap of Algorithm~\ref{algo:individual_bootstrap} will only be called $n_h\cdot F$ times in total (thanks to the bootstrap-storage mechanism described in Algorithm~\ref{algo:individual_test}).
Also note that in spite of its name, the whole sequential bootstrap algorithm is intended to run only once (and not sequentially) -- after the reference dataset becomes available.

As a practical remark for implementation, note that Algorithm~\ref{algo:individual_test} necessarily returns $p$-value$\ge \frac{1}{B+1}$, which is the resolution of the inner bootstrap. Now consider the case where in Algorithm~\ref{algo:sequential_bootstrap}, in more than $\alpha_0$ of the simulated sequential tests, there is certain individual test whose return value is $\frac{1}{B+1}$.
In other words, the bootstrap found that under $H_0$, with probability higher than $\alpha_0$, a sequential test of length $\tilde{h}$ will encounter the most extreme possible result of Algorithm~\ref{algo:individual_test} at least once.
In that case Algorithm~\ref{algo:sequential_test} would not be able to distinguish any degradation from $H_0$: we would have the individual test threshold set to $\alpha=\text{quantile}_{\alpha_0}(P)=\frac{1}{B+1}$, which can never be overcome.
For this reason, Algorithm~\ref{algo:sequential_test} makes sure to check whether $\alpha=\frac{1}{B+1}$.
Possible solutions in this situation are increase of $B$ for better resolution, or reduction of the required significance level through either $\tilde{h}$ or $\alpha_0$.

\paragraph{Multiple test-statistics:} Every iteration, Algorithm~\ref{algo:sequential_test} calculates the test-statistic for multiple lookback horizons, where Algorithm~\ref{algo:sequential_bootstrap} is responsible of controlling the family-wise type-I error rate through the test-thresholds.
In a similar manner, Algorithm~\ref{algo:sequential_test} can be generalized to run multiple test-statistics in parallel: simply iterate over the statistics the same as iterating over the lookback-horizons.

Heterogeneous test-statistics should provide more robustness to the alternative hypothesis, since every statistic is often affected differently by every alternative hypothesis.
This comes at the cost of reduced sensitivity of each statistic, expressed through decrease of the test-thresholds by Algorithm~\ref{algo:sequential_bootstrap}.


\section{Algorithms (Pseudocode)}
\label{sec:algorithms}

This appendix concentrates the pseudo-code of the algorithms for hypothesis testing and for bootstrap-based $\alpha$ tuning, in the contexts of both individual and sequential tests.
Algorithm~\ref{algo:sequential_bootstrap2} is a more detailed version of the pseudo-code of Algorithm~\ref{algo:sequential_bootstrap}.

\begin{algorithm}[!ht]
\SetAlgoLined
 {\bf Input}: $x\in \mathbb{R}^{N\times T}$ assumed to be drawn from a $T$-long episodic signal; sample size $n=KT+\tau_0\in\mathbb{N}$; a test-statistic function $s: \mathbb{R}^n \rightarrow \mathbb{R}$; number of repetitions $B\in\mathbb{N}$\;
 {\bf Output}: test-statistic bootstrap distribution $S\in \mathbb{R}^B$\;
 
 {\bf Algorithm}:\\
 Initialize $S \in \mathbb{R}^B$\;
 \For{b in 1:B}{
  // sample \\
  Initialize $y \in \mathbb{R}^n$\;
  \For{k in 0:K-1}{
  Sample $j$ uniformly from $(1,...,N)$\;
  $y[kT+1:kT+T] \leftarrow (x_{j1},...,x_{jT})$\;
  }
  Sample $j$ uniformly from $(1,...,N)$\;
  $y[KT+1:KT+\tau_0] \leftarrow (x_{j1},...,x_{j\tau_0})$\;
  
  // calculate \\
  $S_b \leftarrow s(y)$\;
 }
 Return $S$\;
 \caption{Individual\_test\_bootstrap}
 \label{algo:individual_bootstrap}
\end{algorithm}

\begin{algorithm}[!ht]
\SetAlgoLined
 {\bf Input}: reference episodic signal $x_0\in \mathbb{R}^{N\times T}$; test data $x\in \mathbb{R}^n$; a test-statistic function $s: \mathbb{R}^n \rightarrow \mathbb{R}$; bootstrap repetitions $B\in\mathbb{N}$; bootstrap distributions storage $BS$; allowed type-I error rate $\alpha\in(0,1)$\;
 {\bf Output}: rejection $\in \{0,1\}$; P-value $p\in \mathbb{R}$\;
 
 {\bf Algorithm}:\\
 \If{$BS[n]$ not exists}{
  $BS[n] \leftarrow \text{Individual\_test\_bootstrap}(T,N,x_0,n,s,B)$;  \hspace{1cm}  (Algorithm~\ref{algo:individual_bootstrap})\\
 }
 $S \leftarrow BS[n]$\;
 $\kappa_\alpha \leftarrow \text{quantile}_\alpha(S)$\;
 $y \leftarrow s(x)$\;
 $count \leftarrow \left| \{ b\in \{1,...,B\} | S_b \le y \} \right|$\;
 $p \leftarrow \frac{1+\text{count}}{1+B}$\;
 $reject$ $= 1$ if $y < \kappa_\alpha$ else $0$; \hspace{1cm} (or equivalently, $1$ if $p < \alpha$ else $0$)\\
 Return $reject$, $p$\;
 \caption{Individual\_degradation\_test}
 \label{algo:individual_test}
\end{algorithm}

\begin{algorithm}[!ht]
\SetAlgoLined
 {\bf Input}: $x\in \mathbb{R}^{N\times T}$ assumed to be drawn from a $T$-long episodic signal; test-statistic function $s$; inner-bootstrap repetitions $B\in\mathbb{N}$; inner-bootstrap storage $BS$; tests frequency $d \in [1,T]$ and lookback horizons $h_1,...,h_{n_h}\in \mathbb{N}$; sequential test length $\tilde{h}\in \mathbb{N}$; outer-bootstrap repetitions $\tilde{B}\in\mathbb{N}$\;
 {\bf Output}: bootstrap-distribution $P \in [0,1]^{\tilde{B}}$ of the minimal-$p$-value in a sequential test of $\tilde{h}$ episodes under $H_0$\;
 
 {\bf Algorithm}:\\
 Initialize $P = (1,...,1)\in [0,1]^{\tilde{B}}$\;
 $h_{max} \leftarrow \text{max}(h_1,...,h_{n_h})$\;
 \For{b in 1:$\tilde{B}$}{
  // sample \\
  Initialize $Y \in \mathbb{R}^{(h_{max}+\tilde{h})T}$\;
  
  \For{k in 0:($h_{max}$+$\tilde{h}$-1)}{
   Sample $j$ uniformly from $(1,...,N)$\;
   $Y[kT+1:kT+T] \leftarrow (x_{j1},...,x_{jT})$\;
  }
  
  // calculate p-value at any time for any lookback horizon \\
  \For{$k$ in 0:($\tilde{h}$-1)}{
   \For{$\tau$ in 1:d:T}{
    \For{h in $h_1,...,h_{n_h}$}{
     $y \leftarrow Y[(h_{max}+k-h)T : (h_{max}+k)T+\tau]$\;
     $p \leftarrow \text{Individual\_test}(x_0=x,x=y,s=s,B=B,BS=BS,\alpha=1).p$; \hspace{1cm}  (Algorithm~\ref{algo:individual_test})\\
     $P[b] \leftarrow \text{min}(P[b], p)$\;
    }
   }
  }
 }
 Return $P$\;
 \caption{Sequential\_test\_bootstrap}
 \label{algo:sequential_bootstrap2}
\end{algorithm}

\begin{algorithm}[!ht]
\SetAlgoLined
 {\bf Input}: reference episodic signal $x_0\in \mathbb{R}^{N\times T}$; test data stream $x$; test-statistic function $s$; inner-bootstrap repetitions $B\in\mathbb{N}$; tests frequency $d \in [1,T]$ and lookback horizons $h_1,...,h_{n_h}\in \mathbb{N}$; family-wise significance  parameters $\alpha_0\in(0,1), \tilde{h}\in \mathbb{N}$; outer-bootstrap repetitions $\tilde{B}\in\mathbb{N}$\;
 {\bf Output}: time of $H_0$ rejection\;
 
 {\bf Algorithm}:\\
 Initialize bootstrap-storage $BS$\;
 $P \leftarrow \text{Sequential\_bootstrap}(x_0,s,B,BS,d,\{h_i\},\tilde{h},\tilde{B})$; \hspace{1cm}  (Algorithm~\ref{algo:sequential_bootstrap})\\
 $\alpha \leftarrow \text{quantile}_{\alpha_0}(P)$\;
 \If{$\alpha=\frac{1}{B+1}$}{
  // Can never reject $H_0$ \\
  Warn("Either increase $B$ or reduce significance requirements.")\;
  Return ERROR\;
 }
 
 \For{$k$ in ($h_{max}$, $h_{max}$+1, ...)}{
  \For{$\tau$ in 1:d:T}{
   \For{h in $h_1,...,h_{n_h}$}{
    $y \leftarrow x[(k-h)T : kT+\tau]$\;
    $r \leftarrow \text{Individual\_test}(x=y,s=s,BS=BS,\alpha=\alpha)$.reject; \hspace{1cm}  (Algorithm~\ref{algo:individual_test})\\
    \If{r=1}{
     // Reject $H_0$ \\
     Return $kT+\tau$\;
    }
   }
  }
 }
 \caption{Sequential\_degradation\_test}
 \label{algo:sequential_test}
\end{algorithm}


\section{Experiments Implementation Details}
\label{sec:experiments_implementation}

In the Pendulum~\citep{Pendulum} environment, where the goal is to keep a one-dimensional pendulum-pole pointing upwards, we define several alternative scenarios.
\textit{ccostx} scenario (with a parameter $x$) increases the cost of action ("control cost", which is quadratic in the activated force) to $x\%$ of its original value. Note that the control cost is typically the smaller among the components of the reward, which also include the angle of the pendulum and its speed.
\textit{noisex} scenario adds an additive random normally-distributed noise to each action, whose standard deviation is $x\%$ of the range of valid actions.
\textit{lenx} and \textit{massx} scenarios change the Pendulum length and mass respectively to $x\%$ of their original values. Note that while these scenarios are not necessarily harder to act in, the changes are still supposed to cause degradation since the agent is not re-trained for them.

In the HalfCheetah~\citep{HalfCheetah} environment, where the goal is to train a two-dimensional cheetah to run as fast as possible, we also define several alternative scenarios.
\textit{ccostx} and \textit{massx} are similar to the analog scenarios in Pendulum described above.
The "control cost" in this case is also quadratic with the activated force, and is typically smaller than the other component of the reward -- the speed of the Cheetah.
\textit{gravityx} scenario changes the gravity to $x\%$ of its original value.

Table~\ref{tab:scenarios} briefly summarizes the various scenarios, and Table~\ref{tab:envs} (in Section~\ref{sec:methodology}) summarizes the parameters of the tests setup per environment.


\begin{table}
\centering
\caption{Environments scenarios}
\label{tab:scenarios}
\begin{tabular}{|l|l|}
\hline
Environment & Scenarios \\
\hline\hline
Pendulum-v0 &
\vtop{\hbox{\strut $H_0$}\hbox{\strut ccost\textit{x}: action cost $\times$= x\%}\hbox{\strut noise\textit{x}: additive noise = x\% of max action}\hbox{\strut len\textit{x}: lenth $\times$= x\%}\hbox{\strut mass\textit{x}: mass $\times$= x\%}} \\
\hline
HalfCheetah-v3 &
\vtop{\hbox{\strut $H_0$}\hbox{\strut ccost\textit{x}: action cost $\times$= x\%}\hbox{\strut mass\textit{x}: mass $\times$= x\%}\hbox{\strut gravity\textit{x}: gravity $\times$= x\%}}\\
\hline
Humanoid &
\vtop{\hbox{\strut $H_0$}\hbox{\strut ccost\textit{x}: action cost $\times$= x\%}\hbox{\strut len\textit{x}: leg size $\times$= x\%}}\\
\hline
\end{tabular}
\end{table}

For every environment, before running the statistical tests according to Section~\ref{sec:methodology}, the recorded rewards are downsampled in time by factor $d$: every interval of samples $\{x_{t}\}_{t=d\cdot\tilde{t}+1}^{d\cdot\tilde{t}+d}$ is replaced by its mean as a single sample $\tilde{x}_{\tilde{t}} \coloneqq \frac{1}{d} \sum_{t=d\cdot\tilde{t}+1}^{d\cdot\tilde{t}+d} x_t$.
The downsampling reduces the dimension of the covariance matrix $\Sigma_0$ to $\frac{T}{d} \times \frac{T}{d}$, making it less noisy to estimate. In addition, it reduces the computational complexity of the experiments in this section.
After the downsampling, the sequential tests apply an individual test in every single time-step, i.e., the testing-frequency of the sequential tests is $F=T/d$ per episode.

The statistical tests compared in Section~\ref{sec:results} are mostly based on the threshold-tests described in Algorithm~\ref{algo:individual_test} (for individual tests) and Algorithm~\ref{algo:sequential_test} (for sequential tests), with different test-statistics:
\begin{itemize}
    \item {\bf Mean}: simple mean of the rewards.
    \item {\bf CUSUM}: the standard cumulative-sum~\citep{CUSUM,CUSUM_description} test, with every time-step normalized by its standard-deviation (estimated over all the reference episodes), and with reference value  $k=0.5$.
    As CUSUM is online by nature, it is used as is (beginning to run $h$ episodes in advance for any lookback horizon $h$) instead of as part of Algorithm~\ref{algo:sequential_test}. The family-wise significance level of CUSUM is controlled using Algorithm~\ref{algo:sequential_bootstrap}.
    Note that through the normalization mentioned above, we let CUSUM take advantage of the episodic setup and the trusted reference data.
    Other normalization methods (time-invariant normalization and no normalization) did not improve CUSUM results in the experiments of Section~\ref{sec:results}.
    \item {\bf Hotelling}: Hotelling test~\citep{Hotelling} is an optimal test for detection of mean-shift under unchanged covariance in multivariate normal variables. We generalize the implementation so that the test can be applied to non-integer number of multivariate variables (corresponding to non-integer number of episodes).
    Note that the alternative hypothesis of Hotelling test is very general, which results in reduced test power for the specific domain of interest in the current work -- degradation -- as demonstrated in Section~\ref{sec:results}.
    \item {\bf UDWM}: the uniform-degradation weighted-mean from Definition~\ref{def:unif_weighted_mean}. The test based on this statistic is also named UDT.
    \item {\bf PDM}: 0.9-partial-degradation mean from Definition~\ref{def:part_weighted_mean}, corresponding to degradation focused on 90\% of the time-steps. The test based on this statistic is also named PDT.
    While this statistic may look quite similar to UDWM, it can produce different results in environments where UDWM concentrates most of the weights in few time-steps -- if PDM "drops" these time-steps. This may make PDM more robust to degradation scenarios where the dominantly-weighted time-steps are not affected. Further discussion regarding $p$ and the relation to the CVaR statistic is provided in Appendix~\ref{sec:detailed_optimal_test}.
    \item {\bf Mixed}: The mixed statistic essentially incorporates multiple statistics together -- Mean and PDM in this case -- and testing whether any of them has "extreme" values.
    It is defined as $s=min(\text{p-value(Mean)},\text{p-value(PDM)})$, i.e., we calculate both statistics and take the more significant p-value.
    Note that $s$ itself is the statistic, and that its p-value is derived using Algorithm~\ref{algo:individual_bootstrap}'s bootstrap as in any of the other test statistics.
    We see below that the Mixed test enjoys most of the value of PDM, and still performs reasonably well wherever PDM is not robust enough (namely, when the mean reward decreases but the highly-weighted time-steps actually \textit{increase}).
    The test based on this statistic is also named MDT.
\end{itemize}


\section{Complementary Figures}
\label{sec:detailed_results}

Figures~\ref{fig:weights}-\ref{fig:hum_params} introduce additional results from the experiments described in Section~\ref{sec:experiments}.
Note that Section~\ref{sec:experiments} refers directly to some of the results presented in this section.

\begin{figure}[!h]
    \centering
    
    \begin{subfigure}{0.32\textwidth}
    \includegraphics[width=1.\linewidth]{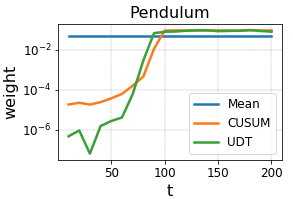}
    \end{subfigure}
    \begin{subfigure}{0.33\textwidth}
    \includegraphics[width=1.\linewidth]{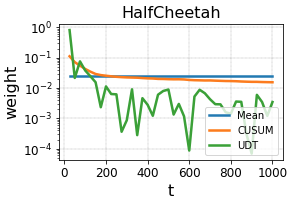}
    \end{subfigure}
    \begin{subfigure}{0.33\textwidth}
    \includegraphics[width=1.\linewidth]{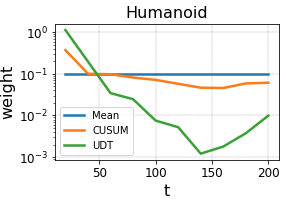}
    \end{subfigure}

    \caption{\small The weights assigned to the various time-steps by the various tests. Mind the logarithmic scale. Note that the weights of CUSUM are received from its normalization scheme, i.e., $w_t=1/\text{std}(r_t)$.}
    \label{fig:weights}
\end{figure}

\begin{figure}[!h]
    \centering
    
    \begin{subfigure}{0.32\textwidth}
    \includegraphics[width=1.\linewidth]{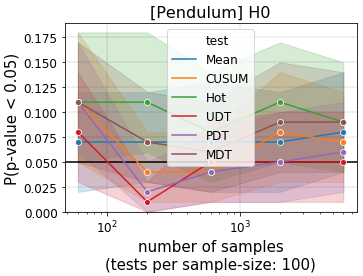}
    \end{subfigure}
    \begin{subfigure}{0.33\textwidth}
    \includegraphics[width=1.\linewidth]{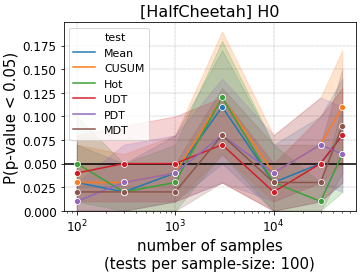}
    \end{subfigure}
    \begin{subfigure}{0.33\textwidth}
    \includegraphics[width=1.\linewidth]{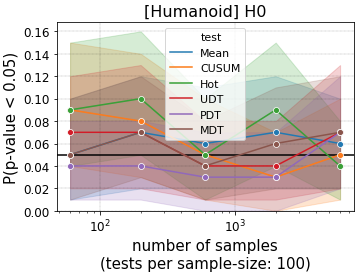}
    \end{subfigure}

    \begin{subfigure}{0.32\textwidth}
    \includegraphics[width=1.\linewidth]{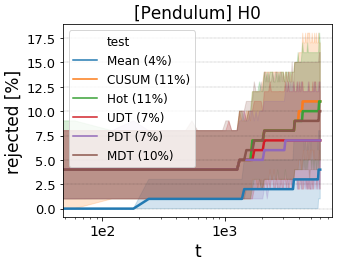}
    \end{subfigure}
    \begin{subfigure}{0.33\textwidth}
    \includegraphics[width=1.\linewidth]{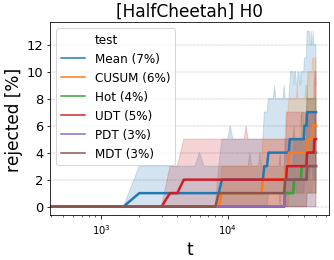}
    \end{subfigure}
    \begin{subfigure}{0.33\textwidth}
    \includegraphics[width=1.\linewidth]{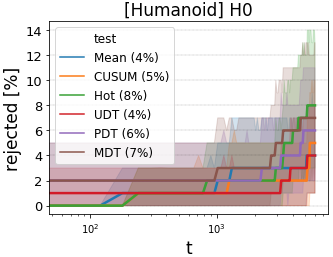}
    \end{subfigure}

    \caption{\small Percent of rejections of $H_0$ when $H_0$ is true, for various statistical tests, for both individual (top) and sequential (bottom) tests. Each point in each plot represents $M=100$ tests, and the shaded area represents 95\% confidence interval. The tests were tuned by Algorithm~\ref{algo:individual_bootstrap} (individual) and Algorithm~\ref{algo:sequential_bootstrap} (sequential), using a reference dataset, to yield rejection rate of 5\% under $H_0$.}
    \label{fig:H0}
\end{figure}

\begin{figure}[!h]
    \centering
    
    \begin{subfigure}{0.33\textwidth}
    \includegraphics[width=1.\linewidth]{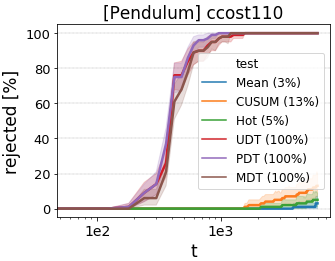}
    \end{subfigure}
    \begin{subfigure}{0.32\textwidth}
    \includegraphics[width=1.\linewidth]{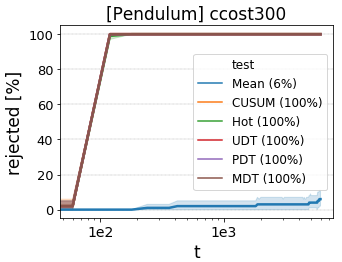}
    \end{subfigure}
    \begin{subfigure}{0.33\textwidth}
    \includegraphics[width=1.\linewidth]{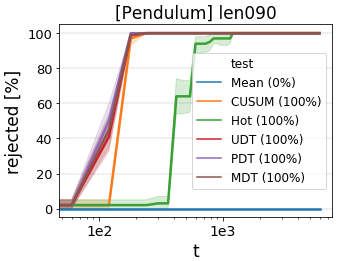}
    \end{subfigure}
    
    \begin{subfigure}{0.33\textwidth}
    \includegraphics[width=1.\linewidth]{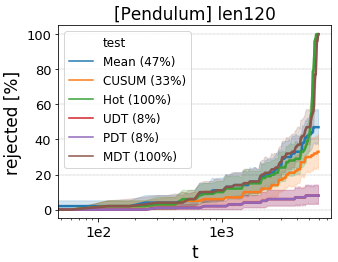}
    \end{subfigure}
    \begin{subfigure}{0.32\textwidth}
    \includegraphics[width=1.\linewidth]{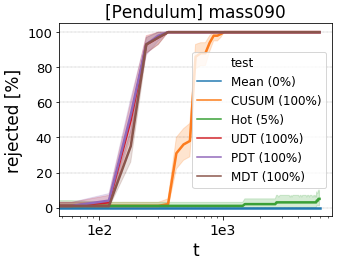}
    \end{subfigure}
    \begin{subfigure}{0.33\textwidth}
    \includegraphics[width=1.\linewidth]{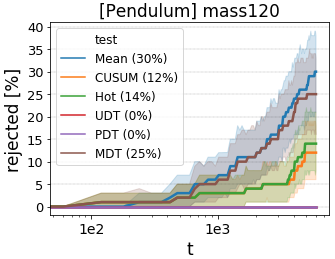}
    \end{subfigure}
    
    \begin{subfigure}{0.33\textwidth}
    \includegraphics[width=1.\linewidth]{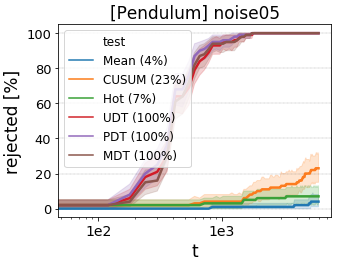}
    \end{subfigure}
    \begin{subfigure}{0.32\textwidth}
    \includegraphics[width=1.\linewidth]{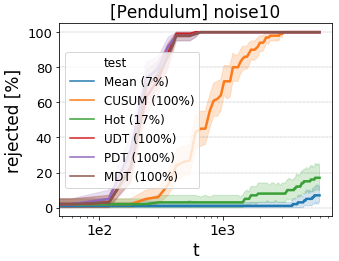}
    \end{subfigure}
    \begin{subfigure}{0.33\textwidth}
    \includegraphics[width=1.\linewidth]{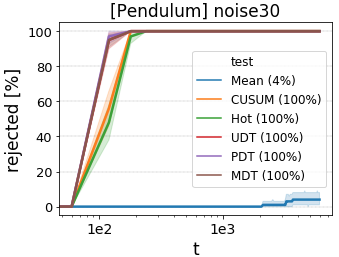}
    \end{subfigure}
    
    \caption{\small Sequential tests in different scenarios in Pendulum environment: cumulative percent of rejections vs. number of simulated time-steps. In the legend, the numbers in parenthesis are the final percents of rejection.}
    \label{fig:pend_seq}
\end{figure}

\begin{figure}[!h]
    \centering
    
    \begin{subfigure}{0.33\textwidth}
    \includegraphics[width=1.\linewidth]{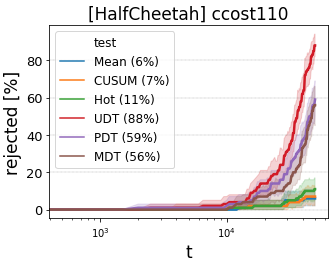}
    \end{subfigure}
    \begin{subfigure}{0.32\textwidth}
    \includegraphics[width=1.\linewidth]{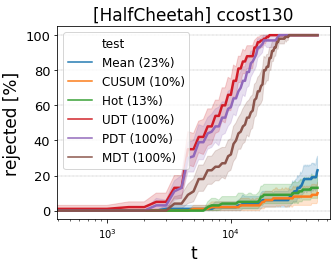}
    \end{subfigure}
    \begin{subfigure}{0.33\textwidth}
    \includegraphics[width=1.\linewidth]{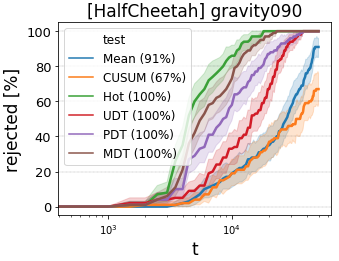}
    \end{subfigure}
    
    \begin{subfigure}{0.33\textwidth}
    \includegraphics[width=1.\linewidth]{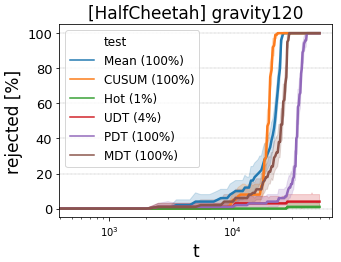}
    \end{subfigure}
    \begin{subfigure}{0.32\textwidth}
    \includegraphics[width=1.\linewidth]{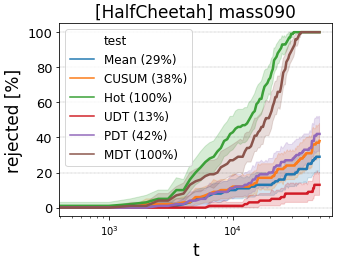}
    \end{subfigure}
    \begin{subfigure}{0.33\textwidth}
    \includegraphics[width=1.\linewidth]{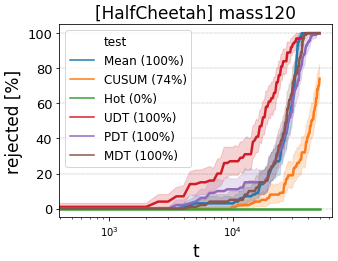}
    \end{subfigure}
    
    \caption{\small Sequential tests in different scenarios in HalfCheetah environment: cumulative percent of rejections vs. number of simulated time-steps. In the legend, the numbers in parenthesis are the final percents of rejection.}
    \label{fig:hc_seq}
\end{figure}

\begin{figure}[!h]
    \centering
    
    \begin{subfigure}{0.33\textwidth}
    \includegraphics[width=1.\linewidth]{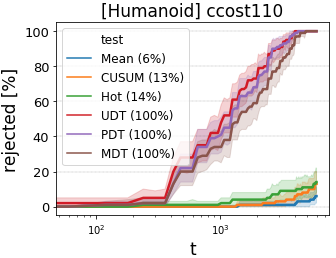}
    \end{subfigure}
    \begin{subfigure}{0.33\textwidth}
    \includegraphics[width=1.\linewidth]{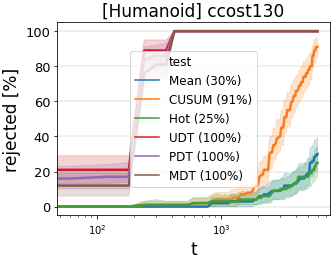}
    \end{subfigure} \\
    
    \begin{subfigure}{0.33\textwidth}
    \includegraphics[width=1.\linewidth]{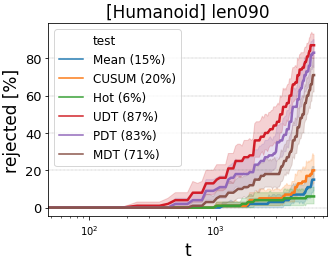}
    \end{subfigure}
    \begin{subfigure}{0.33\textwidth}
    \includegraphics[width=1.\linewidth]{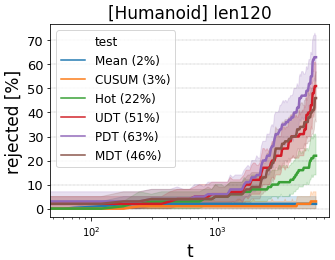}
    \end{subfigure}

    \caption{\small Sequential tests in different scenarios in Humanoid environment: cumulative percent of rejections vs. number of simulated time-steps. In the legend, the numbers in parenthesis are the final percents of rejection.}
    \label{fig:hum_seq}
\end{figure}

\begin{figure}[!h]
    \centering
    
    \begin{subfigure}{0.44\textwidth}
    \includegraphics[width=1.\linewidth]{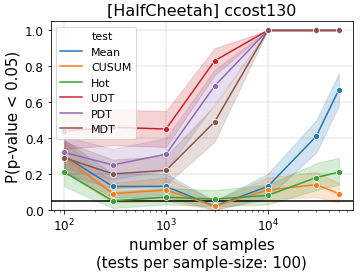}
    \caption{}
    \label{fig:hc_ind_ccost130}
    \end{subfigure}
    \begin{subfigure}{0.55\textwidth}
    \includegraphics[width=1.\linewidth]{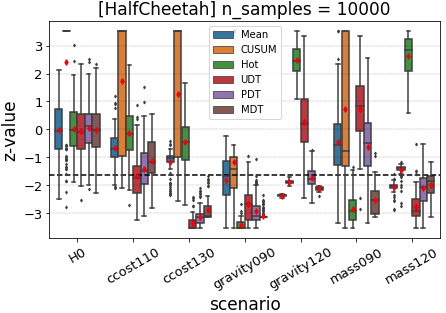}
    \caption{}
    \label{fig:hc_ind_10000}
    \end{subfigure}
    \caption{\small Individual (not sequential) tests in HalfCheetah environment: (a) percent of rejections (with significance $\alpha=0.05$) vs. number of samples: recall that $T=1000$ samples correspond to a single episode, and note that Mean, CUSUM and Hotelling perform better in the beginning of the first episode -- before most of the noise comes in; (b) for each scenario and each test-statistic -– the distribution of the $M=100$ z-values corresponding to simulated data blocks of 10 episodes each. The horizontal line represents the rejection threshold for significance $\alpha=0.05$.}
    \label{fig:hc_ind}
\end{figure}

\begin{figure}[!h]
    \centering
    
    \begin{subfigure}{0.4\textwidth}
    \includegraphics[width=1.\linewidth]{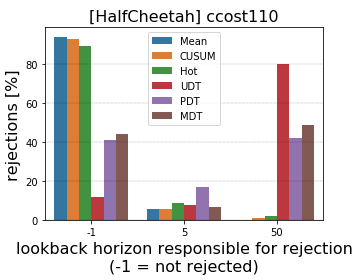}
    \end{subfigure}
    \begin{subfigure}{0.4\textwidth}
    \includegraphics[width=1.\linewidth]{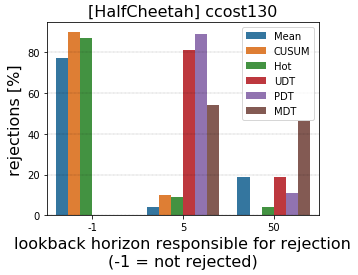}
    \end{subfigure}

    \caption{\small Lookback horizons for which $H_0$ was rejected in sequential tests: smaller degradation requires longer horizon (i.e., more data) for detection.}
    \label{fig:horizons}
\end{figure}

\begin{figure}[!h]
    \centering
    \includegraphics[width=0.8\linewidth]{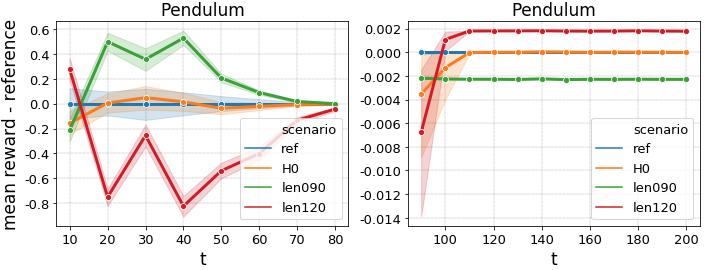}
    \caption{\small Rewards degradation in Pendulum following changes in pole length, over $N=3000$ episodes per scenario. The figure is split due to the extreme scale difference between $t<90$ and $t>90$.}
    \label{fig:pend_degradation}
\end{figure}

\begin{figure}[!h]
\centering
\begin{subfigure}{.37\textwidth}
  \centering
  \includegraphics[width=1.\linewidth]{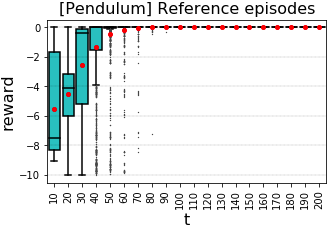}
  \caption{}
  \label{fig:pend_rewards}
\end{subfigure}
\begin{subfigure}{.23\textwidth}
  \centering
  \includegraphics[width=1.\linewidth]{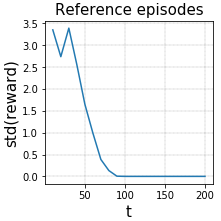}
  \caption{}
  \label{fig:pend_std}
\end{subfigure}
\begin{subfigure}{.37\textwidth}
  \centering
  \includegraphics[width=1.\linewidth]{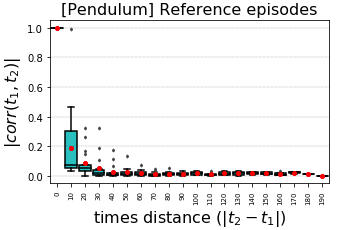}
  \caption{}
  \label{fig:pend_cor_diag}
\end{subfigure}
\caption{\small Parameters of an episodic signal of the rewards in Pendulum environment, estimated over $N=3000$ episodes of $T=200$ time-steps: (a) distribution of rewards per time-step; (b) standard deviations; (c) correlation($t_1,t_2$) vs. $|t_2-t_1|$. The estimations were done in resolution of 10 time-steps, i.e., every episode was split into 20 intervals of 10 consecutive rewards, and each sample is the average over an interval.}
\label{fig:pend_params}
\end{figure}

\begin{figure}[!h]
\centering
\begin{subfigure}{.37\textwidth}
  \centering
  \includegraphics[width=1.\linewidth]{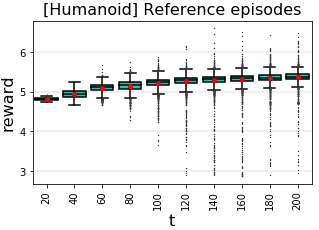}
  \caption{}
  \label{fig:hum_rewards}
\end{subfigure}
\begin{subfigure}{.23\textwidth}
  \centering
  \includegraphics[width=1.\linewidth]{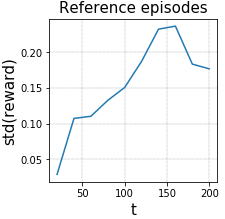}
  \caption{}
  \label{fig:hum_std}
\end{subfigure}
\begin{subfigure}{.37\textwidth}
  \centering
  \includegraphics[width=1.\linewidth]{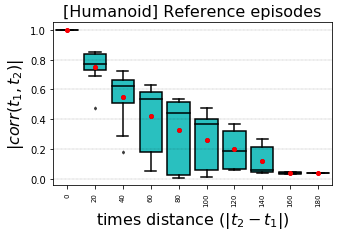}
  \caption{}
  \label{fig:hum_cor_diag}
\end{subfigure} \\

\begin{subfigure}{.4\textwidth}
  \centering
  \includegraphics[width=1.\linewidth]{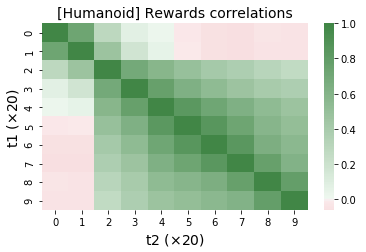}
  \caption{}
  \label{fig:hum_cor}
\end{subfigure}
\begin{subfigure}{.4\textwidth}
  \centering
    \includegraphics[width=1.\linewidth]{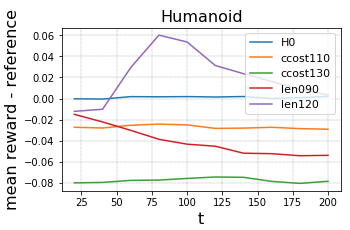}
  \caption{}
  \label{fig:hum_degradation}
\end{subfigure}
\caption{\small Parameters of an episodic signal of the rewards in Humanoid environment, estimated over $N=5000$ episodes of $T=200$ time-steps: (a) distribution of rewards per time-step; (b) standard deviations; (c) correlation($t_1,t_2$) vs. $|t_2-t_1|$; (d) correlations map; (e) rewards degradation following changes in control costs and leg size. The estimations were done in resolution of 20 time-steps, i.e., every episode was split into 10 intervals of 20 consecutive rewards, and each sample is the average over an interval.}
\label{fig:hum_params}
\end{figure}


\section{Sensitivity to Covariance Matrix Estimation}
\label{sec:cov_sensitivity}

In most of the analysis in this work we assume that both the means $\pmb{\mu_0}$ and the covariance $\Sigma_0$ of the episodic signal $X$ are known.
In practice, this can be achieved either through detailed domain knowledge, or by estimation from the recorded reference dataset of Setup~\ref{setup:ind}, assuming it satisfies Eq.~\eqref{eq:episodic_dist}.
The parameters estimation errors decrease as $\mathcal{O}(1/\sqrt{N})$ with the number $N$ of reference episodes, and are distributed according to the Central Limit Theorem (for means) and Wishart distribution~\citep{Wishart} (for covariance).

If $N$ is suspected to be too small for accurate estimation, it is possible to deal with the estimation errors of the model parameters through regularization. One possible regularization is assuming absence of correlations between distant time-steps ($\exists\delta\in\mathbb{N},\forall |t_2-t_1|>\delta: (\Sigma_0)_{t_1t_2}=0$). Another is to essentially reduce $T$ through grouping of sequences of time-steps together (as we do in Section~\ref{sec:experiments}, for example).

To test the practical consequences of inaccurate parameters estimation, we repeated some of the offline (individual) tests of Section~\ref{sec:experiments} for HalfCheetah -- with different sizes of reference datasets. The reference datasets vary between $N=100$ and $N=10000$ episodes (where $N=10000$ corresponds to Section~\ref{sec:experiments}).
As in Section~\ref{sec:experiments}, we downsample each episode from $T=1000$ to $F=T/d=40$ time-steps.

Figure~\ref{fig:cov_sensitivity} shows the results of the sensitivity tests.
Even with as little as $N=100$ reference episodes, the largest weights are successfully assigned to the first time-steps (mind the logarithmic scale in both axes), although certain later weights are still noisy.
$N=300$ is sufficient to yield a consistent statistic distribution under $H_0$, i.e., to reliably tune the false alarm rate.
All sizes of reference datasets yield similar test power in the tested scenarios \textit{ccost130} and \textit{gravity090}.
$N=3000$ is hardly distinguishable from $N=10000$ by any mean.

\begin{figure}[!b]
  \centering
  \includegraphics[width=.8\linewidth]{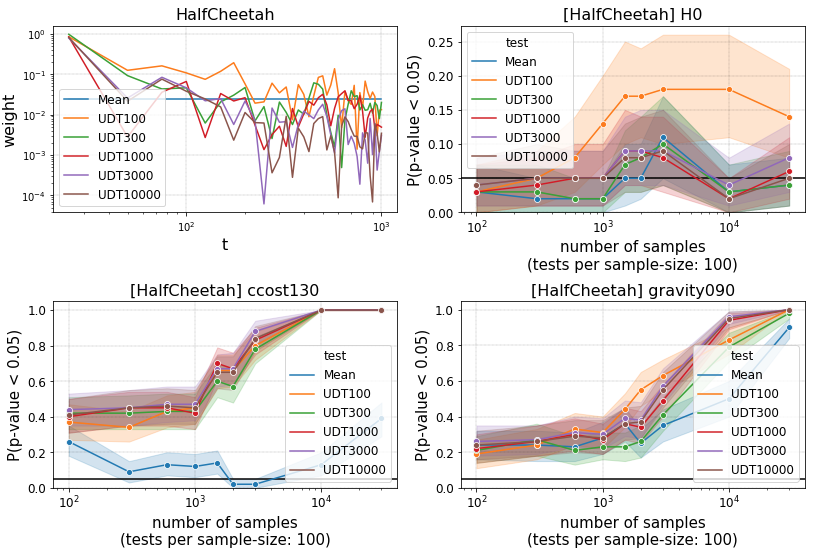}
  \caption{\small The weights of Uniform Degradation Tests (UDT), based on estimation of parameters from reference datasets of various sizes (top left); and percents of degradation detections in individual tests in different scenarios (with significance $1-\alpha=0.95$).}
  \label{fig:cov_sensitivity}
\end{figure}


\end{document}